\documentclass[final,12pt]{clear2024} % Anonymized submission
\usepackage{algorithm}
\usepackage[noend]{algorithmic}
\usepackage{graphicx}

\usepackage{mathtools}

\usepackage{wrapfig}

\usepackage{fancyref}
\usepackage{bm}

\DeclarePairedDelimiter\floor{\lfloor}{\rfloor}
\usepackage{bbm, dsfont}
\usepackage{thmtools,thm-restate}
\usepackage{enumitem}
\usepackage{subcaption}
\usepackage{tikz}
\usetikzlibrary{calc}
\usetikzlibrary{shapes, arrows, positioning}
\usepackage{amsmath,amsfonts,amssymb,mathtools, bbm, enumitem}
\newcommand{\G}[0]{\mathcal{G}}
\usepackage{dsfont}
\DeclareMathOperator*{\argmax}{arg\,max}
\DeclareMathOperator*{\argmin}{arg\,min}
\usepackage{color}
\usepackage{tikz}
\usetikzlibrary{shapes, arrows, positioning}
\usepackage{hyperref}

\newcounter{relctr} %% <- counter for relations
\everydisplay\expandafter{\the\everydisplay\setcounter{relctr}{0}} %% <- reset every eq
\newcommand\labelrel[2]{%
  \begingroup
    \refstepcounter{relctr}%
    \stackrel{\textnormal{(\alph{relctr})}}{\mathstrut{#1}}%
    \originallabel{#2}%
  \endgroup
}
\AtBeginDocument{\let\originallabel\label}

\clearauthor{%
 \Name{Fateme Jamshidi} \Email{fateme.jamshidi@epfl.ch}\\
 \addr EPFL, Switzerland
 \AND
 \Name{Jalal Etesami} \Email{j.etesami@tum.de}\\
 \addr TUM, Germany
\AND
\Name{Negar Kiyavash} \Email{negar.kiyavash@epfl.ch}\\
\addr EPFL, Switzerland
}

\title[Confounded Budgeted Causal Bandits]{Confounded Budgeted Causal Bandits\\
}
\usepackage{times}
\begin{document}

\maketitle

\begin{abstract}%
   We study the problem of learning ``good'' interventions in a stochastic environment modeled by its underlying causal graph. Good interventions refer to interventions that maximize rewards. Specifically, we consider the setting of a pre-specified budget constraint, where interventions can have non-uniform costs. 
    We show that this problem can be formulated as maximizing the expected reward for a stochastic multi-armed bandit with side information. 
    We propose an algorithm to minimize the \textit{cumulative regret} in general causal graphs. This algorithm trades off observations and interventions based on their costs to achieve the optimal reward. 
    This algorithm generalizes the state-of-the-art methods by allowing non-uniform costs and hidden confounders in the causal graph.
    %We present an extension of this algorithm to general causal graphs, even those with hidden variables. 
    %Furthermore, we develop an algorithm to minimize the \textit{simple regret} in general causal graphs. 
    Furthermore, we develop an algorithm to minimize the \textit{simple regret} in the budgeted setting with non-uniform costs and also general causal graphs.
    We provide theoretical guarantees, including both upper and lower bounds, as well as empirical evaluations of our algorithms. Our empirical results showcase that our algorithms outperform the state of the art.
\end{abstract}


 \section{Introduction}
%\vspace{-.2em}
Multi-armed bandits (MAB) problem has been widely studied in sequential decision-making literature \citep{lai1985asymptotically, even2006action}. 
%In this problem, a learner sequentially selects an arm to pull and receives a reward (or pays a cost) with the goal of maximizing the expected reward (or minimizing the associated cost).
In this problem, a learner sequentially selects an arm to pull and receives a stochastic reward. The learner tries different arms with the goal of maximizing the expected reward.
A commonly used assumption in the literature is that the arms are statistically independent. 
%One of the traditional assumptions in MAB problems is assuming the arms' reward to be independent.
In other words, the distribution of one arm's reward contains no information about the reward of the other arms.
Under this assumption, a variety of approaches have been developed in the literature to solve the MAB problem, such as Thompson sampling \citep{thompson1933likelihood} and variants of Upper Confidence Bound (UCB) \citep{auer2002finite, cappe2013kullback}.
Recently, a variant of the problem where dependencies among different arms are allowed has been studied.  In such a setting, prevalent in real-world problems, pulling an arm reveals additional information about other arms. Some examples include linear optimization setting \citep{dani2008stochastic}, combinatorial bandits \citep{cesa2012combinatorial}, and Lipschitz bandits \citep{magureanu2014lipschitz}. 

An effective and succinct representation of interdependencies among a set of variables (e.g., arms) can be captured by its corresponding causal graph \citep{pearl1995causal}. 
Such graphs have been successfully used in a wide range of applications from agriculture \citep{splawa1990application} and genetics \citep{meinshausen2016methods} to marketing \citep{kim2008trust} to model the causal relationships. 
In this work, we study the MAB problem in a stochastic environment in which the dependencies among different arms are modeled by the underlying causal graph. It is assumed that this causal graph is available to the learner as side information. 
This formulation is known as causal MAB, and it has recently gained increasing attention in literature \citep{bareinboim2015bandits, lattimore2016causal, lee2018structural, lee2019structural, lu2020regret, nair2021budgeted, maiti2022causal}.

In some versions of the MAB problem, pulling an arm is associated with a cost \citep{kocaoglu2017cost, lindgren2018experimental, nair2021budgeted}. 
In this setting, the challenge of the learner with a limited budget is to use the budget for exploring different arms effectively in order to maximize the reward. As an example, consider a treatment-effect problem in which the goal of a practitioner is to measure the effectiveness of different treatments and ultimately find the most effective one. 
In this example, the effectiveness of the treatments (e.g., the percentage of recovered patients) denotes the reward.
On the other hand, different treatments may have different costs. 
Suppose that there are two treatments available: A) a medicament and B) a surgery. As pulling arm B is more expensive than arm A in this problem, the practitioner's challenge is to use her given budget effectively to try both treatments and maximize the reward.

We study causal MAbs with \textit{non-uniform} costs for pulling arms. As we discuss in Sections \ref{sec: general cumulative} and \ref{sec:simple}, having non-uniform costs lead to different learning algorithms and theoretical guarantees.  
Additionally, we relax the existing structural assumptions on the underlying causal graph, as such structural assumptions may not be valid in many real-world problems.
These assumptions were put into place to simplify accounting for the information pulling an arm reveals about other arms.
For instance, a standard result in causal inference literature implies that when the causal graph does not have any unblocked backdoor path (see Appendix \ref{sec: tech} for definitions) between the intervened variables and the reward variable, the effect of any intervention (pulling any arm) is equal to the conditional expectation of the reward given that arm \citep{pearl2009causality}.
Lastly, previous work has mainly considered the case where the causal graph is fully observable. 
We relax this assumption by allowing for so-called unobserved confounders, i.e., variables we cannot observe.
%Along with the interventions, pure observation or empty intervention is treated as an arm. 
%Furthermore, the effects of interventions and observations are assumed to be compatible with a known causal graph.
%The causal bandit algorithms aim to learn the intervention that maximizes the reward.
%Although, in the great deal of settings, the budget is limited and doing interventions are pricey \citep{kocaoglu2017cost, lindgren2018experimental, nair2021budgeted}. However, to minimize the regret in this problem, there is no algorithm in literature which considers the cost of the interventions to be non-uniform and a general graph as the underlying causal graph.

\paragraph{Contribution:}
Our main contributions are as follows.
\vspace{-.2cm}
\begin{itemize}[leftmargin=*]

    \item We generalize the setting studied in the state of the art in causal bandit literature by allowing non-uniform
    costs and hidden confounders in the causal graph. Non-uniform costs introduce additional complexity to the MAB problem in terms of the trade-off between exploration and exploitation. It becomes crucial for the learner to select an arm for exploration not only based on its reward but also its associated cost. To address this complexity, we propose algorithms that incorporate cost-dependent exploration criteria both in the setting of simple and cumulative regret. 
    General causal graphs with hidden confounders add yet another challenge: how to avoid spurious correlations in the data as a result of the confounders and harness true causal relationships to learn about other arms besides the one being played. 
    To overcome this challenge, we propose estimators in Section \ref{sec: updates} for the expected reward of arms that leverage both observational and interventional data.
    
    \item We propose two algorithms (Algorithm \ref{alg: g-cumulative} in Section \ref{sec: general cumulative} and Algorithm \ref{alg: g-simple} in Section \ref{sec:simple}) to minimize the cumulative and simple regrets, respectively\footnote{Our theoretical results both generalize the results in \citep{nair2021budgeted} and \citep{maiti2022causal} (to allow for non-uniform costs and general causal graphs) and correct the oversights and errors in the proofs of these papers which affect the validity of the bounds claimed therein (see Section \ref{sec: error} for details).} and upper bound  their expected regrets.      
   % In Algorithm \ref{alg: g-cumulative}, % we  derive an upper bound on its expected cumulative regret in Theorem \ref{th: cumu}.
   We prove that by leveraging causal information, Algorithm \ref{alg: g-cumulative} achieves better cumulative regret than the optimal classic MAB algorithm.
    In Algorithm \ref{alg: g-simple}, we propose a new threshold that accounts for the cost of pulling arms to identify infrequent arms more effectively. As a result, Algorithm \ref{alg: g-simple}
    outperforms prior work \citep{nair2021budgeted} even in their own settings (when the costs are uniform and the causal graph has no-backdoor). % (see Remark \ref{remark: simple no} for more details). 
    %This enhanced performance is attributed to our method of selectively choosing the more effective subset of infrequent arms.
    %we propose a new threshold (see Equation \eqref{eq: compute n}) that accounts for the cost of pulling arms to identify infrequent arms.
    % As shown in Remark \ref{remark: simple no} and Section \ref{sec: exp simple no}, our algorithm outperforms the state-of-the-art algorithms even in their own settings that is the costs are uniform, and the causal graph has no-backdoor.
    %We present an improved upper bound compared to the state of the art on the expected simple regret of this algorithm in Theorem \ref{th: simple g}. 
    Moreover, we present lower bounds on both simple and cumulative regrets of any algorithm and discuss their relations with the presented upper bounds in Section \ref{sec: lowers}. 
    %that Algorithm \ref{alg: g-simple} has a lower regret bound and and show that the cumulative regret of any algorithm is lower bounded by $\small{\Omega\big(\sqrt{\lfloor B/c\rfloor KN}\big)}$, where $B$, $c$, $N$, $K$ denote the budget, maximum cost of the intervention, number and the cardinality of the intervenable variables, respectively.  
    
    \item We evaluate our proposed algorithms in Section \ref{sec: expe}. Our simulation results show that our algorithms perform well for general causal graphs and non-uniform costs and outperform the state of the art even in the settings they were specifically designed for. 
    
    %In Section \ref{sec: expe}, we show that our algorithms outperform the state of the art through different experiments.
    %Furthermore, we discuss the gap between the regrets obtained by our algorithms and the existing lower bounds. 
    %$\mathcal{O}\big(\sqrt{\frac{n(\mathbf{q})}{B} \log \frac{NB}{n(\mathbf{q})}}\big)$.
    %when the cost of interventions are non-uniform. We present Algorithm \ref{alg: g-cumulative} for this problem (see Theorem \ref{th: cumu}).
\end{itemize}

\subsection{Related Work}
%\vspace{-.2em}
%In this section, we review the related developments in causal MAB literature. 

Authors in \citep{tran2012knapsack} propose F-KUBE, an algorithm for a budgeted MAB problem without utilizing the underlying causal graph.
The authors in \citep{lattimore2016causal} study the problem of minimizing the simple regret in a special causal graph called parallel graphs\footnote{
It is composed of variable set $\mathbf{V}= \{X_1, \dots, X_N, Y\}$ and edges from each $X_i$ to $Y$.}
%The only edges in the parallel graph with the set of variables $\mathbf{V}= \{X_1, \dots, X_N, Y\}$ are from each $X_i$ to $Y$, where $Y$ is the reward variable.} 
after $T$ steps where the cost of pulling all arms is one. They propose an algorithm with average regret of $\mathcal{O}\big(\sqrt{\frac{a}{T}\log \frac{NT}{a}}\big)$, where $a$ defined in Remark \ref{remark: simple no} depends on the underlying causal model and $N$ is the number of intervenable variables.

The authors in \citep{nair2021budgeted} study a causal MAB problem in which the learner has a limited budget $B$, all interventions have the same cost $c\geq1$, and the cost of observation is one.
They consider the problem of minimizing the simple regret in special causal graphs called no-backdoor graphs\footnote{The graphs in which all backdoor paths from each intervenable variable to the reward variable are blocked. Please refer to Appendix \ref{sec: tech} for details.}.
They show that their proposed algorithm's expected regret is upper bounded by $\mathcal{O}\big(\sqrt{\frac{ca}{B}\log \frac{NB}{ca}}\big)$.
We also study this particular setting in Section \ref{sec:simple} as a special case of our setting but allow for non-uniform costs and derive a tighter bound for the expected regret (Remark \ref{remark: simple no}). 

\cite{nair2021budgeted} studies non-budgeted a causal MAB problem with general causal graphs when the objective is the cumulative regret. The proposed algorithm in \citep{nair2021budgeted} requires access to the distribution of parents of the reward variable for each intervention. This restrictive assumption is also required in \citep{lu2020regret}.  
%and also no backdoor path \footnote{There is no backdoor path from each $X_i$ to $Y$ where the set of variables is $\{X_1, \dots, X_N, Y\}$} where the cost of doing any interventions and observation are equal to $c$ and 1, respectively.  
\cite{maiti2022causal} studies a causal MAB problem when all costs are assumed to be one for both simple and cumulative regret objectives.
In the case of simple regret, the proposed algorithm for causal graphs with possibly hidden confounders attains an expected simple regret upper bounded by $\mathcal{O}\big( \sqrt{\frac{b}{T} \log \frac{NT}{b}}\big)$, where $b$ depends on the causal model. 
In the case of cumulative regret, the proposed algorithm only works for causal graphs with no hidden variables. We generalize both aforementioned results to non-uniform cost settings and derive tighter theoretical bounds on the regret.

%\vspace{-.2em}

\section{Preliminaries}
%\vspace{-.25em}
Throughout this paper, random variables and their realizations are denoted by capital and lowercase letters, respectively.
We use bold capital and lowercase letters to denote sets of variables and their realizations, respectively. 
\\
\textbf{Causal structure:} 
 Let $\mathcal{G}=(\mathbf{V}, \mathbf{E}^d, \mathbf{E}^b)$ denote an acyclic-directed mixed graph (ADMG) with the set of observed variables $\mathbf{V}$, the set of \textit{directed} edges $\mathbf{E}^d \subseteq \mathbf{V} \times \mathbf{V}$ and the set of \textit{bidirected} edges $\mathbf{E}^b \subseteq \binom{\mathbf{V}}{2}$.
 The existence of a bidirected edge between nodes $V_1$ and $V_2$ represents a hidden confounder that influences both $V_1$ and $V_2$.
 
Given two arbitrary variables $V_1,V_2 \in \mathbf{V}$, when $(V_1,V_2) \in \mathbf{E}^d$, $V_1$ is a parent of $V_2$ and $V_2$ is a child of $V_1$. The set of parents of $V_2$ is denoted by $\mathbf{Pa}(V_2)$.

Given two subsets of variables $\mathbf{R}$ and $\mathbf{S}$ and their realizations $\mathbf{r}$ and $\mathbf{s}$, respectively,
let $ P_{\mathbf{s}}(\mathbf{r}):=P(\mathbf{R}=\mathbf{r}|do(\mathbf{S}=\mathbf{s}))$ denote the post-interventional distribution of $\mathbf{R}$ after intervening on  $\textbf{S}$.%, i.e., $do(\mathbf{S}=\mathbf{s})$.

\begin{definition}[C-component \citep{tian2002general}]
    Two observed variables $V_1$ and $V_2$ are said to be in a c-component of an ADMG $\mathcal{G}$, if and only if they are connected by a bi-directed path.
\end{definition}

As an example, in Figure \ref{fig: example}, $\{X_1,X_2,X_3,X_5\}$ and $\{X_4\}$ are two c-components of $\mathcal{G}$.

\begin{definition}[Identifiability \citep{tian2002general}]
    Given an ADMG $\mathcal{G}= (\mathbf{V}, \mathbf{E}^d, \mathbf{E}^b)$, and two disjoint subsets  $\mathbf{R} ,\mathbf{S} \subseteq \mathbf{V}$, $P_{\mathbf{s}}(\mathbf{r})$ is said to be identifiable in $\mathcal{G}$ if $P_{\mathbf{s}}(\mathbf{r})$ is uniquely computable from $P(\mathbf{V})$.% in any causal model which induces $\mathcal{G}$.
\end{definition}
\textbf{Causal multi-arm bandits:} 
 Let $\mathbf{X}= \{X_1, \dots, X_N\}\subseteq \mathbf{V}$ and $Y\in \mathbf{V}$ denote the set of \textit{intervenable} variables (variables that the learner is allowed to intervene on) and the reward variable, respectively. For ease of presentation, we assume that all variables are binary. 
All our results can be extended to sets of finite-domain variables. 

In the causal MAB setting, at each round, a learner can explore either by intervening in the system or merely observing it. 
If the learner decides to intervene, they will select an intervenable variable, e.g., $X_i\in \textbf{X}$, set its value, e.g.,  $do(X_i=x)$, and observe the remaining variables, i.e., $\mathbf{V}\setminus\{X_i\}$. This choice of action (arm) is denoted by $a_{i,x}$. On the other hand, when the decision is to observe, i.e., $do()$, she merely observes all observed variables. This action is denoted by $a_0$. 
 \begin{wrapfigure}{r}{0.5\textwidth}
\begin{minipage}{0.5\textwidth}
\vspace{-.5cm}
\begin{figure}[H]
        \centering
        \tikzstyle{block} = [draw, fill=white, circle, text centered]
    	\tikzstyle{input} = [coordinate]
    	\tikzstyle{output} = [coordinate]
    	    \begin{tikzpicture}[->, auto, node distance=1.3cm,>=latex', every node/.style={inner sep=0.05cm}, scale=0.7]
    		    \node (X2) at (-1,0) {$X_2$};
    		    \node (X3) at (-0.7,-2.5) {$X_3$};
    		    \node (X4) at (-3,-.2) {$X_4$};
    		    \node (X5) at (-2,-1.5) {$X_5$};
    		    \node (X1) at (-5,-2) {$X_1$};
    		    \path[->] (X2) edge[ style = {->}](X4);
    		    \path[->] (X2) edge[ style = {->}](X3);
    		    \path[->] (X3) edge[ style = {->}](X5);
    		    \path[->] (X3) edge[ style = {->}](X1);
    		    \path[->] (X4) edge[ style = {->}](X5);
    		    \path[->] (X4) edge[ style = {->}](X1);
    		    \path[->] (X5) edge[ style = {->}](X1);
    		   \path[->] (X2) edge[dashed, style = {<->}](X5);
    		   \path[->] (X3) edge[dashed, style = {<->}, bend left=30](X1);
    		   \path[->] (X2) edge[ dashed, style = {<->}, bend right=60](X1);
    		\end{tikzpicture}
    	\caption{An ADMAG $\mathcal{G}$ over $\textbf{V}=\{X_1,...,X_5\}$. Bidirected edges are represented by dashed edges.} %\vspace{-.3em}
    	\label{fig: example}
    \end{figure} 
\vspace{-.5cm}
\end{minipage}
\end{wrapfigure}
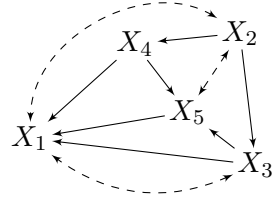

We denote the set of possible actions by 
$
\mathcal{A}:=\big\{a_{i,x}| i \in [N] , x \in \{0,1 \}\big\} \cup \{a_0\}$.
We assume that the cost of pulling an arm  $a\in\mathcal{A}$ is $c_a\in \mathbb{R}_+$ and denote the set of costs by $\mathcal{C}:= \{c_a|a \in \mathcal{A}\}$. Moreover, without loss of generality, we assume that the cost of $a_0$ is one, i.e., $c_0=1$. 
We indicate the pulled arm by the learner, the resulting reward, and the observed values of the variables in an arbitrary subset $\mathbf{S}\subseteq \mathbf{V}$ at time $t$, by $a^t$, $y^t$, and $\mathbf{s}^t$, respectively.
\\
\\
\textbf{Problem Setting:} 
We study a causal MAB problem, in which a learner with a budget $B\geq0$ aims to minimize either its \textit{simple regret} (Section \ref{sec:simple}) or \textit{cumulative regret} (Section \ref{sec: general cumulative}). It is assumed that the learner knows the underlying causal graph. This problem is known as the budgeted causal MAB  \citep{nair2021budgeted}.
Regret is a commonly used measure to evaluate the performance of learners in a bandit setting, and it captures the foregone utility from the actual action choice against the optimum action \citep{cesa2006prediction}.

%We discuss each case in two different subcases: budgeted and non-budgeted. In the budgeted case unlike the non-budgeted, the learner has limited budget to spend.  

In order to formally introduce the regret, we first define the average reward of action $a\in\mathcal{A}$ as follows:
$\mu_{a}:=\mathbb{E}[Y|a]$.
For example, $\mu_{a_{i,x}}$ denotes $\mathbb{E}[Y|do(X_i=x)]$.
\\
\\
\textbf{Simple regret:} Let $a^*$ denote the arm that maximizes the expected reward with budget $B$. %, i.e., $a^*_B:=\arg\max_{a\in\mathcal{A}}\mu_a$.
The simple regret of a learner using budget $B$ is defined by
\begin{equation}\label{eq: simple formula}
    R_{s}(B):= \mu_{a^*}- \mu_{\Tilde{a}_B},
\end{equation}
where $\Tilde{a}_B$ denotes the arm selected by the learner after exhausting budget $B$. When the learner's objective is the minimize simple regret, it suffices to find the best arm at the final step (i.e., after spending their budget) without having to worry about the intermediate actions that they chose. 

In many real-world applications, it is important that the learner does not pull sub-optimal arms too often during her exploration. In this case, the objective function should reflect the intermediate regrets the learner accumulates.
\\
\\
\textbf{Cumulative regret:} 
Let $T_B^{\ell}$ denote the time step that a learner $\ell$ consumes its budget $B$, i.e., at time step $T_B^{\ell}+1$, it does not have enough budget to perform even the lowest cost action. In this case, the expected reward accumulated by the learner $\ell$ will be 
$\mathcal{R}^{\ell}(B):=\sum_{t\leq T_{B}^{\ell}} \mu_{a^t},$  
where $\mu_{a^t}$ is the rewards of action taken at time $t$.
Furthermore, let $\mathcal{R}^*(B)$ denote the expected reward accumulated by the optimum learner with budget $B$.
Then, the cumulative regret of the learner $\ell$ using budget $B$ is given by
\begin{equation}\label{eq: cumu formula}
    R_{c}(B) := \mathcal{R}^*(B) -\mathcal{R}^{\ell}(B).
\end{equation}
As we consider a single learner in this work, in the rest of the paper, we drop the superscript $\ell$. A learner minimizing cumulative regret must trade off exploration vs. exploitation. 
\begin{remark}\label{remark: non-budget}
We can define a non-budgeted causal MAB problem in which there is no cost associated with pulling an arm, but the learner has limited time $T$ to either identify the best arm or minimize cumulative regret during $T$ steps. This problem is a special case of the budgeted MAB problem. Assume all arms have the same cost $c_a=c>0$, then a budgeted causal MAB with budget $B$ is equivalent to a non-budgeted causal bandit with the time limit $T=B/c$ and the simple and cumulative regrets are given by $R_s(B/c)$ and $R_c(B/c)$, respectively.
\end{remark}

%\vspace{-.2em}
\section{Cumulative Regret in General Graphs}\label{sec: general cumulative}
In this section, we study the budgeted causal MAB problem in general causal graphs with hidden confounders when the learner's objective is to minimize cumulative regret.
We propose Algorithm \ref{alg: g-cumulative}, developed based on Upper Confidence Bound (UCB) algorithm \citep{auer2002finite}, which generalizes the state-of-the-art in causal MAB in two ways: it allows for non-uniform costs among the arms and as well as the existence of hidden confounders in the causal graph. 

Non-uniform costs change the optimal exploitation policy as, depending on the costs, pulling the arm with the highest reward repeatedly, in general, does not maximize the learner’s accumulated reward within the budget. Indeed our empirical studies in Section \ref{sec: exp cumu} show that Algorithm \ref{alg: g-cumulative} outperforms existing causal MAB algorithms designed for uniform costs.

Algorithm \ref{alg: g-cumulative} works in general graphs and relaxes existing structural assumptions on the underlying causal graph in the literature. 
Recently, \cite{maiti2022causal} studied the non-budgeted causal MAB problem with graphs that have no hidden confounders. 
This is a limiting assumption in many real-world applications such as medical science, epidemiology, and sociology when it is impossible to ensure that all common confounders are measured in a study \citep{leek2007capturing,imai2010general,colombo2012learning}.
Our proposed algorithm merely requires the {identifiability} assumption for all intervenable variables in $\mathcal{G}$.
%For an intervenable variable $X_i$, its causal effect $P(Y | do(X_i=x))$ is said to be identifiable in $\mathcal{G}$, if it is uniquely computable from observational data. 
A sufficient graphical condition for identifiability of $P_{x_i}(y)$ is that there does not exist a path of bi-directed edges from $X_i$ to its children \citep{tian2002general} which significantly relaxes the existing structural assumptions in the state of the art.

Algorithm \ref{alg: g-cumulative} takes as inputs the causal graph $\G$, the budget $B$, and the cost set $\mathcal{C}$. 
In the beginning, it pulls each arm once (line 1). 
Assuming that the intervenable variables are binary, i.e., $X_i\in\{0,1\}$, this requires $2N+1$ number of steps and costs $\sum_{i,x} c_{i,x}+1$. 
\begin{wrapfigure}{r}{0.5\textwidth}
\begin{minipage}{0.5\textwidth}
\vspace{-0cm}
\begin{algorithm}[H]
\caption{\small{Budgeted Cumulative Regret in General Graphs}}
\label{alg: g-cumulative}
\textbf{Input}: $\G$, $B$, $\mathcal{C}$
\begin{algorithmic}[1] %[1] enables line numbers
\STATE{Pull each arm once and set $t=2N+1$;}
\STATE{Set $B^{t}= B- \sum_{i,x}c_{i,x}-1$ and $\beta=1$;}
%\STATE{Set $\beta=1$;}
\WHILE{$B^t \geq 1 $}
\IF{$N^{t-1}_{0} < \beta^2 \log{t}$ or $B^t < \min_{i,x}c_{i,x}$}
\STATE{pull $a^t=a_0$}
\ELSE
\STATE{pull $ a^t= \argmax_{a\in \mathcal{A}} \Bar{\mu}_{a}^{t-1}$ }
\ENDIF
\FOR {$ a \in \mathcal{A}$}
   \STATE{ Update $N_{a}^{t}= N_{a}^{t-1}+\mathds{1}\{a^t=a\}$.}
   \STATE{Update $\hat{\mu}_{a}^t$ and $\Bar{\mu}_{a}^t$ using Equations \eqref{eq: estimate mu0}, \eqref{eq: estimate mu}, and \eqref{eq: estimate mu0 bar}.}
   %\STATE{Update $\Bar{\mu}_{a}^t$ using Equations \eqref{eq: estimate mu0 bar}.}% and \eqref{eq: estimate mui bar}.}
\ENDFOR
\STATE{Let $\Tilde{a}= \argmax_{a \in \mathcal{A}} (\hat{\mu}_{a}^t/c_a ) $.}
\IF{$\hat{\mu}_{0}^t < (\hat{\mu}_{\Tilde{a}}^{t}/c_{\Tilde{a}})$}
\STATE{Update $\beta = \min \{ \frac{2\sqrt{2}}{(\hat{\mu}^t_{\Tilde{a}}/c_{\Tilde{a}})-\hat{\mu}_{0}^t}, \sqrt{\log{t}}\}$}
\ENDIF
\STATE{set $t=t+1$.}
\STATE{Update  $B^{t} = B^{t-1}- c_{a^{t-1}}$.}
\ENDWHILE
\end{algorithmic}
\end{algorithm}
\vspace{-1.2cm}
\end{minipage}
\end{wrapfigure}

Let $N_{a}^{t}$ denote the number of times that arm $a \in \mathcal{A}$ is pulled at the end of $t$ rounds. We denote the estimated average reward by pulling arm $a$ and its estimated UCB at the end of round $t$ by $\hat{\mu}_a^t$ and $\Bar{\mu}_a^t$, respectively. 
The procedure for these estimators will be discussed in Section \ref{sec: updates}.
As long as the remaining budget at round $t$, $B^t$, is larger than one, Algorithm \ref{alg: g-cumulative} continues to explore and exploit by checking at round $t$ whether arm $a_0$ is pulled at least $\beta^2 \log t$ times (this threshold might change in line 17).
If so, Algorithm \ref{alg: g-cumulative} pulls an arm with the highest $\Bar{\mu}_a^t$ in line 8; otherwise, it pulls arm $a_0$. 
Afterward, in lines 10-14, it updates $\hat{\mu}_a^t$ and $\Bar{\mu}_a^t$ using the newly acquired observational or interventional data, which is discussed in Section \ref{sec: updates}. 
In the end, the threshold $\beta$ and the remaining budget are updated in lines 16-19.
%\vspace{-.2em}

%\setlength{\textfloatsep}{10pt}

\subsection{Estimation and update steps}\label{sec: updates}
%\vspace{-.3em}
Herein, we explain how to estimate and update $\hat{\mu}_a^t$ and $\Bar{\mu}_a^t$ in Algorithm \ref{alg: g-cumulative}.
Recall that $\mu_{i,x} = \mathbb{E}[Y| do (X_i=x)]$. Therefore, to estimate $\mu_{i,x}$, it suffices to estimate $P\big(Y=1 | do (X_i=x)\big)$. %from both observational and interventional data.
%First, we show how to estimate such post interventional distribution from observational data and then discuss how to update them from additional interventional or observational data. 

% To do so, we require additional notations and definitions. 
% \begin{definition}[Effective parents] \label{def: effective parents}
%     Suppose $\mathcal{G}$ be a causal ADMG with set of observed variables $\mathbf{V}$.
%     The effective parents of an arbitrary variable $V_i \in \mathbf{V}$ denoted by $\mathbf{Z}_i$ is the set of all variables $V_j$ such that  $V_j \in \mathbf{Pa}(V_i)$ in the ADMG after latent projection \footnote{See Definition \cite{pearl1995theory}}.
%     % either $V_j \in \mathbf{Pa}(V_i)$ or there exists a directed path from $V_j$ to $V_i$ containing only unobserved variables in $\mathcal{G}$.
% \end{definition}

Let $\mathbf{C}_i$ and $\mathbf{W}_i$ denote the c-component containing $X_i$ and $\mathbf{V}\setminus \{X_i\}$, respectively. 
Given two subsets $\mathbf{S}$ and $\mathbf{R}$ of observed variables such that $\mathbf{S}\subseteq\mathbf{R}$ and a subset of realizations $\mathbf{r}$ for $\mathbf{R}$, we use ${(\mathbf{r})}_\mathbf{S}$ to denote the restriction of $\mathbf{r}$ to the variables in $\mathbf{S}$.
Given two subsets of variables $\mathbf{S}_1$ and $\mathbf{S}_2$, and realizations $\mathbf{s}_1$ for $\mathbf{S}_1$ and $\mathbf{s}_2$ for $\mathbf{S}_2$, we denote the assignments to $\mathbf{S}_1 \cup \mathbf{S}_2$ by $\mathbf{s}_1 \circ \mathbf{s}_2$.

Under the identifiability assumption for intervenable variables, \cite{bhattacharyya2020learning} shows that ${\small  P_{x}(\mathbf{w}_i):=P(\mathbf{W}_i=\mathbf{w}_i|do(X_i=x))}$ can be factorized as follows,
\begin{small}
\begin{align}\label{eq: factor 1}
    P_{x}(\mathbf{w}_i)= 
     \sum_{x' \in \{0,1\}} & \prod_{V_j \in \mathbf{C}_i}P\big(
    (x' \circ \mathbf{w}_i )_{V_j}\big|( x' \circ \mathbf{w}_i)_{\mathbf{Z}_j}\big) \prod_{V_j \notin \mathbf{C}_i }P\big({(\mathbf{w}_i)}_{V_j}\big|(x \circ \mathbf{w}_i)_{\mathbf{Z}_j}\big),
\end{align}
\end{small}
where ${\small\mathbf{Z}_j \! =\!  (  \bigcup_{V_k \in \mathbf{C}_j} \! \! \mathbf{Pa}(V_k)  \cup  \mathbf{C}_j )\! \setminus \! V_j}$ and $\mathbf{C}_j$ is c-component of $V_j$.

 Using \eqref{eq: factor 1}, the expected reward of pulling $a_{i,x}$ would be
 \begin{small}
\begin{align}\notag
    & \mathbb{E}[Y|do(X_i=x)]\!=\! \sum_{\mathbf{w}_i': y=1}\!P\big( \mathbf{W}_i=\mathbf{w}_i'|do(X_i=x)\big)\\ \label{eq: factor 2}
    &= \sum_{\mathbf{w}_i': y=1} \sum_{x' \in \{0,1\}} \prod_{V_j \in \mathbf{C}_i} P\big(
    (x' \circ \mathbf{w}_i')_{V_j} |(x' \circ \mathbf{w}_i' )_{\mathbf{Z}_j}\big) \prod_{V_j \notin \mathbf{C}_i }P\big({(\mathbf{w}_i')}_{V_j}|( x \circ \mathbf{w}_i')_{\mathbf{Z}_j}\big),
\end{align}
\end{small}
where the first summation is over all realization of $\mathbf{W}_i=\mathbf{V}\setminus \{X_i\}$ in which $Y=1$. This is because the terms with $Y=0$ have no contribution to the expectation. 

%$\mathbf{w}_i'$ is an arbitrary realization of $\mathbf{W}_i=\mathbf{V}\setminus \{X_i\}$ such that $Y=1$ and $\mathbf{w}_i''$ is an arbitrary realization of $\mathbf{W}_i\setminus \{ Y\}$. XXX

Define $
    \mathbf{O}^t:=\{t'\leq t|\ a^{t'}=a_0\}$, and
    $\mathbf{I}_{i,x}^{t}:=\{t'\leq t|\ a^{t'}=a_{i,x}\}$.
$\mathbf{O}^t$ and $\mathbf{I}_{i,x}^{t}$ denote the set of time steps at which arms $a_0$ and $a_{i,x}$ are pulled by the end of time $t$, respectively.
Hence, an empirical estimation of average reward of $a_0$ is given by
\begin{small}
\begin{equation}\label{eq: estimate mu0}
    \hat{\mu}_{0}^{t} := \frac{1}{N_{0}^{t}} \sum_{t' \in \mathbf{O}^t} \mathds{1}\big\{a^{t'}=a_0, y^{t'}=1\big\}.
\end{equation}
\end{small}
To estimate $\mu_{i,x}$ from observational data, it suffices to estimate each term in \eqref{eq: factor 2}.
To do so, we partition $\mathbf{O}^t$ into $|\mathbf{V}|$ number of subsets randomly and denote the $j$-th partition by $\mathbf{O}^t_{j}$. We will use the data in $\mathbf{O}^t_{j}$  to estimate $P(V_j|\mathbf{Z}_j)$.
Given a realization $x'$ of $X_i$ and a realization $\mathbf{w}_i'$ of $\mathbf{W}_i$, let 
$
\mathbf{O}^t_{j}(x', \mathbf{w}_i'):= \big\{t' \in  \mathbf{O}^t_j|\mathbf{z}^{t'}_j=(x' \circ \mathbf{w}_i' )_{\mathbf{Z}_j}\big\}.
$
Recall that $\mathbf{w}_i'$ is an arbitrary realization of $\mathbf{W}_i$ in which  $Y=1$. To proceed, we require the following definitions, 
\begin{small}
\begin{align*}
          &S_{j,i}^{t} := \min_{\mathbf{w}_i'} \min_{x'} |\mathbf{O}_{j}^{t}(x',\mathbf{w}_i')|,\ \text{if}\ V_j \in \mathbf{C}_i, \quad
      \tilde{S}_{j,i,x}^{t} := \min_{\mathbf{w}_i'} |\mathbf{O}_{j}^{t}(x,\mathbf{w}_i')|,\ \text{if}\ V_j \notin \mathbf{C}_i,
\end{align*}
\end{small}
where $|\mathbf{O}|$ denotes the size of set $\mathbf{O}$.
We also define the minimum number of data points in the partition sets as 
$$
\small{S_{i,x}^{t}:=\min\big\{\min_{j:V_j \in \mathbf{C}_i}S_{j,i}^{t}, \min_{j:V_j \notin \mathbf{C}_i}\Tilde{S}_{j,i,x}^{t}\big\}.}
$$
In the next step, we partition each $\mathbf{O}_{j}^{t}(x',\mathbf{w}_i')$ into $S_{i,x}^t$ number of subsets randomly and denote the $s$-th subset by $\mathbf{O}^{t,s}_{j} (x', \mathbf{w}_i')$. Let
\begin{small}
\begin{align}\label{eq: Px'}
       &\hat{P}^{t,s}_{j}(x', \mathbf{w}_i'):=  \frac{\sum_{t' \in \mathbf{O}^{t,s}_{j}(x',\mathbf{w}_i')}{\mathds{1}\{v_{j}^{t'}=(\mathbf{w}_i' \circ x')_{V_j}\}}}{|\mathbf{O}^{t,s}_{j}(x',\mathbf{w}_i')|}, \quad V_j \in \mathbf{C}_i,\\ \label{eq: Px}
       &\hat{P}_{j}^{t,s} (x, \mathbf{w}_i'):=  \frac{\sum_{t' \in \mathbf{O}^{t,s}_{j}(x, \mathbf{w}_i')}{\mathds{1}\{v_j^{t'}=(\mathbf{w}_i')_{V_{j}}\}}}{|\mathbf{O}^{t,s}_{j}(x, \mathbf{w}_i')|},\quad V_j \notin \mathbf{C}_i. 
\end{align}
\end{small}
Finally, the expected reward of pulling $a_{i,x}$ is estimated as follows,
\begin{small}
\begin{equation}\label{eq: estimate mu}
    \hat{\mu}_{i,x}^t:= \frac{\sum_{t' \in \mathbf{I}_{i,x}^{t}}\mathds{1}\{y^{t'}=1\}+ \sum_{s \in [S_{i,x}^{t}]}Y^{s}_{i,x}}{N_{i,x}^{t}+S_{i,x}^{t}},
\end{equation}
\end{small}
where $[S_{i,x}^{t}]\!=\!\{1,...,S_{i,x}^{t}\}$ and 
$
Y_{i,x}^{s}\!:=\!\!\sum_{\mathbf{w}_i': y=1} \sum_{x' \in \{0,1\}}\prod_{V_j \in \mathbf{C}_i}\! \hat{P}^{t,s}_{j}(x', \mathbf{w}_i')
    \prod_{V_j \notin \mathbf{C}_i }\!\!\hat{P}^{t,s}_{j}(x, \mathbf{w}_i').
$
%Next result shows that the proposed estimators for $\mu_{i,x}$ and $\mu_0$ are unbiased estimators.
\begin{restatable}{lemma}{lemmaunbiased}
        $\hat{\mu}_{i,x}^t$ in \eqref{eq: estimate mu} and $\hat{\mu}_{0}^{t}$ in \eqref{eq: estimate mu0} are unbiased estimators of $\mu_{i,x}$ and $\mu_0$.
\end{restatable}
%We prove that the expectation of $\hat{\mu}_{0}^t$ and $\hat{\mu}_{i,x}^t$ computed by Equations \eqref{eq: estimate mu0} and \eqref{eq: estimate mu} are equal to $\mu_0$ and $\mu_{i,x}$, respectively, in Appendix \ref{sec: proof cumu}.

Analogous to UCB algorithm in bandits literature \citep{auer2002finite,cesa2006prediction},  Algorithm \ref{alg: g-cumulative} computes UCB estimate of $\mu_a$ at round $t$ using the following equations,
\begin{small}
\begin{align}\label{eq: estimate mu0 bar}
    &\Bar{\mu}_{i,x}^{t}:= \hat{\mu}_{i,x}^{t}+ \sqrt{\frac{2\ln t}{N_{i,x}^{t}+S_{i,x}^{t}}}, \quad \Bar{\mu}_{0}^{t}:= \hat{\mu}_{0}^{t}+ \sqrt{\frac{2\ln t}{N_{0}^{t}}}. %\label{eq: estimate mui bar}
\end{align}
\end{small}
Let $a^*:= \argmax_{a \in \mathcal{A}} \frac{\mu_a}{c_a}$ and for $a \in \mathcal{A}$, let $\delta_a := \frac{\mu_{a^*}}{c_{a^*}} -\frac{\mu_a}{c_a}$. 
Recall that $p_{i,x}=P(X_i=x)$.
%Next result presents an upper bound for the expected cumulative regret of Algorithm \ref{alg: g-cumulative}.
\begin{restatable}{theorem}{theoremcumu}\label{th: cumu}
The expected cumulative regret of Algorithm \ref{alg: g-cumulative} is bounded by
\begin{small}
\begin{align*}
    &\delta_0 \Big( \frac{8\ln B}{\delta_{0}^2}+1 +\frac{\pi^2}{3} \Big) + \sum_{\delta_{i,x}>0}\delta_{i,x} \Big(\frac{8\ln B}{\delta_{i,x}^2}+2 -\frac{8 p_{i,x}}{18\delta_0^2|\mathbf{V}|}\ln b_{i,x} \cdot \tau_{i,x,b} + \frac{\pi^2}{3} \Big),
\end{align*}
\end{small}
where $ b_{i,x} :=  \frac{8}{\delta_{i,x}^2} \ln (\frac{B}{\max_{a}c_a})+1$,
$\tau_{i,x,b} :=\max\big\{0,1-|\mathbf{V}|\cdot\mathcal{W}_i\cdot b_{i,x}^{-p_{i,x}^{2}/(2|\mathbf{V}|)}\big\}$, and $\mathcal{W}_i$ denotes the alphabet size of variables in $\mathbf{V}\setminus\{X_i, Y\}$.
\end{restatable}

The proof of Theorem \ref{th: cumu} is provided in Appendix \ref{sec: proof cumu}.
% This theorem implies that if $\max_{i,x}(\frac{\mu_{i,x}}{c_{i,x}}) \leq \mu_0$, (that is, the observational arm $a_0$ is optimal), the expected cumulative regret of Algorithm \ref{alg: g-cumulative} equals $\mathcal{O}(1)$, i.e., is bounded by a constant. Furthermore, if $a_0$ is not the optimal arm, steps $15-18$ of Algorithm \ref{alg: g-cumulative} imply that $\mathbbm E [\beta^2] \geq \frac{8}{9 \delta_0^2}$. 
This theorem ensures that the maximum number of pulling a sub-optimal arm $a$ is bounded by a factor of
%the term multiplying 
$\delta_a$.

\section{Simple Regret in General Graphs }\label{sec:simple}

\begin{wrapfigure}{r}{0.5\textwidth}
\begin{minipage}{0.5\textwidth}
\vspace{-.8cm}
\begin{algorithm}[H]
\caption{\small{Budgeted Simple Regret in General Graphs}}
\label{alg: g-simple}
\textbf{Input}: $\G$, $B$, $\mathcal{C}$
\begin{algorithmic}[1] %[1] enables line numbers
\FOR {$t\in \{1, 2, \dots, B/2 \}$}
  \STATE{ Pull arm $a_0$ and observe $\mathbf{v}^t$}
\ENDFOR
\STATE{$\hat{\mu}_{0}= {2(\sum_{t=1}^{B/2} y^t)}/{B}$ } %Estimate the average reward of $a_0$:
\FOR{$a_{i,x} \in \mathcal{A} $}
    \STATE{Estimate $\hat{\mu}_{i,x}$ using Alg. \ref{alg: compute mu} in Appendix \ref{sec: estimation mu}}
    \STATE{Estimate $\hat{q}_{i,x}$ using Equation \eqref{eq: q-quantity}}
\ENDFOR
\STATE {Compute $ n(\hat{\mathbf{q}})$ using Equation \eqref{eq: compute n}}
\STATE{Construct $\mathcal{A}' := \{ a_{i,x} \in \mathcal{A} | {\hat{q}_{i,x}}^{k_i} \leq \frac{1}{n(\hat{\mathbf{q}})} \}$}
\IF {$|\mathcal{A}'|=0$}
    \STATE Pull arm $a_0$ for the remaining $\frac{B}{2}$ rounds
    \STATE{Re-estimate $\hat{\mu}_{0}= {(\sum_{t=1}^{B/2} y^t)}/{B}$}
    \FOR{$a_{i,x} \in \mathcal{A} $}
        \STATE{Re-estimate $\hat{\mu}_{i,x} $ using Alg. $\ref{alg: compute mu}$}
    \ENDFOR
\ELSE
\STATE{Compute $n=\frac{B}{2\sum_{i,x}c_{i,x} \mathds{1}\{a_{i,x} \in \mathcal{A}' \}}$}
\STATE{Pull each arm $a_{i,x} \in \mathcal{A}'$ for $n$ rounds}
\FOR{$a_{i,x} \in \mathcal{A}'$}
\STATE{ $\hat{\mu}_{i,x}\!=\! \frac{1}{n} \sum_{t=\frac{B}{2}+1}^{\frac{B}{2}+n|\mathcal{A}'|}  y^t\mathds{1}\{a^t=a_{i,x}\} $}
\ENDFOR
\ENDIF
\STATE{ \textbf{return} $\hat{a}^*\in\argmax_{a \in \mathcal{A}} \hat{\mu}_a$.}
\end{algorithmic}
\end{algorithm}
\vspace{-.2cm}
\end{minipage}
\end{wrapfigure}
In this section, we study the budgeted causal MAB problem with general graph $\mathcal{G}$ for a learner whose objective is simple regret. 
The novelty of our results is that it generalizes the state-of-the-art by allowing non-uniform costs for arms. 
As discussed in the previous section, having non-uniform costs may change the trade-off between exploration vs. exploitation and hence requires a different treatment than non-budgeted
% \footnote{As discussed in Remark \ref{remark: non-budget}, this is equivalent to a budgeted causal MAB in which all costs are one.}
causal MAB. 
Our experiments in Section \ref{sec: expe} showcase that indeed our algorithm outperforms the state of the art, which is designed for uniform costs. 
%Assume that $P(Y | do(X_i=x))$ is \textit{identifiable} given observational distribution. 
% In this section, we generalize Algorithm \ref{alg: NB-simple} to settings with general causal graphs. 
% In such graphs, $\mathbb{E}[ Y | do(X_i=x)]= \mathbb{E}[Y | X_i=x]$ does not always hold. 
%In such setting, since the underlying graph is not limited to no-backdoor structures, it is not possible to compute $\hat{\mu}_{i,x}$ as the previous section.
%Therefore, $\hat{\mu}_{i,x}$ is computable using observational distribution.

Under the identifiability assumption for all intervenable variables in $\mathcal{G}$, we present Algorithm \ref{alg: g-simple} to minimize the simple regret for a budget $B$.
This algorithm generalizes the one in \citep{maiti2022causal} to a budgeted causal MAB setting when the arms have non-uniform costs.
It uses its given budget $B$ to estimate the average reward of each arm and then outputs an arm with the maximum estimated average reward. More specifically, Algorithm \ref{alg: g-simple} takes the causal graph $\G$, the budget $B$, and the cost set $\mathcal{C}$ as inputs.

It pulls arm $a_0$, i.e., collects observational data until it has exhausted half of its budget.  
This leads to an initial estimate of the expected reward of each arm $a \in \mathcal{A}$ (lines 4-8). 
Note that estimating the expected rewards is possible due to the identifiability assumption of intervenable variables and is done by Algorithm \ref{alg: compute mu} presented in Appendix \ref{sec: estimation mu}.
% \begin{wrapfigure}{r}{0.5\textwidth}
% \begin{minipage}{0.5\textwidth}

Algorithm \ref{alg: compute mu} is proposed by \cite{bhattacharyya2020learning} to estimate $\mathbb{E}[Y|do(X)]$ from observational data when the causal effect $P_{x}(y)$ is identifiable in  $\mathcal{G}$. 

When an arm is observed frequently during the first part of the algorithm, the initial estimate of its expected reward becomes accurate. Algorithm \ref{alg: g-simple} spends the other half of its budget to explore the so-called infrequent arms (lines 9-23). 
An arm $a_{i,x} \in \mathcal{A}$ is considered to be infrequent if  $\hat{q}_{i,x}\leq\big(\frac{1}{n(\hat{\mathbf{q}})}\big)^{1/k_i}$, where 
\begin{small}
\begin{align}\label{eq: compute n}
    &n(\hat{\mathbf{q}})\!:=\! \min\! \Big\{\tau| \sum_{i,x} c_{i,x}\mathds{1}\big\{ \hat{q}_{i,x} < \Big(\frac{1}{\tau}\Big)^{\frac{1}{k_i}}\big\} \leq \tau \Big\},\\ \label{eq: q-quantity}
 &\hat{q}_{i,x} := \frac{2}{B} \min_{\mathbf{z}}{ \Big\{\sum_{t=1}^{B/2} \mathds{1} \big\{x_i^t=x,  \widetilde{\mathbf{Pa}}^t(x_i) = \mathbf{z} \big\} \Big\}},
\end{align}    
\end{small}
where  $\tiny{\widetilde{\mathbf{Pa}}(X_i):=\Big(\bigcup_{V_j \in \mathbf{C}_i} \mathbf{Pa}(V_j) \cup \mathbf{C}_i \Big)\setminus X_i}$, and $\mathbf{C}_i$ denotes the c-component in $\G$ containing $X_i$. The size of $C_i$ is denoted by $k_i$.

Let $\mathcal{A}'$ denote the set of all infrequently observed arms. % or equivalently, the set of arms to be explored by the algorithm.
If $\mathcal{A}'= \emptyset $, Algorithm \ref{alg: g-simple} spends the remaining budget for observation, i.e., pulls $a_0$. 
Otherwise, it uses the remaining budget to pull the infrequent arms and update their corresponding estimations.
Finally, it outputs an arm with the maximum estimated average reward.

%Similar to Algorithm \ref{alg: NB-simple}, it plays arm $a_0$ for the first $\frac{B}{2}$ rounds and computes $\hat{\mu}_0$ in lines 2 and 4, respectively.

%Then, the algorithm uses Algorithm \ref{alg: compute mu} proposed by \cite{bhattacharyya2020learning} in Appendix \ref{sec: estimation mu} to compute $\hat{\mu}_{i,x}$ for each $a_{i,x} \in \mathcal{A}$ in line 6.
%Next, to identify the insufficiently observed arms, it compute the probability $\hat{q}_{i,x}$ for each arm $a_{i,x} \in \mathcal{A}$ in line 7. Let $\hat{q}_{i,x}$ be an estimation of $q_{i,x}= \min_{\mathbf{z}} P(X_i=x, \widetilde{\mathbf{Pa}}(X_i)= \mathbf{z})$ where $\widetilde{\mathbf{Pa}}(X_i)= \{\mathbf{Pa}(V_j)| V_j \in \mathbf{C}_i \} \cup \mathbf{C}_i \setminus X_i$ and $\mathbf{C}_i$ be the c-component of $X_i$ in $\G$.
%We define a threshold, denoted by $n(\mathbf{q})$, which is computed as Equation \eqref{eq: compute n}.
%Algorithm \ref{alg: g-simple} estimates $n(\mathbf{q})$ in line 9.

%where $k_i$ is the size of the c-component of $X_i$, i.e.,  $k_i:= |\mathbf{C}_i|$.
%Then, it constructs the set of arms which are inadequately observed, $\mathcal{A}'$, in line 10. 
%The remaining steps are similar to Algorithm \ref{alg: NB-simple}. 
%Finally, it chooses the arm with the maximum estimated expectation of the reward.

\vspace{-.1cm}
\begin{restatable}{remark}{remarkno}\label{remark: no}
% % The parameter $n(\hat{\mathbf{q}})$ and the choice of the set $\mathcal{A}'$ is to trade off the estimation of high-probability actions through Algorithm \ref{alg: compute mu} and low-probability actions through interventions.
% If more efficient estimators are available to replace Algorithm \ref{alg: compute mu} in line 6 of Algorithm \ref{alg: g-simple}, this trade-off can be adjusted.
Consider the special case of no-backdoor graphs (causal graphs with no unblocked backdoor paths from intervenable variables to the reward variable $Y$). 
% Such graphs are called no-backdoor graphs. 
This graphical constraint ensures that for all $X\in\mathbf{X}$, $\mathbb{E}[ Y | do(X=x)]= \mathbb{E}[Y | X=x]$. This is due to the second rule of do-calculus \citep{pearl1995causal}.
For causal MABs with no-backdoor graphs, $\mu_{i,x}$ can be estimated using observation as follows ${\sum_{t=1}^{B/2}y^t\mathds{1}\{x_i^t=x\} }/{\sum_{t=1}^{B/2}\mathds{1}\{x_i^t=x\}}$.
When the interventions have non-uniform costs, redefining $\hat{q}_{i,x}= \frac{2}{B} \sum_{t=1}^{B/2} \mathds{1}\{x_i^t=x\}$ yields drastically lower regrets. This special case and our improvements are discussed in Appendix \ref{sec: noback}. 
\end{restatable}
\vspace{-.1cm}
\begin{restatable}{theorem}{theoremsimpleg}\label{th: simple g}
    The expected simple regret of Algorithm \ref{alg: g-simple} is bounded by  $\small{\mathcal{O}\Big(\sqrt{\frac{n(\mathbf{q})}{B} \log \frac{NB}{n(\mathbf{q})}}\Big)}$.
\end{restatable}
% The proof of this theorem is provided in Appendix \ref{sec: proof simple g}.
\vspace{-.1cm}
\begin{remark}
\cite{maiti2022causal} proposes an algorithm for non-budgeted causal MAB with general causal graphs which is a special case of our setting in all costs are one. 
By setting $c_{i,x}=1$ for all $i$ and $x$ in Theorem \ref{th: simple g}, we can recover their expected simple regret bound.
%Moreover, they use different threshold for determining the infrequent arms which gives them an upper bound of the form $\small{\mathcal{O}\Big(\sqrt{\frac{n'(\mathbf{q})}{T} \log \frac{NT}{n'(\mathbf{q})}}\Big)}$, where $T$ is the time horizon, and $n'(\mathbf{q})$ is defined similar to \eqref{eq: compute n} but with $c_{i,x}=1$ for all $i$ and $x$.
%It is noteworthy that $n(\mathbf{q})\leq c n'(\mathbf{q})$ for all $c\geq1$  which means the bound in Theorem \ref{th: simple g} is smaller than the one in  \citep{maiti2022causal}.
\end{remark}
\vspace{-.1cm}
\begin{restatable}{remark}{remarkbetter} \label{remark: simple no}
\cite{nair2021budgeted} studies the causal MAB problem with no-backdoor graphs and an additional constraint on the costs that is $c_{i,x}=c>1$ for all $i$ and $x$ and $c_0=1$.
Note that this setting does not satisfy the non-budgeted assumption in \citep{maiti2022causal}. 
%Our algorithm generalizes    \citep{nair2021budgeted}'s algorithm to non-uniform costs. In \citep{nair2021budgeted} all arms except $a_0$ have similar costs. %, i.e., $c_{i,x}= c>1$ for all arms $a_{i,x} \in \mathcal{A}\setminus\{a_0\}$ and $c_0=1$. 
Moreover, their algorithm uses a different exploration set than $\mathcal{A}'$ that seems to result in both worse performance and theoretical bound. Specifically, the threshold for determining the infrequent arms in \citep{nair2021budgeted} is given by
$m'(\mathbf{q}):=\min \{ \tau | \sum_{i,x} \mathds{1} \{ p_{i,x} < \frac{1}{\tau}\} \leq \tau\}$. 
As we show in Appendix \ref{sec: noback}, in this setting, $n(\mathbf{q}) \leq c m'(\mathbf{q})$ for all $c>1$ and $\mathbf{q}$.
\cite{nair2021budgeted} shows that the expected simple regret of their algorithm is bounded by $\small{\mathcal{O}\Big(\sqrt{\frac{cm'(\mathbf{q})}{B} \log \frac{NB}{cm'(\mathbf{q})}}\Big)}$. 
Given that $n(\mathbf{q}) \leq cm'(\mathbf{q})$ for all $c>1$, even in the special setting of \citep{nair2021budgeted}, our algorithm achieves better expected simple regret. This is also shown empirically in our experiment in Section \ref{sec: exp simple no}. 
\end{restatable}

\section{Lower Bounds}\label{sec: lowers}
\paragraph{Simple regret:} 
As mentioned earlier, \cite{maiti2022causal} studies a special setting of causal MAB problem with uniform costs ($c_{i,x}=1$ for all $i,x$) in general causal graphs when the objective function of the learner is simple regret. In particular, they showed that there is a large class of causal graphs called tree-graphs, such that for any graph $\mathcal{T}$ in this class with $N$ intervenable nodes and a positive integer $M\leq N$, there exists a joint distribution $P(\cdot)$ compatible\footnote{Compatibility also known as Markov property \citep{pearl2009causality} means that the $P(\cdot)$ factorizes according to the graph $\mathcal{T}$.} with graph $\mathcal{T}$, such that $n(\mathbf{q})=M$ and the expected simple regret of any causal MAB algorithm is $\Omega \big( \sqrt{n(\mathbf{q})/B}\big)$. Comparing this result with the bound introduced in Theorem \ref{th: simple g}, we observe that for some classes of graphs with uniform costs, the expected simple regret obtained by Algorithm \ref{alg: g-simple} differs from the minimum value at most by a factor of $\small{\sqrt{\log\big(NB/n(\mathbf{q})\big)}}$. 

\paragraph{Cumulative regret:}
We prove $\small{\min_{A_B}\max_{\mathcal{C},\mathcal{G}_N,P'} R_c(A_B,\mathcal{G}_N,P,\mathcal{C})\geq}$ $\small{\Omega\big(\sqrt{\lfloor B/c\rfloor KN}\big)}$ in appendix \ref{sec:lower-cum}, where $R_c(A_B,\mathcal{G}_N,P,\mathcal{C})$ denotes the cumulative regret of an adaptive algorithm $A_B$ with total budget $B$ on a causal graph $\mathcal{G}_N$ with $N$ nodes each of which has domain $\{1,...,K\}$. The reward distribution is $P$ and the set of costs is given by $\mathcal{C}=\{1\leq c_a\leq c| a\in\mathcal{A}\}$. 
This shows that for any algorithm there exists a causal bandits problem characterized by $(\mathcal{G}_N,P,\mathcal{C})$ such that it suffers at least $\small{\Omega\big(\sqrt{\lfloor B/c\rfloor KN}\big)}$ of cumulative regret.

\vspace{.2cm}
\section{Experiments}\label{sec: expe}
\begin{wrapfigure}{r}{0.33\textwidth}
\begin{minipage}{0.33\textwidth}
\vspace{-0.5cm}
\begin{figure}[H]
    \begin{thesubfigure}
      \centering
      \includegraphics[width=.95\textwidth]{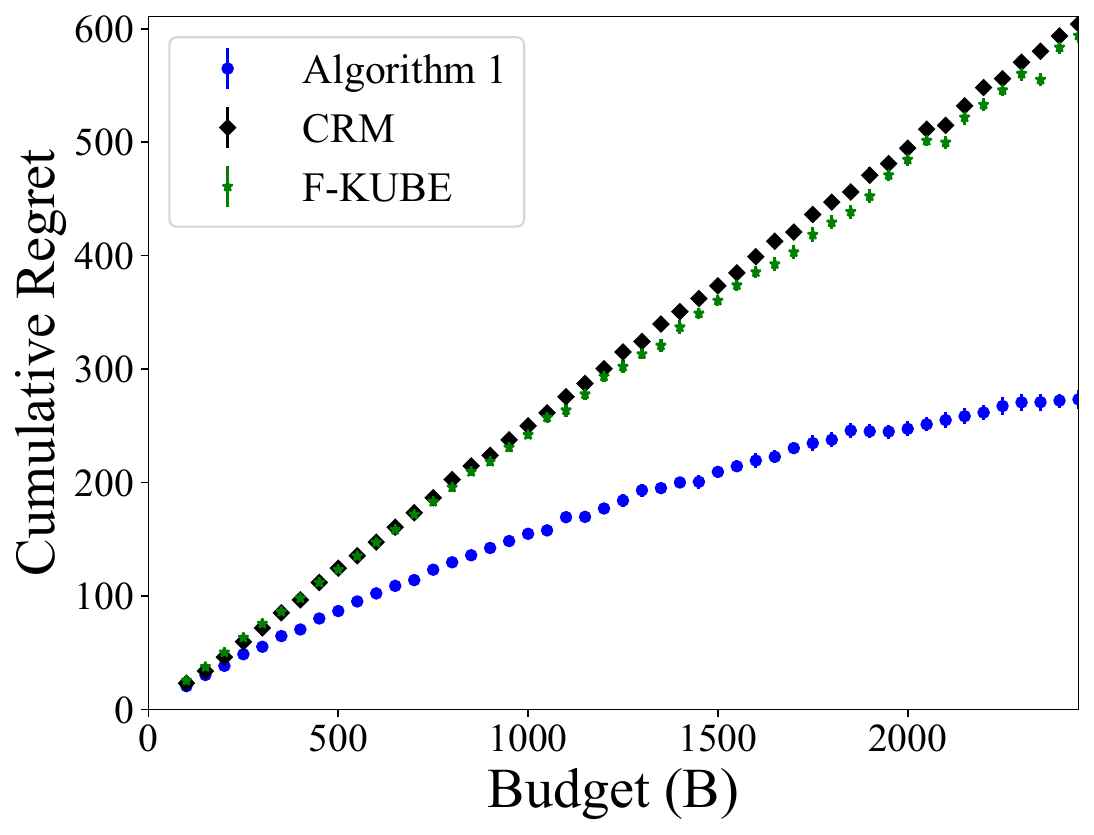}
      % \caption{}
      \label{fig: cumu-BN6}
    \end{thesubfigure}
    
    \begin{thesubfigure}
      \centering
      \includegraphics[width=.95\textwidth]{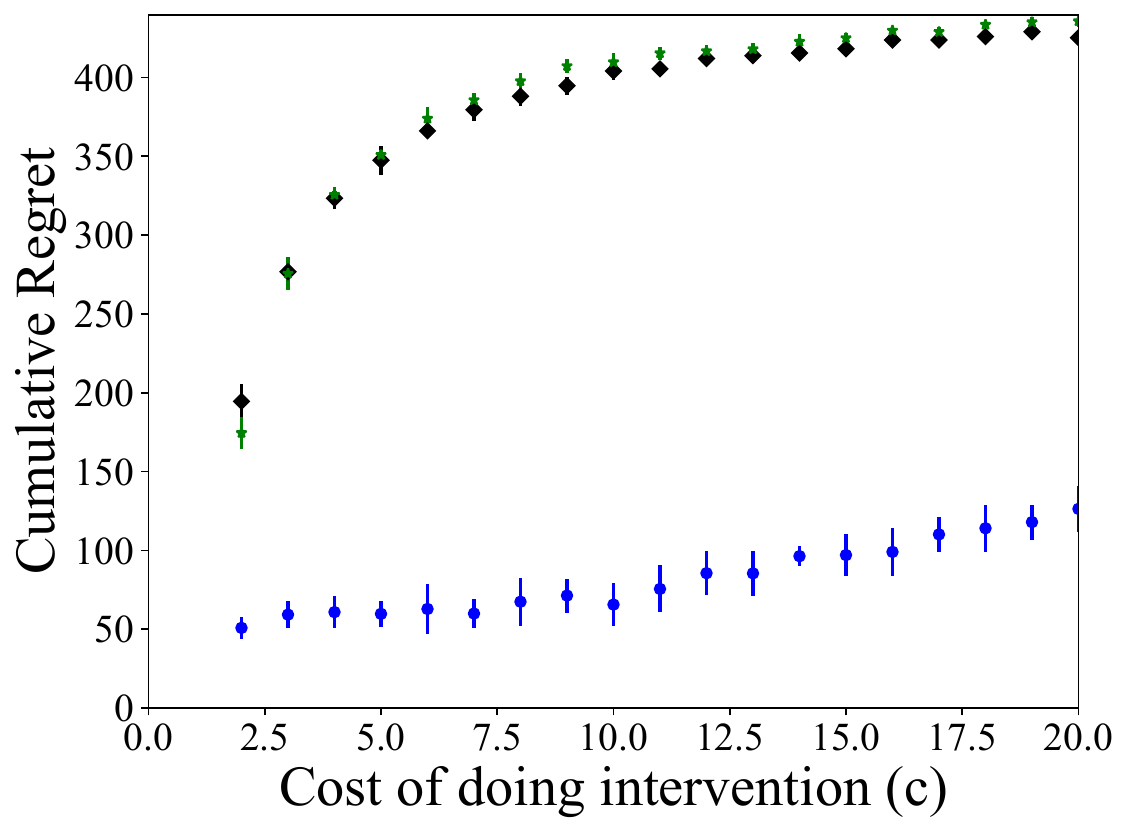}
      % \caption{}
      \label{fig: cumu-cN6}
    \end{thesubfigure}
    % \vspace{-0.5cm}
     \caption{Cumulative regret on a general graph with $N=6$.}
     \vspace{-0.5cm}
    \label{fig: cumu-N6}
\end{figure}
\end{minipage}
\end{wrapfigure}
Herein, we present our empirical evaluations of our algorithm in comparison with state of the art. 
Throughout, each point in figures is obtained as an average of $100$ trials\footnote{Python implementations are provided in the supplementary.}. 

\vspace{.2cm}
\subsection{Cumulative Regret in General Graphs}\label{sec: exp cumu}

In this section, we compared the performance of Algorithm \ref{alg: g-cumulative} with algorithms \textit{CRM} in \citep{maiti2022causal} and \textit{F-KUBE} in \citep{tran2012knapsack}.
\textit{CRM} is a causal MAB algorithm designed for general causal graphs where all of the variables are observable (no hidden confounders exist).
\textit{F-KUBE} is a budgeted MAB algorithm with non-uniform costs that does not use the knowledge of the causal graph.
We used a graph with $6$ intervenable variables, $N=6$, and modeled each $V_i$ with at least a parent in $\mathcal{G}$ to be the XOR of its parents with probability $0.8$ or the XNOR of its parents, otherwise.
Moreover, for each variable $V_i$ without any parents, we modeled $V_i = Be(0.5+0.5 \epsilon)$, where $\epsilon \sim U(0,1)$.
% \vspace{-1.5cm}

% Figure \ref{fig: cumu-BN6} depicts the performance of algorithms in terms of cumulative regret against budget by assuming that the cost of pulling $a_{i,x}$ for $i \in [N], x \in \{0,1\}$ is selected randomly from $\{2,3\}$. 
% As the budget increases, the cumulative regret of all of the algorithms increases.
% However, the growth rate of the cumulative regrets of \textit{F-KUBE} and \textit{CRM} are higher than our algorithm.
The cumulative regret vs. budget plot in Figure \ref{fig: cumu-N6} depicts the performance of algorithms by assuming that the cost of pulling $a_{i,x}$ for $i \in [N], x \in \{0,1\}$ is selected randomly from $\{2,3\}$. 
As the budget increases, the cumulative regret of all of the algorithms increases.
However, the growth rate of the cumulative regrets of \textit{F-KUBE} and \textit{CRM} are higher than our algorithm.
Moreover, since the cumulative regrets of \textit{F-KUBE} and \textit{CRM} do not converge to a constant for $B\leq 2500$, they fail to identify the optimal arm within this budget range while the regret of our algorithm remains a constant for large budgets, which indicates that it could identify the best arm in the experiment.
%Hence, the regret of Algorithm \ref{alg: g-cumulative} converges to a constant for large budgets as it identifies the optimal arm using while the regret of others keeps growing. 

% \begin{figure}[!ht] 
%         \centering
%         \captionsetup{justification=centering}
%         \begin{subfigure}[b]{0.32\textwidth}
%             \centering
%              \includegraphics[width=0.99\textwidth]{Figures/general_cumulative_BN6.pdf}
%             \caption{Cumulative regret vs budget.}
%             \label{fig: cumu-BN6}
%         \end{subfigure}
%         \begin{subfigure}[b]{0.32\textwidth}
%             \centering
%             \includegraphics[width=0.99\textwidth]{Figures/general_cumulative_cN6.pdf}
%             \caption{Cumulative regret vs cost of intervening.}
%             \label{fig: cumu-cN6}
%         \end{subfigure}
%         \caption{Cumulative regret of different algorithms on a general graph with $N=6$ illustrated in Appendix \ref{sec: ad exp cumu}.}
%         \label{fig: cumu-N6}
% \end{figure}
%\vspace{-.2em}

% Figure \ref{fig: cumu-cN6} illustrates the performance of the algorithms in terms of cumulative regret against intervention cost.
% In this experiment the budget was fixed to $1000$ and $c_{i,x}=c$ for all $i \in [N], x \in \{0,1\}$ such that $c \in \{2, 3, \dots, 20\}$(uniform cost for all interventional arms).
The cumulative regret vs. intervention cost in Figure \ref{fig: cumu-N6} illustrates the performance of the algorithms when the budget was fixed to $1000$ and $c_{i,x}=c$ for all $i \in [N], x \in \{0,1\}$ such that $c \in \{2, 3, \dots, 20\}$ (uniform cost for all interventional arms).
As shown in this figure, the cumulative regret of Algorithm \ref{alg: g-cumulative} grows slower than the others.
Note that since \textit{CRM} considers only the causal graphs without hidden variables, for fairness, we compared these algorithms for the graph without hidden variables.
The underlying graph for this experiment and additional experiments on graphs with hidden variables are provided in Appendix \ref{sec: ad exp cumu}.

\vspace{.2cm}
\subsection{Simple Regret in No-backdoor Graphs}\label{sec: exp simple no}

\begin{wrapfigure}{r}{0.33\textwidth}
\begin{minipage}{0.33\textwidth}
\begin{figure}[H]
    \begin{thesubfigure}
      \centering
      \includegraphics[width=.95\textwidth]{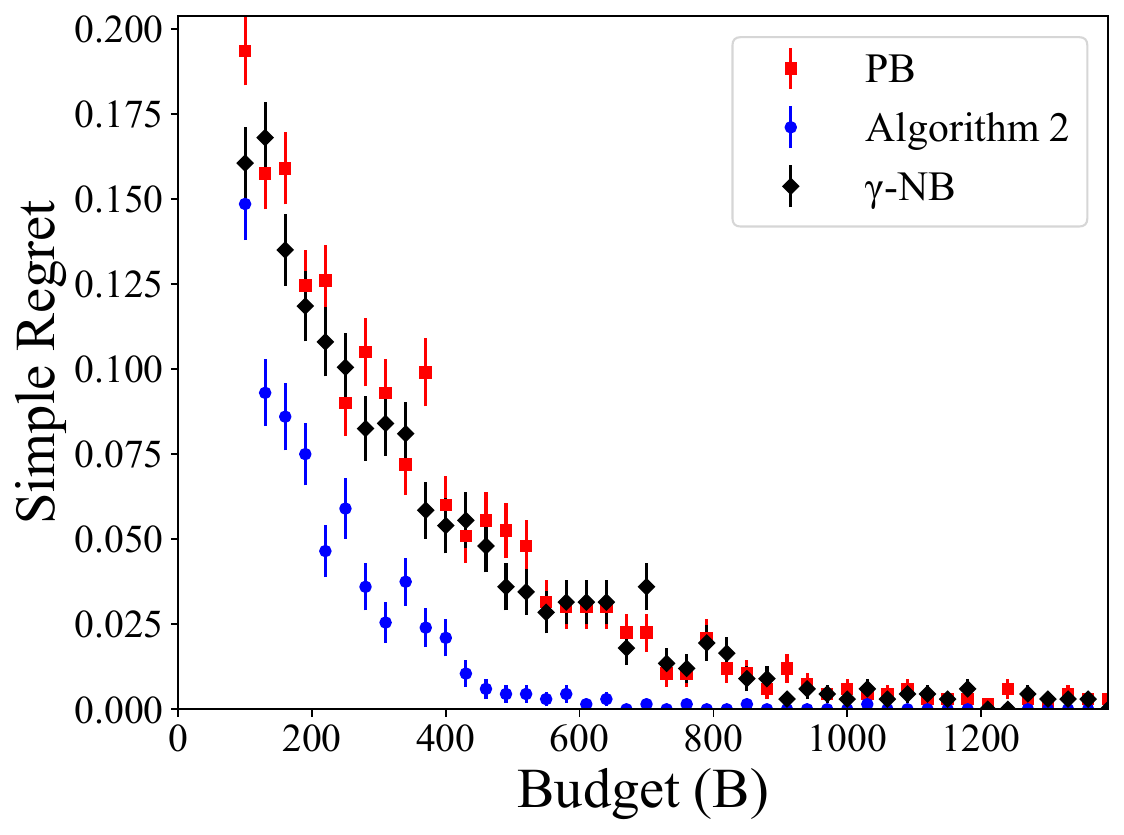}
      % \caption{}
      \label{fig: nobackdoor_ciN50}
    \end{thesubfigure}
    
    \begin{thesubfigure}
      \centering
      \includegraphics[width=.95\textwidth]{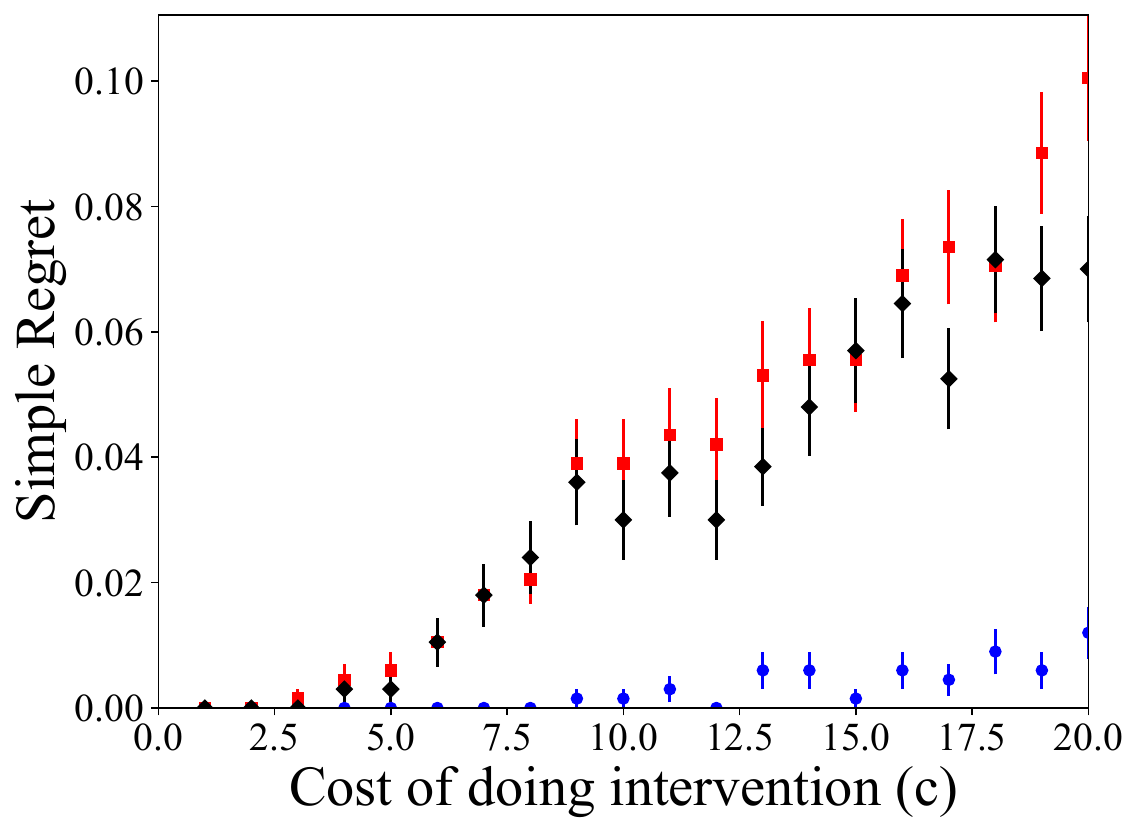}
      % \caption{}
      \label{fig: nobackdoor_cN50}
    \end{thesubfigure}
    % \vspace{-0.5cm}
    \caption{Simple regret on the parallel graph with $N=50$.}
    \label{fig: nobackdoor_N50}
\end{figure}
\end{minipage}
\end{wrapfigure}
In order to be able to compare our algorithm for simple regret with several related works, we studied the causal MAB for the special case of no-backdoor graphs in this section.
We compared the performance of Algorithm \ref{alg: g-simple} with two causal bandit algorithms $\gamma$-NB \citep{nair2021budgeted} and \textit{PB} \citep{lattimore2016causal}.
\textit{PB} is a non-budgeted algorithm that is designed to minimize the simple regret when the graph is no-backdoor.
$\gamma$-NB is a budgeted version of \textit{PB} that allows uniform costs on arms, i.e., $c_{i,x}=c>1$ for all $i$ and $x$.

We used the same setting as in \citep{nair2021budgeted} and \citep{lattimore2016causal} in which the underlying graph has $50$ intervenable variables and all of these variables are parents of the reward variable $Y$. This particular structure is called a parallel graph.
We modeled $X_i \sim Bernoulli(p_i)$ with $p_1=p_2=0.02$ for $i\in\{1,2\}$ and $p_i= 0.5$ for $i \in \{3,\dots,50\}$.
Moreover, we modeled the reward variable as $Y  \sim Bernoulli (\frac{1}{2}+ \epsilon)$ if $X_1=1$, and otherwise, $Y \sim Bernoulli (\frac{1}{2}- \epsilon')$, where $\epsilon=0.3$ and $\epsilon'= \frac{p_1 \epsilon}{1-p_1}$.

% Figure \ref{fig: nobackdoor_ciN50}
The simple regret vs. budget plot in Figure \ref{fig: nobackdoor_N50} was obtained by selecting the cost of pulling each arm $a_{i,x}$ for $i \in [N], x \in \{0,1\}$ randomly from $\{2,3,4,5\}$. 
The simple regrets of all of the algorithms converge to zero as the budget increases, but Algorithm \ref{alg: g-simple} demonstrates faster convergence.
% In Figure \ref{fig: nobackdoor_cN50},
In the simple regret vs. cost of intervention plot in Figure \ref{fig: nobackdoor_N50} we considered the setting in which the budget was fixed to $1500$ and $c_{i,x}=c$ for $i \in [N], x \in \{0,1\}$ such that $c \in \{1, 2, \dots, 20\}$. %In this experiment, $c$ varied between $1$ to $20$ 
The simple regret is increasing in terms of intervention costs as expected.
Since Algorithm \ref{alg: g-simple} uses a different exploration set compared to the others, it drastically outperforms them even in a setting favorable to them.
Additional experiments are presented in Appendix \ref{sec: ad exp simple no} including an experiment using \textit{Successive Rejects} algorithm in \citep{audibert2010best} which is a baseline MAB algorithm\footnote{Successive Rejects is not included in the experiments of the main text as it fails to perform well for large $N$.}.

\vspace{.2cm}
\subsection{Simple Regret in General Graphs}\label{sec: exp simple g} 
% \vspace{-0.5cm}
\begin{wrapfigure}{r}{0.4\textwidth}
\begin{minipage}{0.4\textwidth}
\vspace{-0.cm}
\begin{figure}[H]
    \begin{thesubfigure}
      \centering
      \includegraphics[width=.8\textwidth]{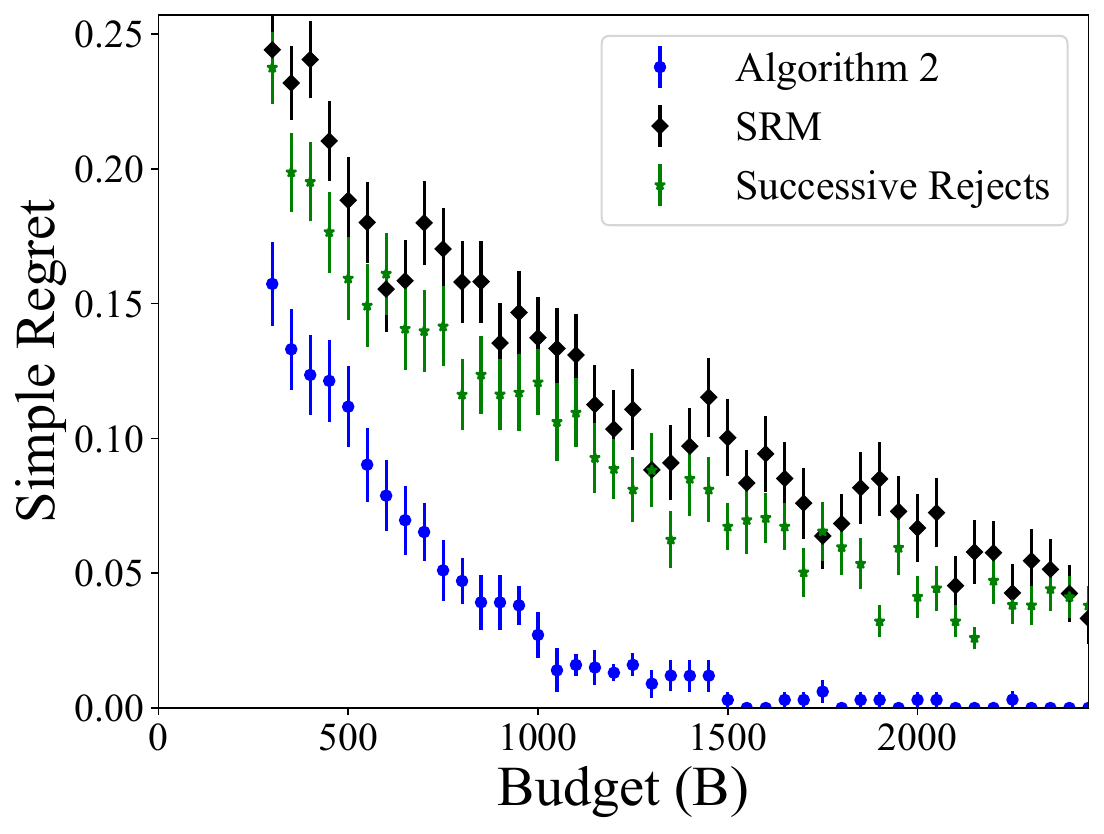}
      % \caption{}
      \label{fig: simgen-BN7}
    \end{thesubfigure}
    
    \begin{thesubfigure}
      \centering
      \includegraphics[width=.8\textwidth]{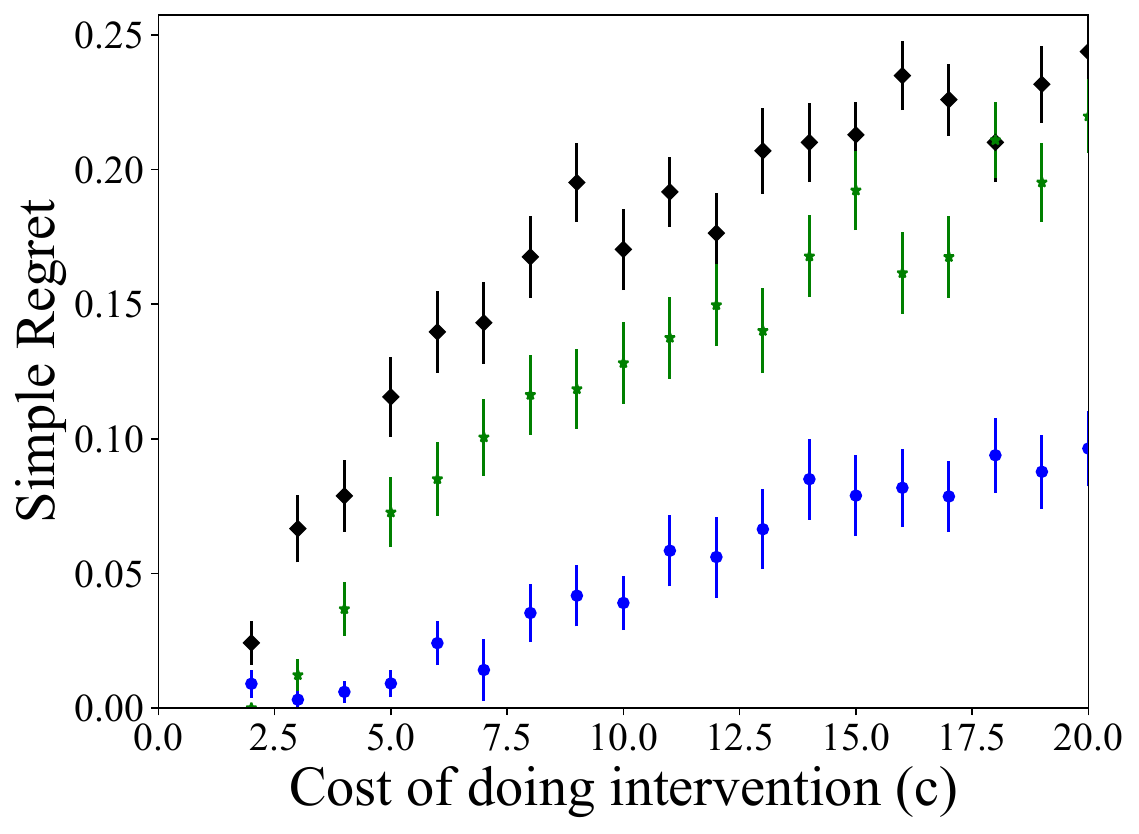}
      % \caption{}
      \label{fig: simgen-cN7}
    \end{thesubfigure}
    \caption{Simple regret on a general graph with $N=7$.}
        \vspace{-0.5cm}
    \label{fig: simgenN7}
\end{figure}
\end{minipage}
\end{wrapfigure}
We compared the performance of Algorithm \ref{alg: g-simple} with two algorithms, $SRM$ \citep{maiti2022causal} and \textit{Successive Rejects} for general graphs in addition to the special structures of the previous section.
$SRM$ is a causal MAB algorithm for minimizing simple regret in the non-budgeted setting where the underlying graph is general.
Here, we used a causal graph that violates the no-backdoor criterion.
%from the intervenable variables to the reward variable despite the experiments of the previous section.
The graph has $N=7$ intervenable variables, i.e., it has $15$ arms ($14$ interventional and one observational arm).
We used the same procedure to construct the model as Section \ref{sec: exp cumu}.

% Figure \ref{fig: simgen-BN7} illustrates the performance of algorithms in terms of simple regret versus budget $B$.
% The cost of each interventional arm was selected randomly from $\{5,6,7\}$.
The simple regret vs. budget plot in Figure \ref{fig: simgenN7} illustrates the performance of algorithms when the cost of each interventional arm was selected randomly from $\{5,6,7\}$.
Algorithm \ref{alg: g-simple} converges to $0$ faster than the others as $B$ grows.
% Figure \ref{fig: simgen-cN7} 
The simple regret vs. the cost of intervention plot in 
Figure \ref{fig: simgenN7} compares the performance of the algorithms when the budget is fixed to $800$, and the cost of all interventional arms is equal to $c$, where $c \in \{2, 3, \dots, 20\}$.
% Note that for $c=1$, our algorithm and $SRM$ perform similarly.
%Analogous to the no-backdoor setting, simple regret increases in terms of the interventional cost, where Algorithm \ref{alg: g-simple} demonstrates superior performance compared to the other two.
Additional experiments and the underlying graph of this experiment are provided in Appendix \ref{sec: ad exp simple g}.

\section{Conclusion}

We studied the budgeted causal MAB problem with non-uniform costs for different arms in general causal graphs in which all intervenable variables have identifiable causal effects. 
We considered two different learners; one with simple regret as its objective and the other with cumulative regret objective.
For each learner, we proposed an algorithm and provided theoretical guarantees.
Furthermore, through empirical studies in different scenarios, we evaluated the performances of our proposed algorithms and showed that they outperform the state-of-the-art.
%We provided theoretical guarantees and empirical analyses to illustrate that our algorithms by considering the non-uniform costs over arms and using the underlying causal graph as a side information, outperform the state of the art.

% \clearpage
% \bibliographystyle{plain}
\bibliography{bibliography.bib}
\appendix
\onecolumn

\newpage
\section{Technical preliminaries}\label{sec: tech}

\begin{definition}[Directed Path]
Let $V_1, V_2, \dots, V_m$ be a set of distinct vertices in an ADMG $\G$.
There is a directed path from $V_1$ to $V_m$ if $V_i \in \mathbf{Pa}(V_{i+1})$ for every $1 \leq i \leq m-1$.
\end{definition}

\begin{definition}[Descendant]
Let $X_i$ and $X_j$ be two vertices in an ADMG $\G$.
$X_j$ is called a descendant of $X_i$ if there exists a directed path from $X_i$ to $X_j$.
\end{definition}

\begin{definition}[Blocked] Given a causal graph $\G$ and two vertices $X_1, X_n \in \mathbf{V}$, a path between $X_1$ and $X_n$ is called blocked by a set of vertices $\mathbf{W}$ (with neither $X_1$ nor $X_n$ in $\mathbf{W}$) whenever there is a vertex $X_k$, such that one of the followings holds:
\begin{itemize}
    \item [(1)] $X_k \in \mathbf{W}$ and $X_{k-1} \to X_k \to X_{k+1}$ or $X_{k-1} \gets X_k \gets X_{k+1}$ or $X_{k-1} \gets X_k \to X_{k+1}$,
    \item [(2)] $X_{k-1} \to X_k \gets X_{k+1}$ and neither $X_k$ nor any of its descendants is in $\mathbf{W}$.
\end{itemize}
    
\end{definition}
\begin{lemma}[Chernoff inequalities]\label{lemma: chernoff}
    Let $X$ be a random variable. Then, for every $s \geq 0$, the followings hold:
    \begin{itemize}
        \item[(1)] $P(X \geq \mathbbm E[X]+s) \leq \min_{\lambda \geq 0}\mathbbm E\big[\exp \big(\lambda(X-\mathbbm E[X])\big)\big]\exp(-\lambda s)$
        \item[(2)] $P(X \leq \mathbbm E[X]-s) \leq \min_{\lambda \geq 0} \mathbbm E\big[\exp \big(\lambda(\mathbbm E[X]-X)\big)\big]\exp(-\lambda s)$
    \end{itemize}
\end{lemma}

\begin{lemma}[Hoeffding inequalities]\label{lemma: hoef}
    Let $X$ be a random variable such that $X \in [a,b]$. Therefore, for every $\lambda \in \mathbb R$:
    \begin{equation*}
        \mathbbm E\Big[\exp \big (\lambda(X - \mathbbm E[X] )\big)\Big] \leq \exp \Big(\frac{\lambda^2(b-a)^2}{8} \Big).
    \end{equation*}
\end{lemma}

\begin{lemma}[Chernoff-Hoeffeding inequality]\label{lemma: ch-h}
    Assume $X^1, \dots, X^T$ are independent random variables such that $ 0 \leq X^t \leq 1$ for $t=1, \dots, T$.
    Then, for every $\epsilon > 0$, the following inequalities hold:
\begin{itemize}
    \item[(1)] $P \Big(\sum_{t \in [T]}X^t - \mathbbm E\big [\sum_{t \in [T]}X^t \big] \geq \epsilon \Big) \leq \exp(\frac{-2\epsilon^2}{T})$,
    \item[(2)] $P \Big(\sum_{t \in [T]}X^t - \mathbbm E\big [\sum_{t \in [T]}X^t \big] \leq -\epsilon \Big) \leq \exp(\frac{-2\epsilon^2}{T})$.
\end{itemize}
\end{lemma}

\section{Proofs of Section \ref{sec: general cumulative} } \label{sec: proof cumu}

In order to prove Theorem \ref{th: cumu}, we require several technical lemmas which we present below.

In the following lemma, we show that $\hat{\mu}_{i,x}^t$ is an unbiased estimator of $\mu_{i,x}$.

\lemmaunbiased*

\begin{proof}
    Recall Equation \eqref{eq: estimate mu}:
    \begin{equation*}
    \hat{\mu}_{i,x}^t:= \frac{\sum_{t' \in \mathbf{I}_{i,j}^{t}}\mathds{1}\{Y^{t'}=1\}+ \sum_{s \in [S_{i,x}^{t}]}Y^{s}_{i,x}}{N_{i,x}^{t}+S_{i,x}^{t}}.
\end{equation*}
Note that $Y^{s}_{i,x}$ is an unbiased estimator of $\mu_{i,x}$ because we partition the time steps that arm $a_0$ was pulled into $|\mathbf{V}|$ different number of subsets.  
Taking expectations from both sides of the above equation yields
\begin{equation*}
    \begin{aligned}
        \mathbbm E [\hat{\mu}_{i,x}^{t}]&= \mathbbm E\left[\frac{\sum_{t' \in \mathbf{I}_{i,j}^{t}}\mathds{1}\{Y^{t'}=1\}+ \sum_{s \in [S_{i,x}^{t}]}Y^{s}_{i,x}}{N_{i,x}^{t}+S_{i,x}^{t}}\right]\\
        &=\sum_{a=1}^{\infty}\sum_{b=0}^{\infty} \mathbbm E\left[\frac{\sum_{t' \in \mathbf{I}_{i,j}^{t}}\mathds{1}\{Y^{t'}=1\}+ \sum_{s \in [S_{i,x}^{t}]}Y^{s}_{i,x}}{N_{i,x}^{t}+S_{i,x}^{t}} \Big|N_{i,x}^{t}=a,S_{i,x}^{t}=b \right]P(N_{i,x}^{t}=a,S_{i,x}^{t}=b)\\
        &=\sum_{a=1}^{\infty}\sum_{b=0}^{\infty} \mathbbm E \Big[\frac{a \mu_{i,x}+b \mu_{i,x}}{a+b}|N_{i,x}^{t}=a,S_{i,x}^{t}=b \Big]P(N_{i,x}^{t}=a,S_{i,x}^{t}=b)\\
        &= \mu_{i,x} \sum_{a=1}^{\infty}\sum_{b=0}^{\infty}P(N_{i,x}^{t}=a,S_{i,x}^{t}=b)= \mu_{i,x}.
    \end{aligned}
\end{equation*}
Similarly, one can show that $\hat{\mu}_{0}^{t}$ is an unbiased estimator of $\mu_0$.
\end{proof}

Next, we show a concentration result for $\hat{\mu}_{i,x}^t$ in \eqref{eq: estimate mu}.
\begin{lemma}\label{lemma: bound mu  N&S}
    For $\hat{\mu}_{i,x}^{t}$ given in Equation \eqref{eq: estimate mu}, we have $P \big(|\hat{\mu}_{i,x}^{t}- \mu_{i,x}| \geq \epsilon \big) \leq 2\exp \big({-2(N_{i,x}^{t}+S_{i,x}^{t})\epsilon^2}\big)$. 
\end{lemma}

\begin{proof}
    \begin{equation}
        \begin{aligned}
            &P(\hat{\mu}_{i,x}^t - \mu_{i,x} \geq \epsilon)= P\left(\frac{\sum_{j \in \mathbf{I}_{i,x}^{t}}\mathds{1}\{Y^{t'}=1\}+ \sum_{s \in [S_{i,x}^{t}]}Y^{s}_{i,x}}{N_{i,x}^{t}+S_{i,x}^{t}} \geq \mu_{i,x}+\epsilon\right)\\
            &= P \left(\sum_{t' \in \mathbf{I}_{i,x}^{t}}\mathds{1}\{Y^{t'}=1\}+ \sum_{s \in [S_{i,x}^{t}]}Y^{s}_{i,x}\geq ({N_{i,x}^{t}+S_{i,x}^{t}} )\mu_{i,x}+({N_{i,x}^{t}+S_{i,x}^{t}} )\epsilon \right)\\
            & \labelrel\leq{eq: cher}  \min_{\lambda \geq 0} \mathbbm E \left[\exp \Big(\lambda \big(\sum_{t' \in \mathbf{I}_{i,x}^{t}}(\mathds{1}\{Y^{t'}=1\}-\mu_{i,x})+ \sum_{s \in [S_{i,x}^{t}]}(Y^{s}_{i,x}-\mu_{i,x})\big) \Big)\right]\exp \big(-\lambda(N_{i,x}^{t}+S_{i,x}^{t})\epsilon \big)\\
             &=\min_{\lambda \geq 0}\mathbbm E \left[ \prod_{t' \in \mathbf{I}_{i,x}^{t}} \exp \big(\lambda(\mathds{1}\{Y^{t'}=1\}-\mu_{i,x})\big) \prod_{s \in [S_{i,x}^{t}]} \exp \big(\lambda(Y^{s}_{i,x}-\mu_{i,x})\big) \right]\exp \big(-\lambda(N_{i,x}^{t}+S_{i,x}^{t})\epsilon \big)\\
            &\labelrel={eq: ind} \min_{\lambda \geq 0}\prod_{t' \in \mathbf{I}_{i,x}^{t}} \mathbbm E \Big[\exp \big(\lambda(\mathds{1}\{Y^{t'}=1\}-\mu_{i,x})\big) \Big] \prod_{s \in [S_{i,x}^{t}]}\mathbbm E \Big[\exp \big(\lambda(Y^{s}_{i,x}-\mu_{i,x})\big) \Big]\exp \big(-\lambda(N_{i,x}^{t}+S_{i,x}^{t}\big)\epsilon)\\
            & \labelrel\leq{eq: hoef} \min_{\lambda \geq 0} \exp \Big(\frac{N_{i,x}^{t}\lambda^2}{8}+\frac{S_{i,x}^{t}\lambda^2}{8}-\lambda(N_{i,x}^{t}+S_{i,x}^{t})\epsilon \Big)\\
            &\labelrel\leq{eq: min} \exp \big(-2(N_{i,x}^{t}+S_{i,x}^{t})\epsilon^2\big).
        \end{aligned}
    \end{equation}    
    
    The inequality in \eqref{eq: cher} holds using Lemma \ref{lemma: chernoff}. 
    The equality in \eqref{eq: ind} is true since the terms in the product are independent.
    The inequality in \eqref{eq: hoef} follows from Lemma \ref{lemma: hoef} where $Y^{t'}, Y_{i,x}^{s} \in \{0,1\}$.
    Finally, the inequality in \eqref{eq: min} follows after substituting the optimal $\lambda:=4\epsilon$.
    In a similar way, it can be shown
    \begin{equation*}
        P\big(\hat{\mu}_{i,x}^{t}-\mu_{i,x} \leq  - \epsilon \big) \leq \exp \big(-2(N_{i,x}^{t}+S_{i,x}^{t})\epsilon^2 \big).
    \end{equation*}
    Therefore, we get 
    \begin{equation*}
        P \big(|\hat{\mu}_{i,x}^{t}-\mu_{i,x} |\geq   \epsilon \big) \leq 2\exp \big(-2(N_{i,x}^{t}+S_{i,x}^{t})\epsilon^2 \big).
    \end{equation*}
\end{proof}

In the following lemma, we introduce a bound for the expectation of $S_{i,x}^{T}$.
Recall that $S_{i,x}^{T}:=\min\left\{\min_{j:V_j \in \mathbf{C}_i}S_{j,i}^{T}, \min_{j:V_j \notin \mathbf{C}_i}\Tilde{S}_{j,i,x}^{T}\right\}$ where
\begin{equation*}
    \begin{cases} 
      S_{j,i}^{T} := \min_{\mathbf{w}_i'}$ $\min_{x'} |\mathbf{O}_{j}^{T}(x',\mathbf{w}_i')| & \text{if}\ V_j \in \mathbf{C}_i, \\
       \\
      \tilde{S}_{j,i,x}^{T} := \min_{\mathbf{w}_i'} |\mathbf{O}_{j}^{T}(x,\mathbf{w}_i')| & \text{if}\ V_j \notin \mathbf{C}_i.
   \end{cases}
\end{equation*}
\begin{lemma}\label{lemma: exp S}
    Let $\mathcal{W}_i$ be the size of the domain set of $\mathbf{V}\setminus \{X_i, Y\}$ and $p_{i,x}:= \min_{j}\min_{\mathbf{w}_i'}p_{j,i,x}(\mathbf{w}_i')$.
    Moreover, define $\tau_{i,x,T} :=\max(0,1-|\mathbf{V}|\mathcal{W}_iT^{-\frac{p_{i,x}^{2}}{2|\mathbf{V}|}})$. 
    Then,
    \begin{equation*}
        \mathbbm E[S_{i,x}^{T}] \geq \frac{p_{i,x}}{2|\mathbf{V}|} \mathbbm E[N_{0}^{T}] \tau_{i,x, T}-1.
    \end{equation*}
\end{lemma}

\begin{proof}
    
    We define
    \begin{equation*}
            \hat{p}^{T}_{j,i,x}(\mathbf{w}_i') := 
            \begin{cases}
                \frac{\min_{x'}|\mathbf{O}^{T}_{j}(x',\mathbf{w}_i')|}{|\mathbf{O}_{j}^{T}|} & \text{ if } V_j \in \mathbf{C}_i, \\
                
                \frac{|\mathbf{O}^{T}_{j}(x,\mathbf{w}_i')|}{|\mathbf{O}_{j}^{T}|} & \text{otherwise.}
            \end{cases}
\end{equation*}
    where $|O_j^T|= \floor* {\frac{N_0^T}{|\mathbf{V}|}}$.
    Moreover, let $\hat{p}^{T}_{i,x}:= \min_{j}\min_{\mathbf{w}_i'}\hat{p}_{j,i,x}^{T}(\mathbf{w}_i')$.
    Using the above definition, we have
    \begin{equation}\label{eq: prob}
        P \Big(\hat{p}^{T}_{j,i,x}(\mathbf{w}_i' ) \leq \frac{p_{i,x}}{2} \Big) \labelrel \leq{eq: smaller} P \Big(\hat{p}^{T}_{j,i,x}(\mathbf{w}_i' ) \leq p_{j,i,x}(\mathbf{w}_i')- \frac{p_{i,x}}{2} \Big) \labelrel \leq{eq: use lemma} \exp \Big(-2\frac{p_{i,x}^2}{4}\frac{\ln T}{|\mathbf{V}|} \Big) =T^{-\frac{p_{i,x}^{2}}{2|\mathbf{V}|}}.
    \end{equation}
    Inequality \eqref{eq: smaller} holds because $\frac{p_{i,x}}{2} \leq p_{j,i,x}(\mathbf{w}_i')- \frac{p_{i,x}}{2}$ and  \eqref{eq: use lemma} follows from Lemma \ref{lemma: ch-h}.
    Therefore, we have
    
    \begin{equation}
        P \Big( \min_{j}\min_{\mathbf{w}_i'}\hat{p}^{T}_{j,i,x}(\mathbf{w}_i') \leq \frac{p_{i,x}}{2} \Big) \leq \sum_{j}\sum_{\mathbf{w}_i'}P \Big(\hat{p}^{T}_{j,i,x}(\mathbf{w}_i') \leq \frac{p_{i,x}}{2} \Big) \labelrel \leq{eq: use eq} |\mathbf{V}|\mathcal{W}_iT^{-\frac{p_{i,x}^{2}}{2|\mathbf{V}|}}.
    \end{equation}
    where $\mathcal{W}_i$ is the alphabet size of $\mathbf{V}\setminus \{X_i, Y\}$ and the inequality \eqref{eq: use eq} follows from Equation \eqref{eq: prob}.
    Finally, from the definition of $S_{i,x}^T$, we get
    \begin{equation*}
        \begin{aligned}
                 \mathbbm E[S_{i,x}^{T}] &\geq
                \mathbbm E \left[\min_{j}\min_{\mathbf{w}_i'} \hat{p}^{T}_{j,i,x}(\mathbf{w}_i') \floor* {\frac{N_0^T}{|\mathbf{V}|}} \right] \\& \geq  \frac{1}{|\mathbf{V}|} \mathbbm E \Big[\min_{j}\min_{\mathbf{w}_i'} \hat{p}^{T}_{j,i,x}(\mathbf{w}_i')(N_{0}^{T}-|\mathbf{V}|) \Big]\\
                &=\frac{1}{|\mathbf{V}|} \mathbbm E \Big[\min_{j}\min_{\mathbf{w}_i'} \hat{p}^{T}_{j,i,x}(\mathbf{w}_i')N_{0}^{T}-\max_{j}\max_{\mathbf{w}_i'}\hat{p}^{T}_{j,i,x}(\mathbf{w}_i')|\mathbf{V}| \Big]\\
                & \geq \frac{1}{|\mathbf{V}|} \mathbbm E \Big[\min_{j}\min_{\mathbf{w}'} \hat{p}^{T}_{j,i,x}(\mathbf{w}_i')N_{0}^{T} \Big]-1\\
                &= \frac{1}{|\mathbf{V}|} \sum_{n=1}^{\infty} n \mathbbm E \Big[\min_{j}\min_{\mathbf{w}'} \hat{p}^{T}_{j,i,x}(\mathbf{w}_i')|N_{0}^{T}=n \Big]P(N_{0}^{T}=n)-1\\
                & \geq \frac{1}{|\mathbf{V}|} \sum_{n=1}^{\infty} n\frac{p_{i,x}}{2} P \Big(\min_{j}\min_{\mathbf{w}_i'} \hat{p}^{T}_{j,i,x}(\mathbf{w}_i') > \frac{p_{i,x}}{2}|N_{0}^{T}=n \Big)P(N_{0}^{T}=n)-1\\
                &\geq \frac{p_{i,x}}{2|\mathbf{V}|} \mathbbm E[N_{0}^{T}]\max(0,1-|\mathbf{V}|\mathcal{W}_i T^{-\frac{p_{i,x}^{2}}{2|\mathbf{V}|}})-1\\
                &=\frac{p_{i,x}}{2|\mathbf{V}|} \mathbbm E[N_{0}^{T}]\tau_{i,x, T}-1.
        \end{aligned}
    \end{equation*}

\end{proof}

\begin{lemma}\label{lemma: exp Ns}
    Suppose that $a_{i,x}$ is not the optimal arm, i.e., $a^* \neq a_{i,x}$. In this case, we have
    \begin{equation*}
             \mathbbm E [N_{i,x}^{T}] \leq  \frac{8\ln T}{\delta_{i,x}^2}+2 -\frac{p_{i,x}}{2|\mathbf{V}|} \mathbbm E[N_{0}^{l}] \tau_{i,x,l} + \frac{\pi^2}{3},
    \end{equation*}
    where $l:= \frac{8\ln T}{\delta_{i,x}^2}+1$.
     Moreover, if $a^* \neq a_{0}$, then,
     \begin{equation*}
         \mathbbm E[N_{0}^{T}] \leq  \frac{8\ln T}{\delta_{0}^2}+1 +\frac{\pi^2}{3}.
     \end{equation*}
\end{lemma}

\begin{proof}
    Define $E_{i,x}^{t}$ to be effective number of pulling arm $a_{i,x}$ at the end of time $t$ and let $E_{i,x}^{t} := N_{i,x}^{t} + S_{i,x}^{t}$. Using this definition, we rewrite $N_{i,x}^T$ as follows
    \begin{equation}\label{eq: Nix}
        N_{i,x}^{T}= \sum_{t \in [l]}\mathds{1}\{a^t= a_{i,x}, E_{i,x}^{t} \leq l \} +\sum_{t \in [l+1,T]}\mathds{1}\{a^t= a_{i,x}, E_{i,x}^{t} > l \}.
    \end{equation}
    Let $m:= \max\{t|E_{i,x}^{t} \leq l \}$, then, the first part of Equation \eqref{eq: Nix} will be equal to $N_{i,x}^{m}$, i.e., 
        \begin{equation*}
        N_{i,x}^{T}= N_{i,x}^{m} +\sum_{t \in [l+1,T]}\mathds{1}\{a^t= a_{i,x}, E_{i,x}^{t} > l \}.
    \end{equation*}
    Since $N_{i,x}^{m}= E_{i,x}^{m}-S_{i,x}^{m}$, we get $N_{i,x}^{m} = \sum_{t \in [m]} \mathds{1}\{a^t=a_{i,x}\}= l-S_{i,x}^{m}$.
    This allows us to rewrite Equation \eqref{eq: Nix} as
    \begin{equation*}
        N_{i,x}^{T}=  l-S_{i,x}^{m} +\sum_{t \in [l+1,T]}\mathds{1}\{a^t= a_{i,x}, E_{i,x}^{t} > l \},
    \end{equation*}
    and since $m \geq l$, we have $S_{i,x}^{m} \geq S_{i,x}^{l}$.
    Therefore,
    \begin{equation}\label{eq: Nixs}
         N_{i,x}^{T} \leq  l-S_{i,x}^{l} +\sum_{t \in [l+1,T]}\mathds{1}\{a^t= a_{i,x}, E_{i,x}^{t} > l \}.
    \end{equation}
    By taking expectation on both sides of Equation \eqref{eq: Nixs} we have 
    \begin{equation*}
        \mathbbm E [N_{i,x}^{T}] \leq  l-\mathbbm E[S_{i,x}^{l}] +\sum_{t \in [l+1,T]}P(a^t= a_{i,x}, E_{i,x}^{t} > l).
    \end{equation*}
    Using Lemma \ref{lemma: exp S}, we rewrite the above inequality as
    \begin{equation}\label{eq: N part 1}
        \mathbbm E [N_{i,x}^{T}] \leq  l+1 -\frac{p_{i,x}}{2|\mathbf{V}|} \mathbbm E[N_{0}^{l}]\tau_{i,x, l} +\sum_{t \in [l+1,T]}P(a^t= a_{i,x}, E_{i,x}^{t} > l).
    \end{equation}
    Next, we bound $\sum_{t \in [l+1,T]}P(a^t= a_{i,x}, E_{i,x}^{t} > l)$,
    \begin{equation*}
        \sum_{t \in [l+1,T]}P(a^t= a_{i,x}, E_{i,x}^{t} > l) \leq \sum_{t \in [l+1,T]}P(\Bar{\mu}_{i,x}^{t-1} \geq \Bar{\mu}_{a^*}^{t-1}, E_{i,x}^{t} > l) =
        \sum_{t \in [l,T-1]}P(\Bar{\mu}_{i,x}^{t} \geq \Bar{\mu}_{a^*}^{t}, E_{i,x}^{t} \geq l).
    \end{equation*}
    For clarity, we use $\hat{\mu}_{a}^{t}(E_{a}^{t})$ instead of $\hat{\mu}_{a}^{t}$.
    By substituting the definitions of the UCB, the right hand side of the equation becomes
      \begin{equation*}
        \begin{aligned}
             &\sum_{t \in [l,T-1]}P \Big(\frac{\hat{\mu}_{i,x}^{t}(E_{i,x}^{t})}{c_{i,x}}  + \sqrt{\frac{2 \ln t}{c_{i,x}^2 E_{i,x}^{t}}} \geq \frac{\hat{\mu}_{a^*}^{t}(E_{a^*}^{t})}{c_{a^*}} +\sqrt{\frac{2 \ln t}{c_{a^*}^2 E_{a^*}^{t}}}, E_{i,x}^{t} \geq l \Big)\\
             \labelrel \leq{eq: min max} & \sum_{t \in [l,T-1]}P \Big(\max_{t_1\in [l+1,t]} \frac{\hat{\mu}_{i,x}^{t}(t_1)}{c_{i,x}}  + \sqrt{\frac{2 \ln t}{c_{i,x}^2 t_1}} \geq \min_{t_2\in [l+1,t]} \frac{\hat{\mu}_{a^*}^{t}(t_2)}{c_{a^*}} + \sqrt{\frac{2 \ln t}{c_{a^*}^2 t_2}} \Big)\\
             \leq & \sum_{t \in [T]} \sum_{t_1 \in[l,t-1]} \sum_{t_2 \in [l,t-1]} P \Big(\frac{\hat{\mu}_{i,x}^{t}(t_1)}{c_{i,x}} + \sqrt{\frac{2 \ln t}{c_{i,x}^2 t_1}} \geq \frac{\hat{\mu}_{a^*}^{t}(t_2)}{c_{a^*}}  + \sqrt{\frac{2 \ln t}{c_{a^*}^2 t_2}} \Big).  
        \end{aligned}
    \end{equation*}
    Inequality \eqref{eq: min max} holds because
    $$\max_{t_1\in [l+1,t]} \hat{\mu}_{i,x}^{t}(t_1) + \sqrt{\frac{2 \ln t}{t_1}} \geq \hat{\mu}_{i,x}^{t}(E_{i,x}^{t}) + \sqrt{\frac{2 \ln t}{E_{i,x}^{t}}}, $$ and $$\min_{t_2\in [l+1,t]} \hat{\mu}_{a^*}^{t}(t_2) + \sqrt{\frac{2 \ln t}{t_2}} \leq \hat{\mu}_{a^*}^{t}(E_{a^*}^{t}) + \sqrt{\frac{2 \ln t}{E_{a^*}^{t}}}. $$
    It can be shown that if none of the following hold, then  $\frac{\hat{\mu}_{i,x}^{t}(t_1)}{c_{i,x}}  + \sqrt{\frac{2 \ln t}{c_{i,x}^2 t_1}} \geq \frac{\hat{\mu}_{a^*}^{t}(t_2)}{c_{a^*}} + \sqrt{\frac{2 \ln t}{c_{a^*}^2 t_2}}$ does not hold as well,
    \begin{equation}\label{eq: event1}
       \hat{\mu}_{i,x}^{t}(t_1)- \mu_{i,x} \geq \sqrt{\frac{2\ln t}{t_1}},
    \end{equation}
    \begin{equation}\label{eq: event2}
       \hat{\mu}_{a^*}^{t}(t_2)- \mu_{a^*} \leq -\sqrt{\frac{2\ln t}{t_2}},
    \end{equation}
    \begin{equation}\label{eq: event3}
       \frac{\mu_{a^*}}{c_{a^*}}-\frac{\mu_{i,x}}{c_{i,x}} \leq 2\sqrt{\frac{2\ln t}{c_{i,x}^2 t_1}}.
    \end{equation}
    Now, we bound the probability of events in Equations \eqref{eq: event1} and \eqref{eq: event2},
    \begin{equation*}
        P\big( \hat{\mu}_{i,x}^{t}(t_1)- \mu_{i,x} \geq \sqrt{\frac{2\ln t}{t_1}}\big) \leq \exp\big({-2\frac{2\ln t}{t_1}t_1}\big)= t^{-4},
    \end{equation*}
    \begin{equation*}
        P\big( \hat{\mu}_{a^*}^{t}(t_2)- \mu_{a^*} \leq -\sqrt{\frac{2\ln t}{t_2}}\big) \leq \exp\big({-2\frac{2\ln t}{t_2}t_2}\big)= t^{-4},
    \end{equation*}
    where we used Lemma \ref{lemma: ch-h} to obtain the both above inequalities.
    Furthermore, by assuming that $l:= \frac{8\ln T}{\delta_{i,x}^2}+1$, the event in Equation $\eqref{eq: event3}$ is false,
    \begin{equation}\label{eq: N part 2}
        \sum_{t \in [l+1,T]}P(a^t= a_{i,x}, E_{i,x}^{t} > l) \leq \sum_{t \in [T]} \sum_{t_1 \in[l,t-1]} \sum_{t_2 \in [l,t-1]} 2t^{-4} \leq \frac{\pi^2}{3}.
    \end{equation}
    Therefore, if $a^* \neq a_{i,x}$, using Equations \eqref{eq: N part 1} and \eqref{eq: N part 2}, we obtain the following bound for $N_{i,x}^{T}$:
    \begin{equation*}
         \mathbbm E [N_{i,x}^{T}] \leq  \frac{8\ln T}{\delta_{i,x}^2}+2 -\frac{p_{i,x}}{2|\mathbf{V}|} \mathbbm E[N_{0}^{l}] \tau_{i,x, l} + \frac{\pi^2}{3}.
    \end{equation*}
    %where $l= \frac{8\ln T}{\delta_{i,x}^2}+1$.
    
    For the second part of the proof, suppose that $a^* \neq a_0$. 
    In this case, we decompose $N_{0}^{T}$ in two parts,
    \begin{equation*}
        N_{0}^{T}= \sum_{t \in [T]} \mathds{1}\{a^t = a_0\}= \sum_{t \in [l]}\mathds{1}\{a^t=a_0, N_0^T \leq l\}+ \sum_{t \in [l+1,T]}\mathds{1}\{a^t=a_0, N_0^T > l\}.
    \end{equation*}
    By taking expectation from both sides of the above inequality, we get
    \begin{equation*}
    \begin{aligned}
        \mathbbm E [N_{0}^{T}]&=  \sum_{t \in [l]}P(a^t=a_0 , N_0^T \leq l)+ \sum_{t \in [l+1,T]}P(a^t=a_0 , N_0^T > l)\\ & \leq l + \sum_{t \in [l+1,T]}P(a^t=a_0, N_{0}^{t} > l).
    \end{aligned}
    \end{equation*}
    Following the same procedure used to bound $\sum_{t \in [l+1,T]}P(a^t= a_{i,x}, E_{i,x}^{t} > l)$ and  for $l= \frac{8\ln T}{\delta_{0}^2}+1$, we obtain
    \begin{equation*}
        \sum_{t \in [l+1,T]}P(a^t=a_0, N_{0}^{t} > l) \leq \frac{\pi^2}{3}.
    \end{equation*}
    This implies the following bound for $\mathbbm E [N_{0}^{T}]$,
    \begin{equation*}
         \mathbbm E [N_{0}^{T}] \leq  \frac{8\ln T}{\delta_{0}^2}+1 +\frac{\pi^2}{3}.
    \end{equation*}
    \end{proof}

    \begin{lemma}\label{lemma: exp N0}
        If $a^*= a_0$, we have the following bound for $\mathbbm E[N_{0}^{T}]$,
        \begin{equation*}
            \mathbbm E[N_{0}^T] \geq T-2N (2+\frac{ \pi^2}{3})-\sum_{i,x}\frac{8\ln T}{\delta_{i,x}^2}.
        \end{equation*}
    \end{lemma}
    \begin{proof}
        From definition of $N_a^T$, we know $N_0^T= T- \sum_{i,x}N_{i,x}^{T}$. 
        If $a^* \neq a_{i,x}$ we have the following by Lemma \ref{lemma: exp Ns}:
        \begin{equation*}
             \mathbbm E[N_{i,x}^{T}] \leq \frac{8\ln T}{\delta_{i,x}^2}+2+ \frac{\pi^2}{3}.
        \end{equation*}
        Then, 
        \begin{equation*}
            \mathbbm E[N_0^T] \geq T-2N (2+\frac{\pi^2}{3})-\sum_{i,x}\frac{8\ln T}{\delta_{i,x}^2}.
        \end{equation*}
    \end{proof}
   
    \begin{lemma}\label{lemma: prob mu}
        Suppose that $a^* \neq a_0$ and let $\delta_0 := \frac{\mu_{a^*}}{c_{a^*}} - \mu_0$. We also define
    \begin{equation*}
            \hat{p}^{T}_{j,i,x}(\mathbf{w}_i') := 
            \begin{cases}
                \frac{\min_{x'}|\mathbf{O}^{T}_{j}(x',\mathbf{w}_i')|}{|\mathbf{O}_{j}^{T}|} & \text{ if } V_j \in \mathbf{C}_i, \\
                
                \frac{|\mathbf{O}^{T}_{j}(x,\mathbf{w}_i')|}{|\mathbf{O}_{j}^{T}|} & \text{otherwise.}
            \end{cases}
\end{equation*}

    Moreover, let $\hat{p}^{T}_{i,x}:= \min_{j}\min_{\mathbf{w}_i'}\hat{p}_{j,i,x}^{T}(\mathbf{w}_i')$, $p_{i,x}:= \min_{j}\min_{\mathbf{w}_i'}p_{j,i,x}(\mathbf{w}_i')$, and $p:= \min_{i,x} p_{i,x}$. Then,
        \begin{equation*}
            P \Big(|\hat{\mu}_0^T- \mu_0| \geq \frac{\delta_0}{4} \Big) \leq 2T^{-\frac{\delta_0^2}{8}},
        \end{equation*}
        and 
        \begin{equation*}
             P \Big(|\frac{\hat{\mu}_{i,x}^{T}}{c_{i,x}} -\frac{\mu_{i,x}}{c_{i,x}}| \geq \frac{\delta_0}{4} \Big) \leq  |\mathbf{V}|\mathcal{W}_iT^{-\frac{p^{2}}{2|\mathbf{V}|}} + 2T^{-\frac{p\delta_0^2 c_{i,x}^2}{16|\mathbf{V}|}},
        \end{equation*}
        where $\mathcal{W}_i$ is the size of domain from which $\mathbf{V}\setminus \{X_i, Y\}$ takes values.
    \end{lemma}
    
    \begin{proof}
    By Algorithm \ref{alg: g-cumulative}, the number of times that arm $a_0$ is pulled at the end of $T$ rounds, denoted by $N_0^T$, is at least $\beta^2 \ln T$ and since $\beta \geq 1$, we have $N_0^T \geq \ln T$. 
    Using Lemma \ref{lemma: ch-h}, we 
    \begin{equation*}
        P \Big(|\hat{\mu}_0^T- \mu_0| \geq \frac{\delta_0}{4} \Big) \leq 2 \exp(-2\frac{\delta_0^2}{16}\ln T ) = 2T^{-\frac{\delta_0^2}{8}}.
    \end{equation*}

    Next, we prove the second inequality as follows.
    \begin{equation}\label{eq: compo mu}
        \begin{aligned}[b]
             P \Big(|\frac{\hat{\mu}_{i,x}^{T}}{c_{i,x}} -\frac{\mu_{i,x}}{c_{i,x}} | \geq \frac{\delta_0}{4} \Big) &= P \Big(|\frac{\hat{\mu}_{i,x}^{T}}{c_{i,x}} -\frac{\mu_{i,x}}{c_{i,x}}| \geq \frac{\delta_0}{4}, \hat{p}_{i,x}^{T} \leq \frac{p}{2} \Big) + P \Big(|\frac{\hat{\mu}_{i,x}^{T}}{c_{i,x}} -\frac{\mu_{i,x}}{c_{i,x}}| \geq \frac{\delta_0}{4}, \hat{p}_{i,x}^{T} > \frac{p}{2} \Big)\\
             & \leq P(\hat{p}_{i,x}^{T} \leq \frac{p}{2})+ P \Big(|\frac{\hat{\mu}_{i,x}^{T}}{c_{i,x}} -\frac{\mu_{i,x}}{c_{i,x}}| \geq \frac{\delta_0}{4}, \hat{p}_{i,x}^{T} > \frac{p}{2} \Big)
        \end{aligned}
    \end{equation}
    
    Now, we bound the first part in Equation \eqref{eq: compo mu}.
    To do so, first, we get
    \begin{equation}\label{eq: probb}
        P \Big(\hat{p}^{T}_{j,i,x}(\mathbf{w}_i') \leq \frac{p}{2} \Big) \labelrel \leq{eq: smallerr} P \Big(\hat{p}^{T}_{j,i,x}(\mathbf{w}_i') \leq p_{j,i,x}(\mathbf{w}_i')- \frac{p}{2} \Big) \labelrel \leq{eq: use lemmaa} \exp(-2\frac{p^2}{4}.\frac{\ln T}{|\mathbf{V}|}) =T^{-\frac{p^{2}}{2|\mathbf{V}|}}.
    \end{equation}
    Note that in Equation \eqref{eq: probb}, the inequality in \eqref{eq: smallerr} holds because $\frac{p}{2} \leq p_{j,i,x}(\mathbf{w}')- \frac{p}{2}$ and the inequality in \eqref{eq: use lemmaa} follows from Lemma \ref{lemma: ch-h}.
    Therefore,
    \begin{equation} \label{eq: part 1}
        P(\hat{p}_{i,x}^{T} \leq \frac{p}{2})= P\big(\min_{j}\min_{\mathbf{w}_i'}\hat{p}^{T}_{j,i,x}(\mathbf{w}_i') \leq \frac{p}{2}\big) \leq \sum_{j}\sum_{\mathbf{w}_i'}P\big(\hat{p}^{T}_{j,i,x}(\mathbf{w}_i') \leq \frac{p}{2}\big) \labelrel \leq{eq: use eqq} |\mathbf{V}|\mathcal{W}_i T^{-\frac{p^{2}}{2|\mathbf{V}|}},
    \end{equation}
    where the inequality in \eqref{eq: use eqq} follows from Equation \eqref{eq: probb}.

    Next, we bound the second part of Equation \eqref{eq: compo mu}.
    From Algorithm \ref{alg: g-cumulative}, we have $\beta \geq 1$, and therefore, $N_0^T \geq \ln T$.
    Now, if $\hat{p}_{i,x}^T > \frac{p}{2}$, then $S_{i,x}^{T} > \frac{p}{2} \frac{N_0^T}{|\mathbf{V}|} \geq \frac{p}{2|\mathbf{V}|} \ln T$.
    Therefore, using Lemma \ref{lemma: bound mu  N&S} and \ref{lemma: ch-h}  we have the following bound for the second part of Equation \eqref{eq: compo mu}:
    \begin{equation}\label{eq: part 2}
        P \Big(|\frac{\hat{\mu}_{i,x}^{T}}{c_{i,x}} -\frac{\mu_{i,x}}{c_{i,x}}| \geq \frac{\delta_0}{4}, \hat{p}_{i,x}^{T} > \frac{p}{2} \Big) \leq 2 \exp\big(-2. \frac{\delta_0^2 c_{i,x}^2}{16}\frac{p}{2|\mathbf{V}|}\ln T\big)= 2T^{-\frac{p\delta_0^2 c_{i,x}^2}{16|\mathbf{V}|}}.
    \end{equation}
    
    Finally, using Equations \eqref{eq: part 1} and \eqref{eq: part 2}, we rewrite Equation \eqref{eq: compo mu} as follows,
    \begin{equation*}
         P \Big(|\frac{\hat{\mu}_{i,x}^{T}}{c_{i,x}} -\frac{\mu_{i,x}}{c_{i,x}}| \geq \frac{\delta_0}{4} \Big) \leq |\mathbf{V}|\mathcal{W}_i T^{-\frac{p^{2}}{2|\mathbf{V}|}} + 2T^{-\frac{p\delta_0^2 c_{i,x}^2}{16|\mathbf{V}|}}.
    \end{equation*} 
    \end{proof}

    \begin{lemma} \label{lemma: exp beta}
        Assume that $a^* \neq a_0$ and let $\delta_0 =\frac{\mu_{a^*}}{c_{a^*}}  - \mu_0$.
        If $T \geq \max\big\{e^{\frac{32}{\delta_0}}, \argmin_{t}\{t^{\frac{p\delta_0^2}{16|\mathbf{V}|}} \geq \frac{8N(3+ |\mathbf{V}|\mathcal{W})}{3}\}\big\}$,
        where $\mathcal{W}:= \max_{i}\mathcal{W}_i$, then $\mathbbm E[\beta^2] \geq \frac{8}{9\delta_0}$.
    \end{lemma}
    
    \begin{proof}
        For each arm $a \in \mathcal{A}$, let $ e_a$ be the event that $|\frac{\hat{\mu}_a^T}{c_a} -\frac{\mu_a}{c_a} | \leq \frac{\delta_0}{4}$ and define $e:= \cap_{a \in \mathcal{A}}e_a$.
    Furthermore, let $\Bar{e}_a$ and $\Bar{e}$ denote the compliment of the events $e_a$ and $e$, respectively.
 Lemma \ref{lemma: prob mu} implies the following inequalities for $a_0$ and every $a_{i,x} \in \mathcal{A}$, 
    \begin{equation*}
        P(\Bar{e}_0) \leq 2T^{-\frac{\delta_0^2}{8}},
    \end{equation*}
    \begin{equation*}
        P(\Bar{e}_{i,x})  \leq |\mathbf{V}|\mathcal{W}_i T^{-\frac{p^{2}}{2|\mathbf{V}|}} + 2T^{-\frac{p \delta_0^2 c_{i,x}^2}{16|\mathbf{V}|}}.
    \end{equation*}
    Therefore, using the above equations and the union bound, we get
    \begin{equation*}
        \begin{aligned}
             P(\Bar{e}) &\labelrel \leq{eq: smaller c} 2T^{-\frac{\delta_0^2}{8}}+ 2N\big(|\mathbf{V}|\mathcal{W}T^{-\frac{p^{2}}{2|\mathbf{V}|}} + 2T^{-\frac{p \delta_0^2}{16|\mathbf{V}|}}\big)\\
             &\labelrel \leq{eq: smaller p&}  2NT^{-\frac{p\delta_0^2  }{16|\mathbf{V}|}}+ 2N\big( |\mathbf{V}|\mathcal{W}T^{-\frac{p\delta_0^2}{16|\mathbf{V}|}} + 2T^{-\frac{p\delta_0^2}{16|\mathbf{V}|}}\big)\\
             & = 2N(3+|\mathbf{V}|\mathcal{W}) T^{-\frac{p\delta_0^2 }{16|\mathbf{V}|}},
        \end{aligned}
    \end{equation*}
    where the inequality in \eqref{eq: smaller c} holds since $c_{i,x} \geq 1$ for every $i \in [N] , x \in \{0,1\}$ and \eqref{eq: smaller p&} holds since $p \leq 1$, $\delta_0 \leq 1$.

    Let $\hat{\mu}_{\Tilde{a}}^T := \max_{a \in \mathcal{A}}\frac{ \hat{\mu}_{a}^T}{c_a}$. 
    By the definition of $\Tilde{a}$ and $\delta_0$, we have $\frac{\mu_{\Tilde{a}}}{c_{\Tilde{a}}} -\mu_0 \leq \delta_0 $.
    If event $e$ is true, then
    \begin{equation*} 
         -\frac{\delta_0}{2} \leq \frac{\hat{\mu}_{\Tilde{a}}^T}{c_{\Tilde{a}}} -\hat{\mu}_0^T+(\mu_0 - \frac{\mu_{\Tilde{a}}}{c_{\Tilde{a}}} ) \leq \frac{\delta_0}{2}
    \end{equation*}
    \begin{equation} \label{eq: bound mu-}
         \Longrightarrow \frac{\hat{\mu}_{\Tilde{a}}^T}{c_{\Tilde{a}}} -\hat{\mu}_0^T \leq \frac{3\delta_0}{2}.
    \end{equation}
    From Algorithm \ref{alg: g-cumulative}, at time $T$, we have $\beta = \min\{\frac{2\sqrt{2}}{\hat{\mu}^T_{\Tilde{a}}/{c_{\Tilde{a}}}-\hat{\mu}_{0}^T}, \sqrt{\log{T}}\}$. 
    By assuming $\log T \geq \frac{32}{\delta_0^2}$ and using Equation \eqref{eq: bound mu-}, we have the following bound at the end of round $T$
    \begin{equation*}
       \beta^2 \geq \frac{32}{9 \delta_0^2} .
    \end{equation*}
   
   This implies that if event $e$ is true, then, $\beta^2 \geq \frac{32}{9\delta_0^2}$.
   Now, we bound $\mathbbm E[\beta^2]$:
   \begin{equation*}
       \mathbbm E[\beta^2] \geq \frac{32}{9\delta_0^2} P(e) \geq \frac{32}{9\delta_0^2} \Big(1-2N(3+|\mathbf{V}|\mathcal{W}) T^{-\frac{p\delta_0^2 }{16|\mathbf{V}|}} \Big).
   \end{equation*}
    Since $T^{\frac{p\delta_0^2}{16 |\mathbf{V}|}} \geq \frac{8N(3+|\mathbf{V}|\mathcal{W})}{3}$
    , then,
    \begin{equation*}
        \mathbbm E[\beta^2] \geq \frac{8}{9\delta_{0}^2}.
    \end{equation*}
   
    \end{proof}

    We are now ready to prove Theorem \ref{th: cumu}. 
    \theoremcumu*
     \begin{proof}
          Let $T_B$ denote the number of rounds that Algorithm \ref{alg: g-cumulative} plays the arms before the budget $B$ is exhausted.
          Therefore, we know $\frac{B}{\max_{a}c_a} \leq T_B \leq B$.
          Using Lemma \ref{lemma: exp Ns}, we get the following for $T_B$ satisfying Lemma \ref{lemma: exp beta}:
    \begin{equation}\label{eq: Ex N}
        \begin{aligned}[b]
            \mathbbm E[R_c(B)] &  \leq \delta_0 \Big( \frac{8\ln T_B}{\delta_{0}^2}+1 +\frac{\pi^2}{3} \Big)+  \sum_{\delta_{i,x}>0}\delta_{i,x} \Big(\frac{8\ln T_B}{\delta_{i,x}^2}+2 -\frac{p_{i,x}}{2|\mathbf{V}|} \mathbbm E[N_{0}^{l}]  \tau_{i,x, l} + \frac{\pi^2}{3} \Big)\\
             &\labelrel \leq{eq: exN} \delta_0 \Big( \frac{8\ln T_B}{\delta_{0}^2}+1 +\frac{\pi^2}{3} \Big) + \sum_{\delta_{i,x}>0}\delta_{i,x} \Big(\frac{8\ln T_B}{\delta_{i,x}^2}+2 -\frac{p_{i,x}}{2|\mathbf{V}|} \mathbbm E[\beta^2] \cdot \ln l \cdot \tau_{i,x, l} + \frac{\pi^2}{3} \Big)\\
             &\labelrel \leq{eq: exb} \delta_0 \Big( \frac{8\ln T_B}{\delta_{0}^2}+1 +\frac{\pi^2}{3} \Big)+ \sum_{\delta_{i,x}>0}\delta_{i,x} \Big(\frac{8\ln T_B}{\delta_{i,x}^2}+2 -\frac{8 p_{i,x}}{18\delta_0^2|\mathbf{V}|} \ln l  \cdot\tau_{i,x, l} + \frac{\pi^2}{3} \Big)\\
             &\leq \delta_0 \Big( \frac{8\ln B}{\delta_{0}^2}+1 +\frac{\pi^2}{3} \Big)+ \sum_{\delta_{i,x}>0}\delta_{i,x} \Big(\frac{8\ln B}{\delta_{i,x}^2}+2 -\frac{8 p_{i,x}}{18\delta_0^2|\mathbf{V}|} \ln b  \cdot\tau_{i,x, b} + \frac{\pi^2}{3} \Big),\\
        \end{aligned}
    \end{equation}
    where $l= \frac{8\ln T_B}{\delta_{i,x}^2}+1$,  $ b =  \frac{8}{\delta_{i,x}^2} \ln (\frac{B}{\max_{a}c_a})+1$,  $\tau_{i,x, l} =\max\{0,1-|\mathbf{V}|\mathcal{W}_il^{-\frac{p_{i,x}^{2}}{2|\mathbf{V}|}}\}$, $\tau_{i,x, b} =\max\{0,1-|\mathbf{V}|\mathcal{W}_i b^{-\frac{p_{i,x}^{2}}{2|\mathbf{V}|}}\}$.
    Furthermore, the inequality in \eqref{eq: exN} follows from the fact $\mathbbm E[N_{0}^{l}]\geq \mathbbm E[\beta^2] \ln l$ and the inequality in \eqref{eq: exb} follows from Lemma \ref{lemma: exp beta}.
    The last inequality holds since $\frac{B}{\max_{a}c_a} \leq T_B \leq B$.
     \end{proof}

     \section{Estimating the expected reward from observational distribution}\label{sec: estimation mu}
In this section, we use a procedure (proposed by \cite{bhattacharyya2020learning,maiti2022causal}) to compute $\hat{\mu}_{i,x}$ for each $a_{i,x}$ using the observational data obtained by pulling arm $a_0$.
Algorithm \ref{alg: compute mu} summarizes the steps of this procedure.
\begin{algorithm}[tb]
\caption{Compute $\hat{\mu}_{i,x}$ using observational samples}
\label{alg: compute mu}
\textbf{Input}: ADMG $\G$, observational samples, indices $i \in [N]$ and $x \in \{0,1\}$
\begin{algorithmic}[1] %[1] enables line numbers
\STATE{Reduce ADMG $\G$ to $\mathcal{H}_i$ using Algorithm
\ref{alg: reduce hi};}
\STATE{Compute $\hat{D}_{i,x}$ by Algorithm \ref{alg: learn dix} (use $\mathcal{H}_i$ as input);}
\STATE{Use Equation \eqref{eq: comp mu di} to compute $\hat{\mu}_{i,x}$;} 
\STATE{\textbf{Return:} $\hat{\mu}_{i,x}$.}
\end{algorithmic}
\end{algorithm}

Algorithm \ref{alg: compute mu} takes the underlying causal graph $\G$, observational data that were collected by pulling arm $a_0$ for the first $\frac{B}{2}$ rounds and indices $i,x$ as inputs.
In line 1, it reduces $\G$ to $\mathcal{H}_i$ using Algorithm \ref{alg: reduce hi}.
Then, it computes the Bayes-net $\hat{D}_{i,x}$ in $\mathcal{H}_i$ using Algorithm \ref{alg: learn dix} proposed by \cite{bhattacharyya2020learning} in line 2.
Finally, using $\hat{D}_{i,x}$, it computes $\hat{\mu}_{i,x}$ by substituting the distribution $P_{D_{i,x}}(Y=1, \mathbf{V}'= \mathbf{v}')$ in the following equation with its empirical estimation,
\begin{equation}\label{eq: comp mu di}
    \mu_{i,x}= P_{\mathcal{G}}(Y=1| do(X_i)=x) = P_{\mathcal{H}_i}(Y=1| do(X_i)=x) = \sum_{\mathbf{v}'} P_{D_{i,x}}(Y=1, \mathbf{V}'= \mathbf{v}'),
\end{equation}
where $\mathbf{V}'$ is the set of variables in $\mathcal{H}_i$ except $\{X_i,Y\}$ and $\mathbf{v}'$ is an arbitrary realization of $\mathbf{V}'$.
\begin{algorithm}[tb]
\caption{Reducing $\mathcal{G}$ to $\mathcal{H}_i$}
\label{alg: reduce hi}
\textbf{Input}: Causal garph $\G$ and index $i \in [N]$.
\begin{algorithmic}[1] %[1] enables line numbers
\STATE{ Let $\mathbf{W}_i= X_i \cup \widetilde{\mathbf{Pa}}(X_i) \cup Y$}
\STATE{Let $\G_i$ be the graph obtained from $\G$ by considering $\mathbf{V} \setminus \mathbf{W}_i$ as hidden variables.}
\STATE{Using \textbf{Projection Algorithm} \citep{10.5555/2073876.2073938, 10.5555/647231.719420} do the following steps:
\begin{itemize}
    \item[(a)] For each observable variable $V_j \in \mathbf{V}$ in $\G_i$, add an observable variable $V_j$ in $\mathcal{H}_i$.
    \item[(b)] For each pair of observable variables $V_j, V_k \in \mathbf{V}$ in $\G_i$, add a directed edge from $V_j$ to $V_k$ in $\mathcal{H}_i$ if one of the followings hold:
    
    1) There exists a directed edge from $V_j$ to $V_k$ in $\G_i$, or
    
    2) There exists a directed path from $V_j$ to $V_k$ in $\G_i$ such that it contains only unobservable variables.
    \item[(c)] For each pair of observable variables $V_j , V_k \in \mathbf{V}$ in $\G_i$, add a bidirected edge between $V_j$ and $V_k$ in $\mathcal{H}_i$ if there exists an unobservable variable $U$ in $\G_i$ such that there exist two paths from $U$ to $V_j$ and from $U$ to $V_k$ in $\G_i'$ such that both paths contain only unobservable variables.
\end{itemize}}
\STATE{\textbf{Return}: $\mathcal{H}_i$}
\end{algorithmic}
\end{algorithm}

\begin{algorithm}[tb]
\caption{Computing $\hat{D}_{i,x}$}
\label{alg: learn dix}
\textbf{Input}: Observational samples, $i$, $x$, $\mathcal{H}_i$ (with set of vertices $\mathbf{V}_i$) and parameter $t$.
\begin{algorithmic}[1] %[1] enables line numbers
\FOR{every variable $V_j \in \mathbf{C}_i$}
    \FOR{every realization $V_j= v$ and $\mathbf{Z}_j=\mathbf{z}$, where $\mathbf{Z}_j$ is the set of effective parents of $V_j$ in $\mathcal{H}_i$}
        \STATE{$N_{\mathbf{z}} \leftarrow$ the number of samples that $\mathbf{Z}_j=\mathbf{z}$,}
        \STATE{$N_{\mathbf{z},v}$ the number of samples that $\mathbf{Z}_j=\mathbf{z}$ and $V_j=v$,}
        \STATE{$\hat{D}_{i,x}(V_j=v| \mathbf{Z}_j=\mathbf{z}) \leftarrow  \frac{N_{\mathbf{z}, v}+1}{N_{\mathbf{z}+2}}$,}
    \ENDFOR
\ENDFOR
\FOR{every variable $V_j \in \mathbf{V}_i \setminus \mathbf{C}_i$}
    \FOR{every $V_j=v$ and $\mathbf{Z}_j \setminus X_i= \mathbf{z}$,}
        \IF{$X_i \in \mathbf{Z}_j$}
            \STATE{$N_{\mathbf{z}} \leftarrow$ the number of samples that $\mathbf{Z}_j \setminus X_i= \mathbf{z}$ and $X_i=x$,}
            \STATE{$N_{ \mathbf{z},v} \leftarrow$ the number of samples that $\mathbf{Z}_j \setminus X_i= \mathbf{z}$, $X_i=x$ and $V_j=v$,}
            \IF{$N_{\mathbf{z}} \geq t$}
                \STATE{$\hat{D}_{i,x}(V_j=v| \mathbf{Z}_j=\mathbf{z}) \leftarrow  \frac{N_{\mathbf{z}, v}+1}{N_{\mathbf{z}+2}}$,}
            \ELSE
                \STATE{$\hat{D}_{i,x}(V_j=v|\mathbf{Z}_j\setminus\{X_i\}=\mathbf{z}, X_i=x )=\frac{1}{2}$,}
            \ENDIF
        \ELSE
            \STATE{$N_{\mathbf{z}} \leftarrow$ the number of samples that $\mathbf{Z}_j= \mathbf{z}$,}
            \STATE{$N_{\mathbf{z}, v} \leftarrow$ the number of samples that $\mathbf{Z}_j= \mathbf{z}$ and $V_j=v$,}
            \IF{$N_{\mathbf{z}} \geq t$}
                \STATE{$\hat{D}_{i,x}(V_j|\mathbf{Z}_j= \mathbf{z}) \leftarrow \frac{N_{\mathbf{z}, v}+1}{N_{\mathbf{z}+2}}$,}
            \ELSE
                \STATE{$\hat{D}_{i,x}(V_j|\mathbf{Z}_j= \mathbf{z}) \leftarrow \frac{1}{2}$,}
            \ENDIF
        \ENDIF
    \ENDFOR
\ENDFOR
\STATE{\textbf{Return}$\hat{D}_{i,x}$.}
\end{algorithmic}
\end{algorithm}

\section{Proofs of Section \ref{sec:simple}}\label{sec: proof simple g}
\theoremsimpleg*

\begin{proof}
Recall 
\begin{align*}
    &q_{i,x}(\mathbf{z})= P(X_i=x, \widetilde{\mathbf{Pa}}(X_i)=\mathbf{z}),\\
    &q_{i,x}= \min_{\mathbf{z}} q_{i,x}(\mathbf{z}),\\
    &q =\min\{q_{i,x} | q_{i,x} > 0: i \in [N], x \in \{0,1\}\}.
\end{align*}
When $q_{i,x}=0$ for all $i,x$, we define $q= \frac{1}{N+1}$.

For each $X_i \in \mathbf{X}$, let $k_i$ be the size of the c-component containing $X_i$, and $k= \min_i k_i$.
Moreover, let $\mathcal{Z}_i$ be the size of the domain from which $\widetilde{\mathbf{Pa}}(X_i)$ takes its values and let $\mathcal{Z}:= \max_i \mathcal{Z}_i$.
Next, we prove a lemma that is useful in the proof of Theorem \ref{th: simple g}.

\begin{lemma}\label{lemma: bound n}
    For every $i \in [N]$, define $f_{i,x}(\mathbf{z})$ to be one, if $|\hat{q}_{i,x}(\mathbf{z}) - q_{i,x}(\mathbf{z})|\geq  \frac{1}{4}(1-2^{-1/k})q$ at the end of $\frac{B}{2}$ rounds. 
    Let $f=1$, if there exists $i \in [N], x \in \{0,1\}$, and $\mathbf{z}$ in the domain of $\widetilde{\mathbf{Pa}}(X_i)$ and $f_{i,x}(\mathbf{z})=1$.
    Then, the following statements hold
    \begin{itemize}
        \item [(a)]  $P(f=1) \leq 4N \mathcal{Z} \exp{\big(-\frac{1}{16} {(1-2^{-1/k})}^2 q^2 B \big)}$.
        \item [(b)] If $f=0$, therefore, $ n(\hat{\mathbf{q}}) \leq 2n(\mathbf{q})$ holds at the end of $\frac{B}{2}$ rounds.
    \end{itemize}
\end{lemma}

\begin{proof}
(a) Using Lemma \ref{lemma: ch-h}, we get
\begin{equation*}
    P(f_{i,x}(\mathbf{z})=1) \leq 2 \exp\big(-\frac{1}{16}(1-2^{-1/k})^2q^2B\big).
\end{equation*}
Now, define $f_{i,x}$ to be one, if there exists $\mathbf{z}$ in domain of $\widetilde{\mathbf{Pa}}(X_i)$ and $f_{i,x}(\mathbf{z})=1$. Using union bound, we get
\begin{equation*}
    P(f_{i,x}=1) \leq 2 \mathcal{Z}_i \exp{\big(-\frac{1}{16} {(1-2^{-1/k})}^2 q^2 B \big)}.
\end{equation*}

Next, let $f_i=1$ if there exists $x \in \{0,1\}$ and $f_{i,x}=1$.
Then,
\begin{equation*}
    P(f_i=1) \leq 4\mathcal{Z}_i \exp{\big(-\frac{1}{16} {(1-2^{-1/k})}^2 q^2 B \big)}
\end{equation*}
Applying the union bound on the above equation, we get
\begin{equation*}
    P(f=1) \leq 4N\mathcal{Z} \exp{\big(-\frac{1}{16} {(1-2^{-1/k})}^2 q^2 B \big)}.
\end{equation*}

(b) First, we sort all $q_{i,x}^{k_i}$ for $i \in [N]$ and $x \in \{0,1\}$ in an ascending order.
Without loss of generality, assume the sorted sequence is $q_1^{k_1} \leq q_2^{k_2} \leq \dots \leq q_{2N}^{k_{2N}}$.
Define $g_1:= \max \{i| q_i^{k_i} < \frac{1}{n(\mathbf{q})}\}$.
The definition of $n(\mathbf{q})$ implies that $g_1 \leq n(\mathbf{q})$ and $q_i^{k_i} \geq \frac{1}{n(\mathbf{q})}$ for every $i > g_1$. 
Therefore, by assuming $|\hat{q}_{i,x}(\mathbf{z}) - q_{i,x}(\mathbf{z})| \leq  \frac{1}{4}(1-2^{-1/k})q$, we get the following for every $i > g_1$:
\begin{equation*}
    {(\hat{q}_{i})}^{k_i} \geq {\big( q_i -\frac{1}{4} (1-2^{-1/k}) q \big)}^{k_i} \labelrel\geq{geq: qi} {\big( q_i -\frac{1}{4} (1-2^{-1/k})q_i\big)}^{k_i} \geq \frac{2^{-\frac{k_i}{k}}}{n(\mathbf{q})} \geq \frac{1}{2n(\mathbf{q})},
\end{equation*}
where the inequality in \eqref{geq: qi} holds since $q_i \geq q$.
Hence, 
\begin{equation*}
    \sum_{i \in [N]} c_i\mathds{1}\{\hat{q}_i^{k_i} < \frac{1}{2n(\mathbf{q})}\} < \sum_{i \in [g_1] } c_i = \sum_{i \in [N]} c_i \mathds{1}\{\hat{q}_i^{k_i} < \frac{1}{n(\mathbf{q})}\} \leq n(\mathbf{q}) \leq 2n(\mathbf{q}).
\end{equation*}
 
The above equation implies that the following inequality holds for $\tau = 2n(\mathbf{q})$,
\begin{equation*}
    \sum_{i \in [N]} c_i \mathds{1}\{\hat{q}_i^{k_i} < \frac{1}{\tau}\} \leq \tau.
\end{equation*}
Then, by the definition of $n(\mathbf{q})$, we get
\begin{equation}
    n(\hat{\mathbf{q}}) \leq 2n(\mathbf{q}).
\end{equation}
\end{proof}

\begin{lemma}\label{lemma: mai}
    Let $a_{i,x} \in \mathcal{A}$ be an arbitrary arm and $\epsilon>0$, then $P(|\hat{\mu}_{i,x} -\mu_{i,x}|>\epsilon) \leq \exp(-\epsilon^2 \frac{q_{i,x}^{k_i} B}{M})$ at the end of $\frac{B}{2}$ rounds, where $M \geq 1$ is a constant number which is independent of the distribution but dependent on the underlying graph.
\end{lemma}
\begin{proof}
    Theorem 2.5 in \citep{bhattacharyya2020learning} implies that $\hat{\mu}_{i,x}$ can be estimated with probability $1-\delta_i$ such that $|\hat{\mu}_{i,x}- \mu_{i,x}|\leq \epsilon$ using $\mathcal{O}(\frac{2^{u_i}}{q_{i,x}^{k_i} \epsilon^{2}}   \log 2^{u_i} \log \frac{1}{\delta_i})$ number of samples, where $u_i= 2(1+k_i(d+1))^2$. 
    Therefore, using $B= K \frac{2^{2u_i}}{q_{i,x}^{k_i} \epsilon^2} \log \frac{1}{\delta_i}$ samples, where $K$ is a constant, we achieve 
    \begin{equation*}
       P\big(|\hat{\mu}_{i,x} -\mu_{i,x}| \leq \epsilon\big) \geq 1-\delta_i.
    \end{equation*}
    Next, we re-write $\delta_i$ in terms of $\epsilon$ and $B$ and get
    \begin{equation*}
         P(|\hat{\mu}_{i,x} -\mu_{i,x}| > \epsilon) \leq \exp \Big(\epsilon^2 \frac{B q_{i,x}^{k_i}}{K 2^{2u_i}} \Big) \leq \Big(\epsilon^2 \frac{B q_{i,x}^{k_i}}{M} \Big),
    \end{equation*}
    where $M=\max \{1, K 2^{2u_i}\} $.
    % Moreover, for $a_0$ we get the following using Lemma \ref{lemma: ch-h}
    % \begin{equation*}
    %     P(|\hat{\mu}_0 -\mu_0| \geq \epsilon) \leq \exp (-\epsilon^2 B )
    % \end{equation*}
\end{proof}

We are now ready to prove Theorem \ref{th: simple g} using the aforementioned Lemmas. 
Let $M' := 2^{k-1}M$ and 
\begin{align*}
    &B_1 := \min_b \Big\{ \sqrt{\frac{4M' n(\mathbf{q})}{b} \log \frac{N b}{n(\mathbf{q})}} \geq 6\frac{n(\mathbf{q})}{b}\Big\},\\
    &B_2 := \min_b  \Big\{  \sqrt{\frac{36 M' n(\mathbf{q})}{b} \log \frac{N b}{n(\mathbf{q})}} \geq 4N \mathcal{Z} \exp{\big(-\frac{1}{16} {(1-2^{-1/k})}^2 q^2 b \big)}\Big \},
\end{align*}
and assume $B \geq \max\{B_1, B_2\}$.

For every $a_{i,x} \in \mathcal{A}'$, Algorithm \ref{alg: g-simple} pulls each arm $\frac{B}{2 \sum_{i,x} c_{i,x} \mathds{1} \{a_{i,x \in \mathcal{A}'}\}}$ additionally to recompute $\hat{\mu}_{i,x}$. 
Therefore, by the definition of $n(\mathbf{q})$ and Lemma \ref{lemma: bound n}, we get the following if $f=0$ (Please see Lemma \ref{lemma: bound n} for the definition of the event $f$),

\begin{equation}
    \frac{B}{2 \sum_{i,x} c_{i,x} \mathds{1} \{a_{i,x \in \mathcal{A}'}\}} \geq \frac{B}{2n(\hat{\mathbf{q}})} \geq \frac{B}{4n(\mathbf{q})}.
\end{equation}

Then, by Lemma \ref{lemma: ch-h}, we have the following equation for every $a_{i,x} \in \mathcal{A}'$:
\begin{equation}\label{eq: mu1}
    P\Big(|\hat{\mu}_{i,x}-\mu_{i,x}| \geq \epsilon | f=0\Big) \leq 2 \exp\big(-\epsilon^2 \frac{B}{2n(\mathbf{q})}\big) \leq 2 \exp\big(-\epsilon^2 \frac{B}{4M'n(\mathbf{q})}\big).
\end{equation}
For each $a_{i,x} \notin \mathcal{A}'$, we know that  $\hat{q}_{i,x}^{k_i} \geq \frac{1}{n(\hat{\mathbf{q}})}$.
However, depending on $q_{i,x}^{k_i}$, the proof technique varies. Below, we present the proof under two different cases:
\begin{itemize}
    \item [\textbf{Case 1.}] If $a_{i,x} \notin \mathcal{A}'$ and $q_{i,x}^{k_i} < \frac{1}{n(\mathbf{q})}$.
    Conditioning on $f=0$, we have :
    \begin{equation*}
        \begin{aligned}
              q_{i,x}^{k_i} \geq {(\hat{q}_{i,x} - \frac{1}{4}(1-2^{-1/k})q)}^{k_i} &\geq {\Big({(\frac{1}{n(\hat{\mathbf{q}})})}^{1/{k_i}}- \frac{1}{4} {(\frac{1}{n(\mathbf{q})})}^{1/{k_i}} \Big)}^{k_i}\\
              & \labelrel\geq{geq: p's} {\Big({(\frac{1}{2n(\mathbf{q})})}^{1/{k_i}}- \frac{1}{4} {(\frac{1}{n(\mathbf{q})})}^{1/{k_i}} \Big)}^{k_i}\\
              & \geq \frac{1}{2^{k+1}n(\mathbf{q})},
        \end{aligned}
    \end{equation*}
    where the inequality in \eqref{geq: p's} follows from Lemma \ref{lemma: bound n}.
    Using the above bound for $q_{i,x}^{k_i}$ and Lemma \ref{lemma: mai} yield
    
    \begin{equation}\label{eq: mu2}
        P(|\hat{\mu}_{i,x}-\mu_{i,x}| \geq \epsilon | f=0) \leq \exp(-\epsilon^2 \frac{B}{2^{k+1}Mn(\mathbf{q})}) = \exp(-\epsilon^2 \frac{B}{4M' n(\mathbf{q})}).
    \end{equation}

    \item [\textbf{Case 2.}] If $a_{i,x} \notin \mathcal{A}'$ and $q_{i,x}^{k_i} \geq \frac{1}{n(\mathbf{q})}$.
    From Lemma \ref{lemma: mai}, we have
        \begin{equation}\label{eq: mu3}
          P\big(|\hat{\mu}_{i,x}-\mu_{i,x}| \geq \epsilon | f=0\big) \leq \exp(-\epsilon^2 \frac{q_{i,x}^{k_i} B}{M}) 
        \leq \exp(-\epsilon^2 \frac{ B}{4M' n(\mathbf{q})}).
\end{equation}
\end{itemize}
For $a=a_0$, Lemma \ref{lemma: ch-h} gives us
\begin{equation}\label{eq: mu0}
    P(|\hat{\mu}_0 - \mu_0| \geq \epsilon) \leq 2 \exp(-\epsilon^2 B) \leq 2\exp(-\epsilon^2 \frac{ B}{4M' n(\mathbf{q})}).
\end{equation}

Now, let $e$ be an event where $e=1$ if there exists an arm $a \in \mathcal{A}$, such that $|\hat{\mu}_a - \mu_a| \geq \epsilon$. We also define event $e_a$ where $e_a=1$ if  $|\hat{\mu}_a - \mu_a| \geq \epsilon$.
Equations \eqref{eq: mu1}, \eqref{eq: mu2}, \eqref{eq: mu3} and \eqref{eq: mu0} imply that for every action $a \in \mathcal{A}$,
\begin{equation*}
     P(e_a=1|f=0) \leq 2\exp\big(-\epsilon^2 \frac{ B}{4M' n(\mathbf{q})}\big).
\end{equation*}

Applying the union bound on the above equation implies
\begin{equation*}
    P(e=1|f=0) \leq (4N+2)\exp\big(-\epsilon^2 \frac{ B}{4M' n(\mathbf{q})}\big) \leq 6N \exp\big(-\epsilon^2 \frac{ B}{4M' n(\mathbf{q})}\big).
\end{equation*}

Using the above inequalities and substituting $\epsilon= \sqrt{\frac{4M' n(\mathbf{q})}{B} \log \frac{N B}{n(\mathbf{q})}}$ yield
\begin{equation}\label{eq: bound f0}
    \begin{aligned}
        \mathbbm E[R_s(B)|f=0] &= \mathbbm E[R(B)| e=0] P(e=0) + \mathbbm E[R(B)| e=1] P(e=1)\\
        &\leq \mathbbm E[R(B)| e=0]+ P(e=1)\\
        &\leq 2\epsilon+ 6N \exp(-\epsilon^2 \frac{ B}{4M' n(\mathbf{q})})\\
        & = 2\sqrt{\frac{4M' n(\mathbf{q})}{B} \log \frac{N B}{n(\mathbf{q})}}+6\frac{n(\mathbf{q})}{B}\\
        & \leq \sqrt{\frac{36 M' n(\mathbf{q})}{B} \log \frac{N B}{n(\mathbf{q})}}.
    \end{aligned}
\end{equation}

Finally,  Lemma \ref{lemma: bound n} and Equation \eqref{eq: bound f0} imply
\begin{equation*}
    \begin{aligned}
        \mathbbm E[R_s(B)]&= \mathbbm E[R(B)| f=0] P(f=0)+\mathbbm E[R(B)|f=1] P(f=1)\\
        & \leq  \mathbbm E[R(B)| f=0]+P(f=1)\\
        &\leq \sqrt{\frac{36 M' n(\mathbf{q})}{B} \log \frac{N B}{n(\mathbf{q})}}+ 4N \mathcal{Z} \exp{\big(-\frac{1}{16} {(1-2^{-1/k})}^2 (p')^2 B \big)}\\
        & \leq  2 \sqrt{\frac{36 M' n(\mathbf{q})}{B} \log \frac{N B}{n(\mathbf{q})}}.
    \end{aligned}
\end{equation*}
Therefore, $\mathbbm E[R_s(B)] \in \mathcal{O} (\sqrt{\frac{n(\mathbf{q})}{B} \log \frac{NB}{n(\mathbf{q})}})$.
\end{proof}

\section{Discussion on simple regret bounds in no-backdoor graphs}\label{sec: noback}

\remarkno*

\begin{proof}
    In general graphs, we estimate $q_{i,x}= \min_{\mathbf{z}}P(X_i= x, \widetilde{\mathbf{Pa}}(X_i)=\mathbf{z})$ by $$\hat{q}_{i,x} = \frac{2}{B} \min_{\mathbf{z}}{ \Big\{\sum_{t=1}^{B/2} \mathds{1} \big\{x_i^t=x,  \widetilde{\mathbf{Pa}}^t(x_i) = \mathbf{z} \big\} \Big\}}.$$
    In  no-backdoor graphs, we redefine $q_{i,x} = P(X_i=x)$ and $\hat{q}_{i,x}= \frac{2}{B} \sum_{t=1}^{B/2} \mathds{1}\{x_i^t=x\}$. 
    Through out this section, we denote the redefined $q_{i,x}$ and $\hat{q}_{i,x}$ by $q_{i,x}^{new}$ and $\hat{q}_{i,x}^{new}$.
    
    Following the same procedure as in the proof of Theorem \ref{th: simple g}, it is straightforward to show that Algorithm \ref{alg: g-simple}, using $q_{i,x}^{new}$ and $\hat{q}_{i,x}^{new}$, achieves 
    %Using this redefinition and the same procedure as proof of Theorem \ref{th: simple g}, results 
    $R_s(B) \in  \mathcal{O}\Big(\sqrt{\frac{n(\mathbf{q}^{new})}{B} \log \frac{NB}{n(\mathbf{q}^{new})}}\Big)$.
    Herein, we show that $n(\mathbf{q}^{new}) \leq n(\mathbf{q})$ which implies that in the no-backdoor setting, the new definitions guarantee better regret bound for Algorithm \ref{alg: g-simple}.

    Note that for every $i \in [N]$, $x \in \{0,1\}$ and $\mathbf{z}$ in the domain of $\widetilde{\mathbf{Pa}}(X_i)$, we have
    \begin{equation*}
        \begin{aligned}
                P(X_i=x, \widetilde{\mathbf{Pa}}(X_i)=\mathbf{z})&= P( \widetilde{\mathbf{Pa}}(X_i)=\mathbf{z}| X_i=x)P(X_i=x) \leq P(X_i=x).
        \end{aligned}
    \end{equation*}
    Therefore, $q_{i,x}= \min_{\mathbf{z}}P(X_i=x, \widetilde{\mathbf{Pa}}(X_i)=\mathbf{z}) \leq P(X_i=x)= q_{i,x}^{new}$ for every $i, x$.
    From definition of $n(\mathbf{q})$ we get
    \begin{equation*}
          \sum_{i,x} c_{i,x}\mathds{1}\big\{ q_{i,x} < {(\frac{1}{n(\mathbf{q})})}^{1/k_i}\big\} \leq n(\mathbf{q}),
    \end{equation*}
    and since $q_{i,x} \leq q_{i,x}^{new}$,
    \begin{equation*}
        \sum_{i,x} c_{i,x}\mathds{1}\big\{ q_{i,x}^{new} < {(\frac{1}{n(\mathbf{q})})}^{1/k_i}\big\} \leq \sum_{i,x} c_{i,x}\mathds{1}\big\{ q_{i,x} < {(\frac{1}{n(\mathbf{q})})}^{1/k_i}\big\}.
    \end{equation*}
    Therefore, the following inequality holds for $\tau = n(\mathbf{q})$
    \begin{equation*}
        \sum_{i,x} c_{i,x}\mathds{1}\big\{ q_{i,x}^{new} < {(\frac{1}{\tau})}^{1/k_i}\big\} \leq \tau.
    \end{equation*}
    On the other hand, from the definition, we know that $n(\mathbf{q}^{new})$ is the smallest $\tau$ which holds in the above equation. Therefore, $n(\mathbf{q}^{new}) \leq n(\mathbf{q})$.
\end{proof}

\remarkbetter*
\begin{proof}
Recall $p_{i,x}= P(X_i=x)$.
Similar to Remark \ref{remark: no}, we denote $p_{i,x}$ by $q_{i,x}^{new}$.
By assuming $c_{i,x}= c$ for all $i \in [N]$ and $x \in \{0,1\}$, and all variables are observable in the underlying causal graph, we aim to show that $n(\mathbf{q}^{new}) \leq m'(\mathbf{q}^{new})$.
Note that since all the variables are observable, the size of the c-component for every $X_i$ is equal to one, i.e., $k_i=1$.
By the definition of $m'(\mathbf{q}^{new})$, we have 
\begin{equation*}
    \sum_{i,x} \mathds{1}\{q_{i,x}^{new} < \frac{1}{m'(\mathbf{q}^{new})}\} \leq m'(\mathbf{q}^{new}). 
\end{equation*}
Moreover, we have 
\begin{equation*}
    \sum_{i,x} \mathds{1}\{p_{i,x} < \frac{1}{c m'(\mathbf{q}^{new})}\} \leq \sum_{i,x} \mathds{1}\{q_{i,x}^{new} < \frac{1}{m'(\mathbf{q}^{new})}\}.
\end{equation*}
Using the above equations, we can write $ c \sum_{i,x} \mathds{1}\{q_{i,x}^{new} < \frac{1}{c m'(\mathbf{q}^{new})}\} \leq c.m'(\mathbf{q}^{new})$.
Finally, using the definition of $n(\mathbf{q}^{new})$ implies that $n(\mathbf{q}^{new}) \leq c.m'(\mathbf{q}^{new})$.
\end{proof}

{
\section{Cumulative Regret Lower Bound:}\label{sec:lower-cum}
We need the following technical lemmas. 
\begin{lemma}[Bretagnolle–Huber inequality]
    Let $P$ and $Q$ be probability measures on the same measurable space, and let $A$ be an arbitrary event. Then
    \begin{align*}
        P(A)+Q(A^c)\geq \frac{1}{2}e^{-D(P||Q)},
    \end{align*}
    where $D$ and $A^c$ denote the KL-divergence and the complement of $A$, respectively.
\end{lemma}

\begin{lemma}[\cite{lattimore2020bandit}, Lemma 15.1]\label{lower:lemma}
    Let $v=(P_1,...,P_k)$ be the reward distributions associated with one k-armed bandit, and let $v'=(P_1',...,P_k')$ be the reward distributions associated with another k-armed bandit. Fix some policy $\pi$ and let $P_{\pi}$ and $P_{\pi}'$ be the probability measures on the canonical bandit model  induced by the n-round interconnection of $\pi$ and $v$ and $v'$, respectively. Then,
    \begin{align*}
        D(P_{\pi}||P_{\pi}')=\sum_{i=1}^k\mathbb{E}_{\pi,v}[N_i(n)]D(P_i||p_i').
    \end{align*}
    
\end{lemma}

\begin{lemma}\label{lemma:lower}
Let $\{P_\theta : \theta \in \mathbb{R}\}$ be a parametric family of distributions such that $P_\theta$ has mean $\theta$. Suppose that the densities are twice continuously differentiable. Then, there exists $x_0$ such that $|x_0|\leq |\delta|$ and
\begin{align*}
    D(P_\theta|| P_{\theta+\delta}) = \frac{I(\theta+x_0)}{2}\delta^2,
\end{align*}
where $I(x)$ denotes the Fisher information of the family $P_\theta$ at $x$.
\end{lemma}
\begin{proof}
    Using a Taylor expansion of $h(x):=D(P_{\theta}||P_{\theta+x})$ around $x=0$, we get
    \begin{align*}
        & h(x)=h(0)+h'(0)x + h''(x_0)\frac{x^2}{2},
    \end{align*}
    where $|x_0|\leq |x|$. We have $h(0)=0$, $h''(x_0)=I(\theta+x_0)$, and 
    \begin{align*}
        h'(0)&=\frac{\partial}{\partial x}D(P_\theta||P_{\theta+x})\Big|_{x=0}=\int\frac{\partial}{\partial x}\log\frac{dP_\theta}{dP_{\theta+x}}\Big|_{x=0}dP_{\theta}=-\int\frac{\partial}{\partial x}\log\frac{dP_{\theta+x}}{dP_\theta}\Big|_{x=0}dP_{\theta}\\
        &=-\int\frac{\frac{\partial}{\partial x}\big(\frac{dP_{\theta+x}}{dP_\theta}\big)}{\frac{dP_{\theta+x}}{dP_\theta}}\Big|_{x=0}dP_{\theta}=-\int\frac{\partial}{\partial x}\big(\frac{dP_{\theta+x}}{dP_\theta}\big)\Big|_{x=0}dP_{\theta}=-\frac{\partial}{\partial x}\int\big(\frac{dP_{\theta+x}}{dP_\theta}\big)dP_{\theta}\Big|_{x=0}\\
        &=-\frac{\partial}{\partial x}\int dP_{\theta+x}\Big|_{x=0}=-\frac{\partial}{\partial x}1\Big|_{x=0}=0.
    \end{align*}
\end{proof}

%\begin{lemma}[\cite{kaufmann2016complexity}, Lemma 1]\label{lemma:lower}
%    m
%\end{lemma}

To establish the result, we use an analogous argument as in the classical multi-arm bandits sand how that any algorithm suffers $\Omega(\sqrt{\lfloor B/c\rfloor KN})$ regret on a specific causal graph $\mathcal{G}_0$ depicted in Figure \ref{fig:lower} with a predefined reward distributions $P_1$ and $P_2$ and the uniform costs $\mathcal{C}_0:=\{c_{i,x}=c\}$. 
%the causal graph $\mathcal{G}_0$ in Figure \ref{fig:lower} and  show that any adaptive algorithm suffers at least a cumulative regret of the order $\Omega(\sqrt{BKN})$ when all actions have the unit cost, i.e., $c_{i,x}=1$ and the total available budget is $B$. 
\begin{figure}[H]
        \centering
        \tikzstyle{block} = [draw, fill=white, circle, text centered]
    	\tikzstyle{input} = [coordinate]
    	\tikzstyle{output} = [coordinate]
    	    \begin{tikzpicture}[->, auto, node distance=1.3cm,>=latex', every node/.style={inner sep=0.05cm}, scale=0.61]
    		    \node (X1) at (-4,0) {$X_1$};
                    \node (X2) at (-2,0) {$X_2$};
                    \node (X.) at (0,0) {$\cdots$};
                    \node (Xn) at (2,0) {$X_N$};
    		    \node (Y) at (-1,-2) {$Y$};
    		    \path[->] (X1) edge[ style = {->}](Y);
                    \path[->] (X2) edge[ style = {->}](Y);
                    \path[->] (Xn) edge[ style = {->}](Y);
    		   %\path[->] (X2) edge[dashed, style = {<->}](X5);
    		   %\path[->] (X3) edge[dashed, style = {<->}, bend left=30](X1);
    		   %\path[->] (X2) edge[ dashed, style = {<->}, bend right=60](X1);
    		\end{tikzpicture}
    	\caption{The ADMAG $\mathcal{G}_0$ over $\textbf{V}=\{X_1,...,X_N, Y\}$.}
    	\label{fig:lower}
    \end{figure} 
    
In this causal bandits setting, we assume that all variables can take values in $[K]=\{1,...,K\}$ and consider two different distributions over the reward variable $Y$ belonging to some parametric family of distributions that have twice differentiable density functions, e.g., Gaussian distributions. 
Distribution $P_1$ is selected such that $\mathbb{E}_1[Y|do(X_i=x)]=b$ for some constant $b$, for all $i\in[N]$, and all $x\in[K]$ and $\mathbb{E}_1[Y]=b+\sqrt{\frac{KN}{w\lfloor B/c\rfloor}}$ for some constant $w$ to be defined later. Hence, the best action to play in this setting would be $a_0$, i.e., no intervention.

Let $A_B$ to be an arbitrary adaptive algorithm that selects the actions possibly based on its previous interactions with the problem and the total budget $B$ and let $P_{A,1}$ to be the resulting distribution of applying this algorithm when the rewards are distributed according to $P_1$. 
In order to design the second distribution, we select the least played action by $A_B$ and assign a higher expected reward to it. To be more precise, let
\begin{align*}
a_{i^*,x^*}:=\arg\min_{a_{i,x}}\mathbb{E}_{A,1}[N_{a_{i,x}}^B],    
\end{align*}
where $N_{a_{i,x}}^B$ denotes the number of times that arm $a_{i,x}$ is played by the algorithm using all its budget $B$ (since all arms have the same cost $c$, that is equivalent to playing the arms over a time horizon $\lfloor B/c\rfloor$) and the expectation is taken with respect to $P_{A,1}$. Note that $\sum_{a_{i,x}\in\mathcal{A}}N_{a_{i,x}}^B=\lfloor B/c\rfloor$, the total number of actions are $|\mathcal{A}|=NK+1$, and thus 
\begin{align}\label{eq:lower1}
    \mathbb{E}_{A,1}[N_{a_{i^*,x^*}}^B]\leq \frac{\lfloor B/c\rfloor}{NK}.
\end{align}
Now, we can design the second distribution $P_2$ that is identical to $P_1$ except at index $a_{i^*,x^*}$ and at this index it has $\mathbb{E}_2[Y|do(X_{i^*}=x^*)]=b+2\sqrt{\frac{KN}{w\lfloor B/c\rfloor}}$.
Therefore, the optimal arm under this distribution is $a_{i^*,x^*}$.
We have 
\begin{align*}
    R_c(A_B,\mathcal{G}_0,P_1,\mathcal{C}_0)&=\sum_{a_{i,x}\in\mathcal{A}}\mathbb{E}_{A,1}[N_{a_{i,x}}^B]\delta_{i,x}= \mathbb{E}_{A,1}[\lfloor B/c\rfloor-N_{a_{0}}^B]\sqrt{\frac{KN}{w\lfloor B/c\rfloor}}\\
    &\geq P_{A,1}(N_{a_{0}}^B\leq \lfloor B/c\rfloor/2)\frac{\lfloor B/c\rfloor}{2}\sqrt{\frac{KN}{w\lfloor B/c\rfloor}},\\
    R_c(A_B,\mathcal{G}_0,P_2,\mathcal{C}_0)&=\sum_{a_{i,x}\in\mathcal{A}}\mathbb{E}_{A,2}[N_{a_{i,x}}^B]\delta_{i,x}\geq \mathbb{E}_{A,2}[N_{a_{0}}^B]\sqrt{\frac{KN}{w\lfloor B/c\rfloor}}\\
    &\geq P_{A,2}(N_{a_{0}}^B> \lfloor B/c\rfloor/2)\frac{\lfloor B/c\rfloor}{2}\sqrt{\frac{KN}{w\lfloor B/c\rfloor}},
\end{align*}
where $R_c(A_B,\mathcal{G},P,\mathcal{C})$ denotes the cumulative regret of algorithm $A$ on a causal graph $\mathcal{G}$ with the distribution $P$ over the rewards and the cost set $\mathcal{C}$. 
Combining the above inequalities and using the Bretagnolle–Huber inequality \citep{lattimore2020bandit}, we have
\begin{align*}
    &2\max\{R_c(A_B,\mathcal{G}_0,P_1,\mathcal{C}_0),R_c(A_B,\mathcal{G}_0,P_2,\mathcal{C}_0)\}\geq R_c(A_B,\mathcal{G}_0,P_1,\mathcal{C}_0)+R_c(A_B,\mathcal{G}_0,P_2,\mathcal{C}_0)\\
    &\geq \Big(P_{A,1}(N_{a_{0}}^B\leq \lfloor B/c\rfloor/2) + P_{A,2}(N_{a_{0}}^B> \lfloor B/c\rfloor/2) \Big)\frac{\lfloor B/c\rfloor}{2}\sqrt{\frac{KN}{w\lfloor B/c\rfloor}}\\
    &\geq \frac{\lfloor B/c\rfloor}{4}\sqrt{\frac{KN}{w\lfloor B/c\rfloor}}e^{-D(P_{A,1}||P_{A,2})}.
\end{align*}
%Now, we use Lemma \ref{lower:lemma} to find an upper bound for $D(P_{A,1}||P_{A,2})$.
From the definition of $P_1$ and $P_2$ and using Lemma \ref{lower:lemma}, we get
\begin{align*}
    D(P_{A,1}||P_{A,2})=\mathbb{E}_{A,1}[N_{a_{i^*,x^*}}^B]D(P_1(a_{i^*,x^*})|| P_2(a_{i^*,x^*})).
\end{align*}
Using \eqref{eq:lower1} and Lemma \ref{lemma:lower}, we get
\begin{align*}
    D(P_{A,1}||P_{A,2})\leq \frac{\lfloor B/c\rfloor}{NK}\frac{I(b+\epsilon)}{2}\left(2\sqrt{\frac{KN}{w\lfloor B/c\rfloor}}\right)^2=\frac{2I(b+\epsilon)}{w},
\end{align*}
where $\epsilon\leq 2\sqrt{\frac{KN}{w\lfloor B/c\rfloor}}$. 
By selecting $w$ large enough, we can ensure that $2I(b+\epsilon)<w$. Note that $I(b+\epsilon)$ is a constant and depends on the family distribution. For instance, for the family of Gaussian distributions with unit variance and mean $\theta$, we have $I(b+\epsilon)=1$.
%and for family of Bernoulli distributions with parameter $\theta$, we have $I(\epsilon)=1/(\epsilon(1-\epsilon))$.
Combining the above inequalities leads to
\begin{align*}
    &\max_{\mathcal{G}_N,P,\mathcal{C}} R_c(A_B,\mathcal{G}_N,P,\mathcal{C})\geq\max\{R_c(A_B,\mathcal{G}_0,P_1,\mathcal{C}_0),R_c(A_B,\mathcal{G}_0,P_2,\mathcal{C}_0)\}\geq \frac{\lfloor B/c\rfloor}{8e}\sqrt{\frac{KN}{w\lfloor B/c\rfloor}}.
\end{align*}
This inequality holds for any arbitrary adaptive algorithm $A_B$.
}

\section{Additional Experiments}\label{sec: additional exp}
\subsection{Additional Experiments on Cumulative Regret in General Graphs}\label{sec: ad exp cumu}

Figure \ref{fig: graph-cumu-N6} illustrates the underlying causal graph of the experiment in Section \ref{sec: exp cumu} which we used to compare the performance of Algorithm \ref{alg: g-cumulative} with \textit{CRM} and \textit{F-KUBE}.
\begin{figure}[!ht] 
        \centering
        \tikzstyle{block} = [draw, fill=white, circle, text centered]
    	\tikzstyle{input} = [coordinate]
    	\tikzstyle{output} = [coordinate]
    	    \begin{tikzpicture}[->, auto, node distance=1.3cm,>=latex', every node/.style={inner sep=0.05cm}]
    		    \node[block] (X3) at (0,0) {$X_3$};
    		    \node[block] (X1) at (-1,1) {$X_1$};
    		    \node[block] (X2) at (1,1) {$X_2$};
    		    \node[block] (X4) at (1,-1) {$X_4$};
    		    \node[block] (X5) at (-1,-1) {$X_5$};
    		    \node[block] (X6) at (-1,-2) {$X_6$};
    		    \node[block] (Y) at (1,-2) {$Y$};
    		    \path[->] (X1) edge[ style = {->}](X2);
    		    \path[->] (X1) edge[ style = {->}](X3);
    		    \path[->] (X2) edge[ style = {->}](X3);
    		    \path[->] (X2) edge[ style = {->}](X4);
    		    \path[->] (X3) edge[ style = {->}](X5);
    		    \path[->] (X3) edge[ style = {->}](X4);
    		    \path[->] (X4) edge[ style = {->}](X5);
    		    \path[->] (X4) edge[ style = {->}](X6);
    		    \path[->] (X5) edge[ style = {->}](X6);
    		    \path[->] (X6) edge[ style = {->}](Y);
    		  %  \path[->] (X3) edge[dashed, style = {<->}](X4);
    		\end{tikzpicture}
    	\caption{Causal graph of experiments in Section \ref{sec: exp cumu}.}
    	\label{fig: graph-cumu-N6}
    \end{figure}
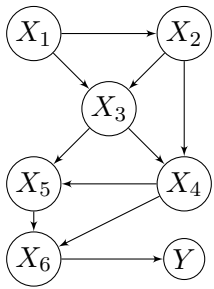
    
As mentioned before, since \textit{CRM} is not designed for graphs with hidden variables, we did not include \textit{CRM} for comparison in graphs with hidden variables. 

Herein, we compare Algorithm \ref{alg: g-cumulative} with \textit{F-KUBE} when the model is constructed as explained in Section \ref{sec: exp cumu}.
The underlying graphs is demonstrated in Figure \ref{fig: graph-cumu-N6h} which has $2$ hidden variables.
% Figure \ref{fig: cumu-BN6h}
The left plot in Figure \ref{fig: cumu-N6h} illustrates the performance of algorithms in terms of cumulative regret versus budget when the cost of pulling each arm was selected randomly from $\{2,3\}$.
By increasing the budget, the cumulative regret of all of the algorithms increases.
Although, our algorithm has a lower growth rate than \textit{F-KUBE} and \textit{CRM}.

% Figure \ref{fig: cumu-cN6h}
Moreover, the right plot compares the performance of the algorithms when the budget is fixed to $1000$, and the cost of all interventional arms is equal to $c$, such that $c \in \{2, 3, \dots, 20\}$ and the cost of the observational arm is equal to $1$.
As shown in this Figure, the cumulative regret of Algorithm \ref{alg: g-cumulative} grows substantially slower than others.
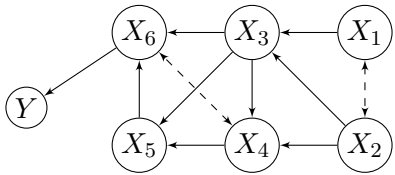
\begin{figure}[t] 
        \centering
        \tikzstyle{block} = [draw, fill=white, circle, text centered]
    	\tikzstyle{input} = [coordinate]
    	\tikzstyle{output} = [coordinate]
    	    \begin{tikzpicture}[->, auto, node distance=1.3cm,>=latex', every node/.style={inner sep=0.05cm}]
    		    \node[block] (X3) at (0,0) {$X_3$};
    		    \node[block] (X1) at (1.5,0) {$X_1$};
    		    \node[block] (X2) at (1.5,-1.5) {$X_2$};
    		    \node[block] (X4) at (0,-1.5) {$X_4$};
    		    \node[block] (X5) at (-1.5,-1.5) {$X_5$};
    		    \node[block] (X6) at (-1.5,0) {$X_6$};
    		    \node[block] (Y) at (-3,-1) {$Y$};
    		    \path[->] (X1) edge[ style = {->}](X3);
    		    \path[->] (X2) edge[ style = {->}](X3);
    		    \path[->] (X2) edge[ style = {->}](X4);
    		    \path[->] (X3) edge[ style = {->}](X4);
    		    \path[->] (X3) edge[ style = {->}](X5);
    		    \path[->] (X3) edge[ style = {->}](X6);
    		    \path[->] (X4) edge[ style = {->}](X5);
    		    \path[->] (X5) edge[ style = {->}](X6);
    		    \path[->] (X6) edge[ style = {->}](Y);
    		    \path[->] (X1) edge[dashed, style = {<->}](X2);
    		    \path[->] (X4) edge[dashed, style = {<->}](X6);
    		\end{tikzpicture}
    	\caption{Causal graph of experiments in Figure \ref{fig: cumu-N6h}.}
    	\label{fig: graph-cumu-N6h}
    \end{figure}

% \begin{figure}[!ht]
% \centering
%   \subfloat[Cumulative regret vs budget.]{
% 	\begin{minipage}[c][1\width]{
% 	   0.3\textwidth}
% 	   \centering
% 	   \includegraphics[width=1\textwidth]{Figures/general_cumulative_BN6h.pdf}
% 	\end{minipage}}\label{fig: cumu-BN6h}
%  \hspace{.8cm}
%   \subfloat[Cumulative regret vs cost of intervening.]{
% 	\begin{minipage}[c][1\width]{
% 	   0.3\textwidth}
% 	   \centering
% 	   \includegraphics[width=1.1\textwidth]{Figures/general_cumulative_cN6h.pdf}
% 	\end{minipage}}\label{fig: cumu-cN6h}
% \caption{Performances of different algorithms on a general graphs with $N=6$ depicted in Figure \ref{fig: graph-cumu-N6h}.}
%         \label{fig: cumu-N6h}
% \end{figure}

\begin{figure}[!ht] 
        \centering
        \captionsetup{justification=centering}
        \begin{thesubfigure}%[b]{0.43\textwidth}
            \centering
             \includegraphics[width=0.49\textwidth]{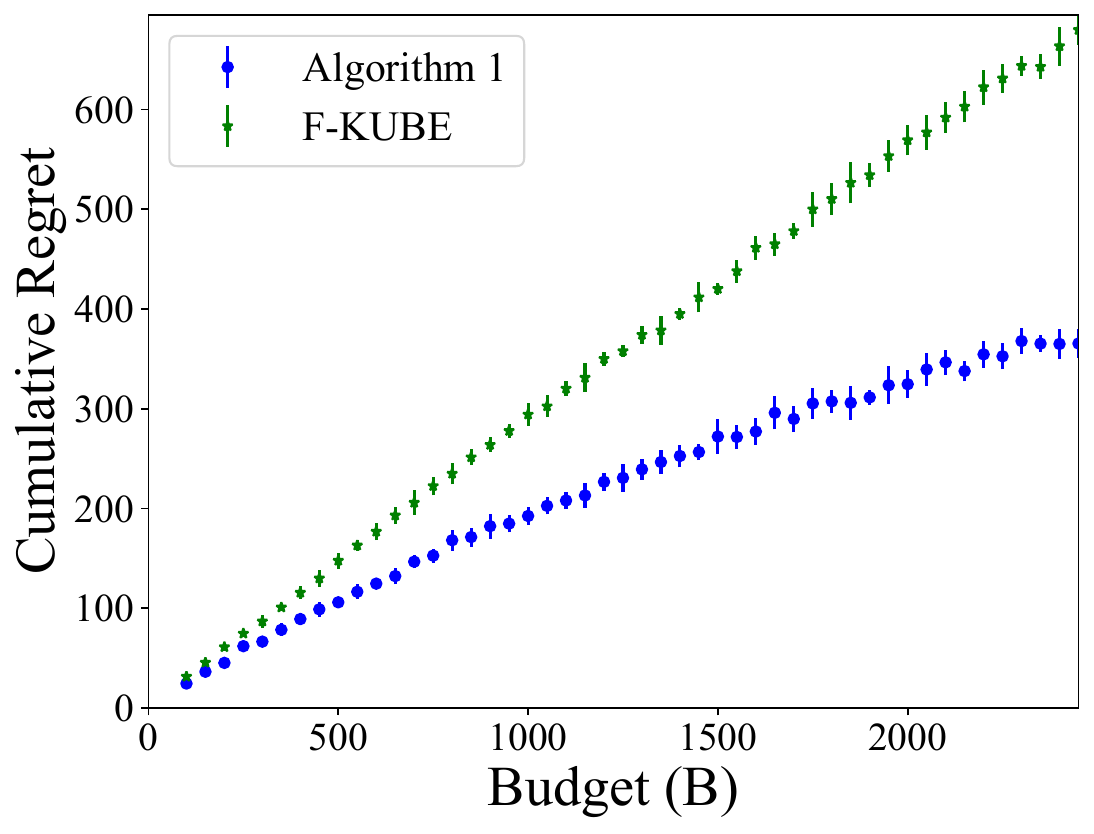}
            % \caption{Cumulative regret vs budget.}
            \label{fig: cumu-BN6h}
        \end{thesubfigure}
        \begin{thesubfigure}%[b]{0.45\textwidth}
            \centering
            \includegraphics[width=0.49\textwidth]{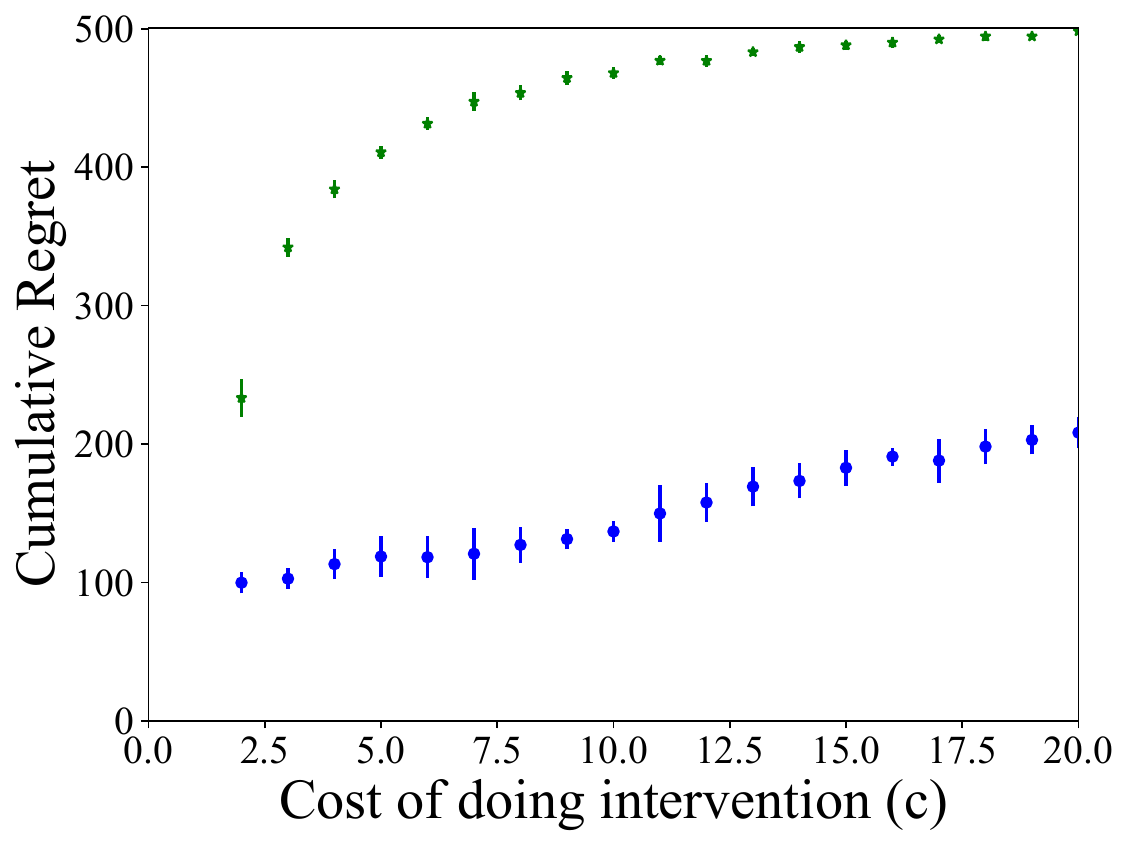}
            % \caption{Cumulative regret vs cost of intervening.}
            \label{fig: cumu-cN6h}
        \end{thesubfigure}
        \caption{Performances of different algorithms on a general graphs with $N=6$ depicted in Figure \ref{fig: graph-cumu-N6h}.}
        \label{fig: cumu-N6h}
\end{figure}

\subsection{Additional Experiments on Simple Regret in No-backdoor Graphs}\label{sec: ad exp simple no}

Since $\gamma$-NB is designed for settings with uniform costs over the arms, to have a fair comparison between Algorithm \ref{alg: g-simple} and $\gamma$-NB, we included an extra experiment in this section.
In this experiment, the cost of pulling each interventional arm $a_{i,x}$ for $i \in [N], x \in \{0,1\}$ is set to be $4$ and the cost of the observational arm is $1$. 
The other setting of this experiment is  similar to the one in Section \ref{sec: exp simple no}.

Figure \ref{fig: nobackdoor_BN50} illustrates the result of this experiment and it shows that when the cost of pulling interventional arms is uniform, by increasing the budget, our proposed algorithm converges quicker to zero than the others.
 
\begin{figure*}[!ht] 
        \centering
        \captionsetup{justification=centering}
            \centering
             \includegraphics[width=0.43\textwidth]{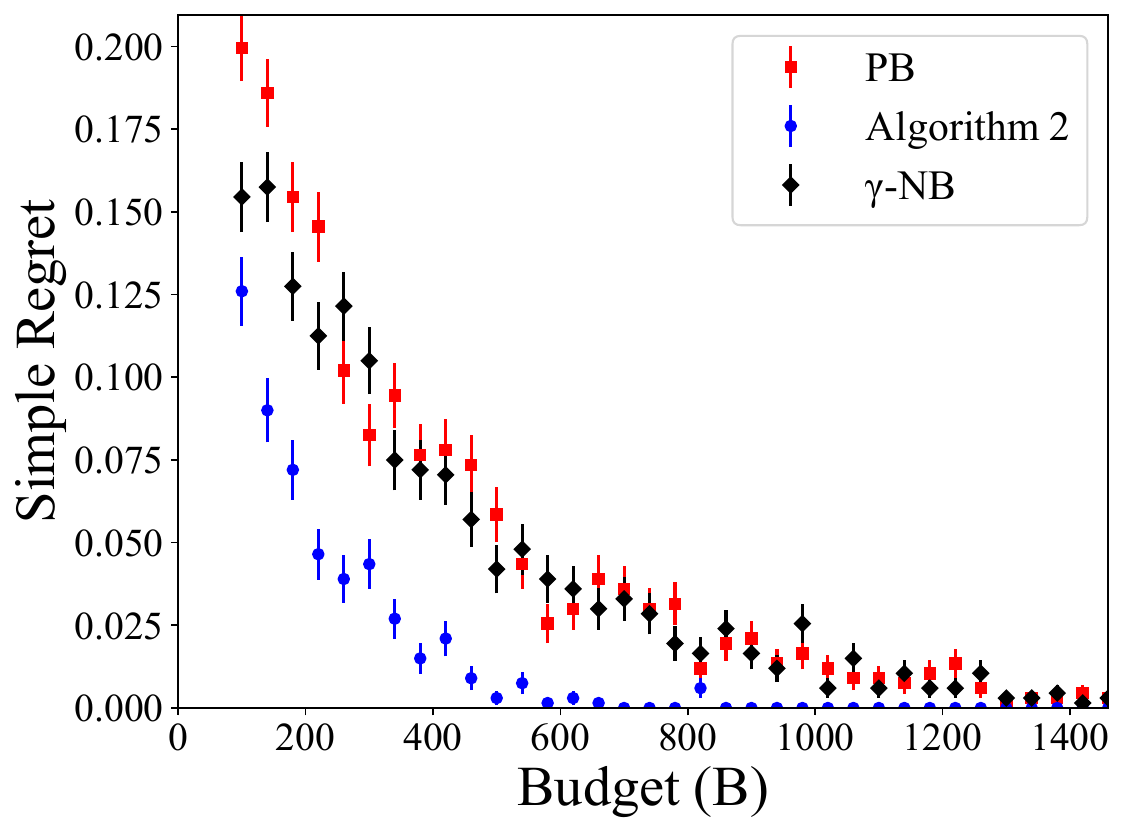}
        \caption{Performance of different algorithms on parallel graphs with $N=50$.}
        \label{fig: nobackdoor_BN50}
\end{figure*}

We also included an experiment on a smaller parallel graph (with $7$ intervenable variables) to be able to compare our algorithm with \textit{Successive Rejects}.
To construct the underlying model, we used the same setting as in Section \ref{sec: exp simple no}.
Figure \ref{fig: nobackdoorN7} illustrates the performance of different algorithms in terms of their simple regret. 
%In Figure \ref{fig: nobackdoor-ciN7}
For simple regret vs. budget, we set the cost of each interventional arm randomly from $\{2, 3, 4, 5\}$. This figure shows that Algorithm \ref{alg: g-simple} convergence is faster to zero than the other algorithms.
%In Figure \ref{fig: nobackdoor-cN7}
For simple regret vs. the cost of doing intervention, the budget is set to $1500$, and the cost of all interventional arms is equal to $c \in \{1,2, \dots, 20\}$.
This figure demonstrates that by increasing $c$, Algorithm \ref{alg: g-simple} has a slower growth rate than others.

\begin{figure}[!ht] 
        \centering
        \captionsetup{justification=centering}
        \begin{thesubfigure}%[b]{0.43\textwidth}
            \centering
            \includegraphics[width=0.45\textwidth]{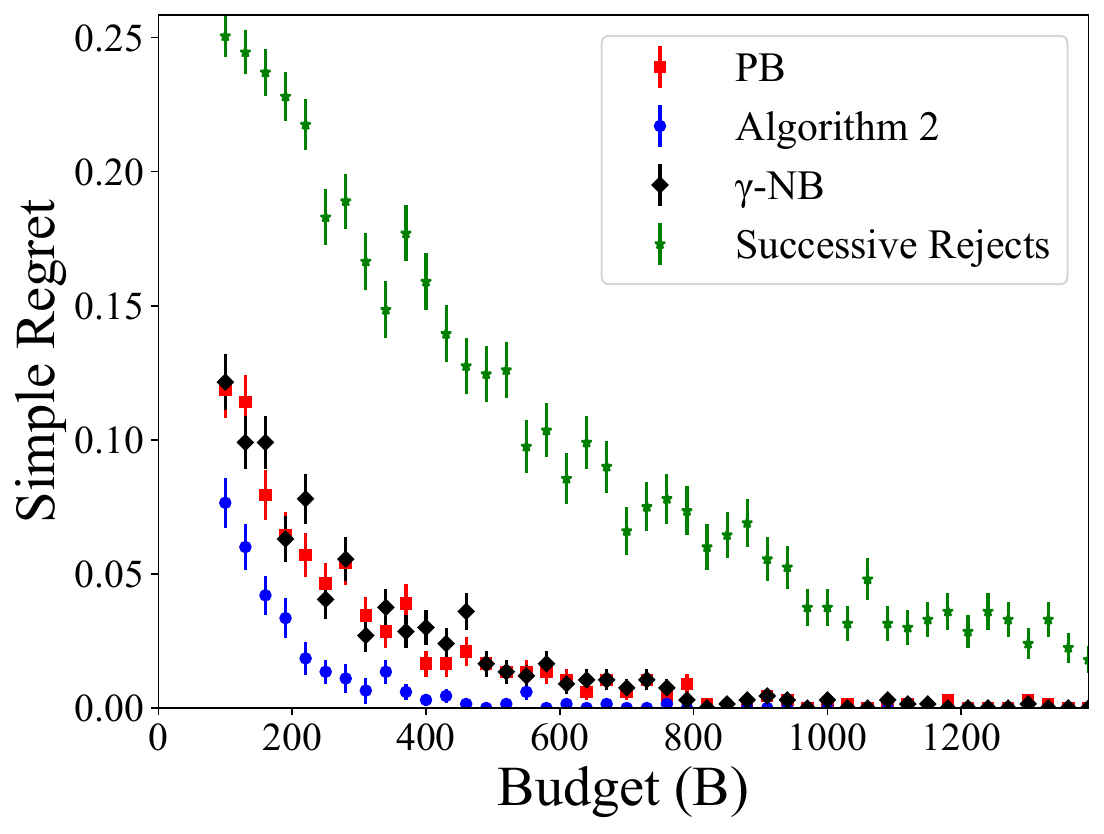}
            % \caption{Simple regret vs budget.}
            % \label{fig: nobackdoor-ciN7}
        \end{thesubfigure}
        \begin{thesubfigure}%[b]{0.45\textwidth}
            \centering
            \includegraphics[width=0.45\textwidth]{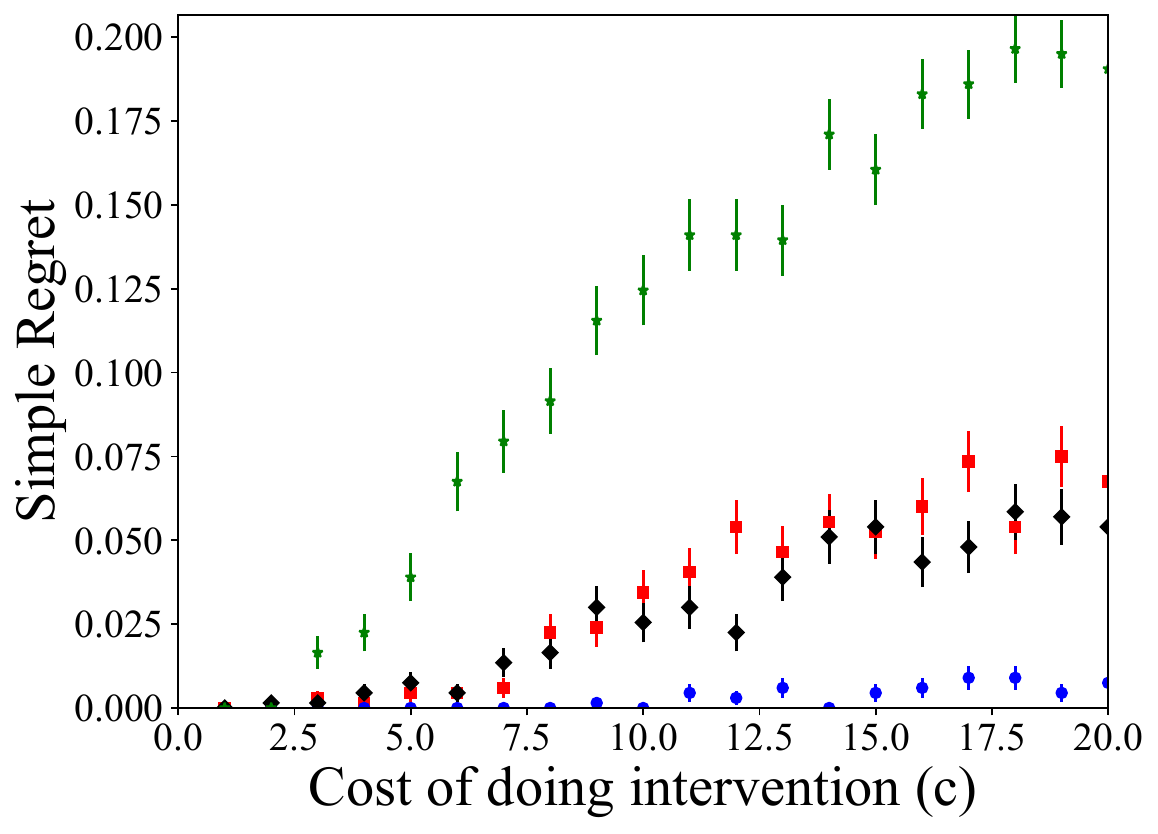}
             % \caption{Simple regret vs cost of intervening.}
            % \label{fig: nobackdoor-cN7}
        \end{thesubfigure}
        \caption{Performance of different algorithms on a parallel graph with $N=7$.}
        \label{fig: nobackdoorN7}
\end{figure}

\subsection{Additional Experiments on Simple Regret in General Graphs}\label{sec: ad exp simple g}
Herein, we present the underlying causal graph of the experiments in Section \ref{sec: exp simple g}.
As shown, this graph is not a no-backdoor graph as it has unblocked backdoor paths from intervenable variables to the reward variable. 
Furthermore, it has two hidden confounders.

\begin{figure}[!ht] 
        \centering
        \tikzstyle{block} = [draw, fill=white, circle, text centered]
    	\tikzstyle{input} = [coordinate]
    	\tikzstyle{output} = [coordinate]
    	    \begin{tikzpicture}[->, auto, node distance=1.3cm,>=latex', every node/.style={inner sep=0.05cm}]
    		    \node[block] (X1) at (0,0) {$X_1$};
    		    \node[block] (X2) at (-1,0) {$X_2$};
    		    \node[block] (X3) at (-0.7,-3) {$X_3$};
    		    \node[block] (X4) at (-2.5,0) {$X_4$};
    		    \node[block] (X5) at (-2,-2) {$X_5$};
    		    \node[block] (X6) at (-5,-2) {$X_6$};
    		    \node[block] (X7) at (-3,-2) {$X_7$};
    		    \node[block] (Y) at (-3,-1) {$Y$};
    		    \path[->] (X1) edge[ style = {->}](X2);
    		    \path[->] (X1) edge[ style = {->}](X3);
    		    \path[->] (X2) edge[ style = {->}](X4);
    		    \path[->] (X2) edge[ style = {->}](X3);
    		    \path[->] (X3) edge[ style = {->}](X4);
    		    \path[->] (X3) edge[ style = {->}](X5);
    		    \path[->] (X3) edge[ style = {->}](X6);
    		    \path[->] (X4) edge[ style = {->}](X5);
    		    \path[->] (X4) edge[ style = {->}](X6);
    		    \path[->] (X5) edge[ style = {->}](X7);
    		    \path[->] (X6) edge[ style = {->}](X7);
    		    \path[->] (X7) edge[ style = {->}](Y);
    		   \path[->] (X2) edge[dashed, style = {<->}](X5);
    		   \path[->] (X2) edge[ dashed, style = {<->}, bend right=60](X6);
    		\end{tikzpicture}
    	\caption{Causal graph of experiments in Section \ref{sec: exp simple g} with $N=7$.}
    	\label{fig: graph_simgen_N7}
\end{figure}
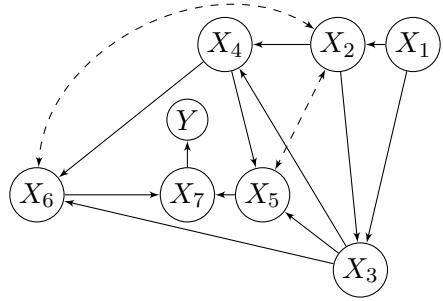

We also provided additional experiments on a different graph illustrated in Figure \ref{fig: graph_simgen_N5} with $N=5$ and one hidden confounder.
Note that this graph also includes unblocked backdoor paths from intervenable variables to the reward variable.
We constructed the underlying model similar to the one in Section \ref{sec: exp cumu}.
%Figure \ref{fig: simgen-BN5} 
The left plot in Figure \ref{fig: simgenN5} compares the performance of algorithms in terms of their simple regret when the cost of pulling each interventional arm was selected randomly from $\{5, 6, 7\}$. 
As depicted, by increasing the budget, Algorithm \ref{alg: g-simple} converges to zero faster.
%Figure \ref{fig: simgen-cN5} 
The right plot in Figure \ref{fig: simgenN5} shows that when the cost of all interventional arms is equal to $c$, the simple regret increases by increasing $c$ as expected, but our algorithm's regret grows at a lower rate.

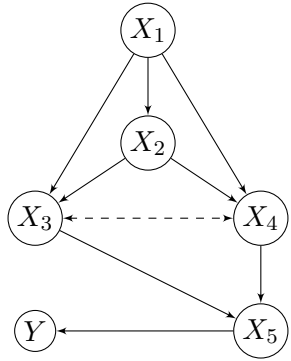
\begin{figure}[!ht] 
        \centering
        \tikzstyle{block} = [draw, fill=white, circle, text centered]
    	\tikzstyle{input} = [coordinate]
    	\tikzstyle{output} = [coordinate]
    	    \begin{tikzpicture}[->, auto, node distance=1.3cm,>=latex', every node/.style={inner sep=0.05cm}]
    		    \node[block] (X2) at (0,1) {$X_2$};
    		    \node[block] (X4) at (1.5,0) {$X_4$};
    		    \node[block] (X3) at (-1.5,0) {$X_3$};
    		    \node[block] (X1) at (0,2.5) {$X_1$};
    		    \node[block] (X5) at (1.5,-1.5) {$X_5$};
    		    \node[block] (Y) at (-1.5,-1.5) {$Y$};
    		    \path[->] (X1) edge[ style = {->}](X3);
    		    \path[->] (X1) edge[ style = {->}](X4);
    		    \path[->] (X1) edge[ style = {->}](X2);
    		    \path[->] (X2) edge[ style = {->}](X3);
    		    \path[->] (X2) edge[ style = {->}](X4);
    		    \path[->] (X3) edge[ style = {->}](X5);
    		    \path[->] (X4) edge[ style = {->}](X5);
    		    \path[->] (X5) edge[ style = {->}](Y);
    		    \path[->] (X3) edge[dashed, style = {<->}](X4);
    		\end{tikzpicture}
    	\caption{Causal graph of the additional experiments.}
    	\label{fig: graph_simgen_N5}
    \end{figure}
    
    \begin{figure}[!ht] 
        \centering
        \captionsetup{justification=centering}
        \begin{thesubfigure}%[b]{0.43}
            \centering
             \includegraphics[width=0.49\textwidth]{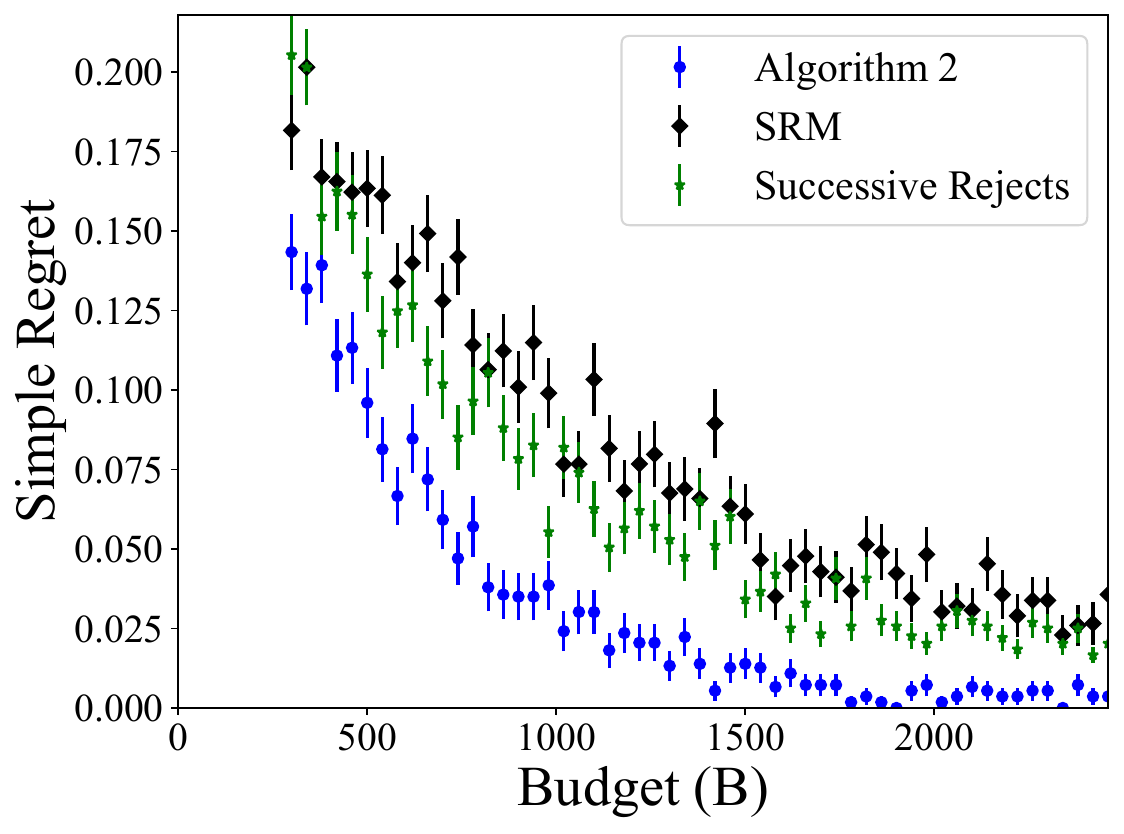}
            % \caption{Simple regret vs budget.}
            \label{fig: simgen-BN5}
        \end{thesubfigure}
        \begin{thesubfigure}%[b]{0.45}
            \centering
            \includegraphics[width=0.49\textwidth]{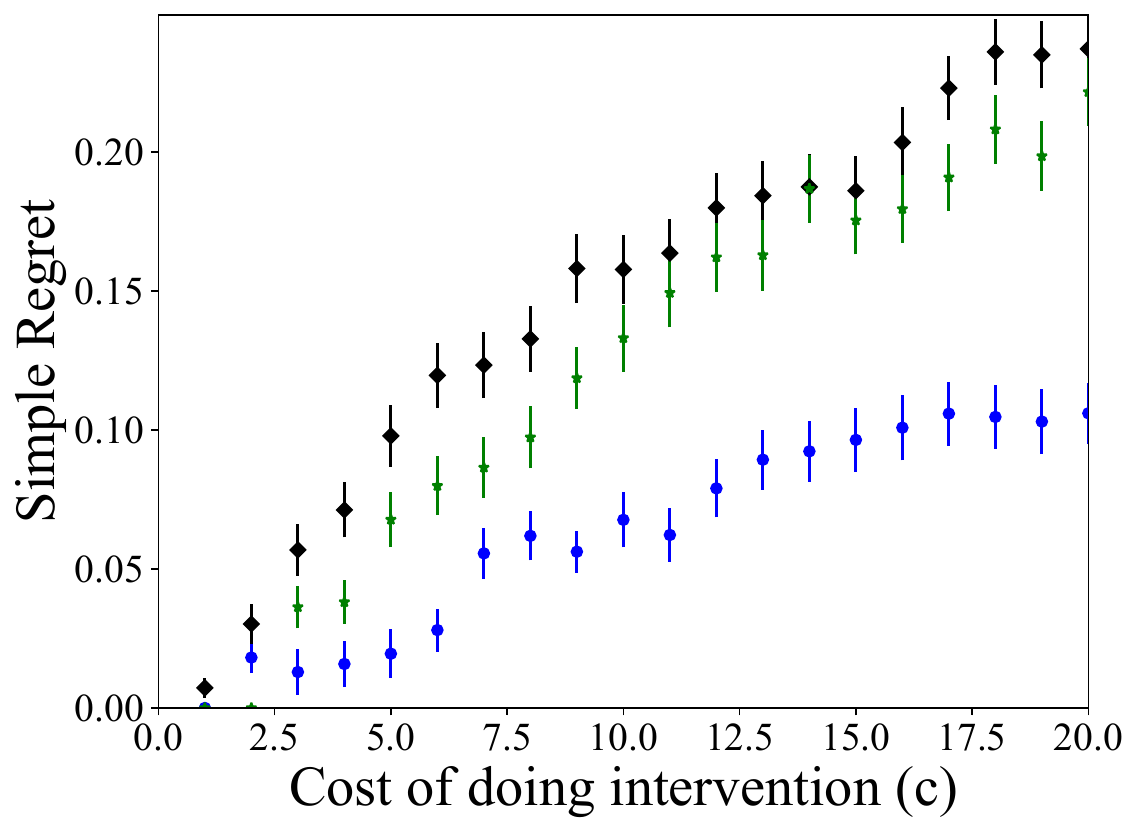}
            % \caption{Simple regret vs cost of intervening.}
            \label{fig: simgen-cN5}
        \end{thesubfigure}
        \caption{Performance of different algorithms on the general graphs depicted in Figure \ref{fig: graph_simgen_N5}.}
        \label{fig: simgenN5}
    \end{figure}

\section{Discussion on the previous work}\label{sec: error}  

In \citep{nair2021budgeted}, Eq. (17) in the proof of Lemma B.6 (crucial for Theorem 3 pertaining to cumulative regret bound) reads as follows  
$$
N_{T}^{i,x} \leq \max (0, l-\sum_{t \in [T]} \mathds{1}\{a(t)=a_0, X_i = x\})+ \sum_{t \in T} \mathds{1}\{a(t)=a_{i,x}, E_{t}^{i,x}\geq l\}
$$
which is wrong. As a counterexample, suppose  that $T=3$ and the pulled arms  are
$\{a(0)\!=\!a_{i,x},a(1)=a_{j,x'},a(2)=a_{0},a(3)=a_{0}\}$,
where $j\neq i$ and observed $X_i=x$ at both times $t=2,3$.
In this case, for $l\!=\!2$, the inequality becomes $1 \leq 0$.

\end{document}